\documentclass{article}

\PassOptionsToPackage{numbers, compress}{natbib}
\usepackage[preprint]{neurips_2026}

\usepackage[utf8]{inputenc} 
\usepackage[T1]{fontenc}    
\usepackage{hyperref}       
\usepackage{url}            
\usepackage{booktabs}       
\usepackage{amsfonts}       
\usepackage{nicefrac}       
\usepackage{microtype}      
\usepackage{xcolor}         

\usepackage{amssymb}
\usepackage{amsthm}
\usepackage{mathtools} 
\usepackage{graphicx}
\usepackage{listings}
\usepackage{enumitem}
\usepackage{cleveref}

\DeclareMathOperator*{\argmax}{arg\,max}

\newcommand{\Lcat}{\mathcal{L}_{\text{CAT}}}
\newcommand{\Lce}{\mathcal{L}_{\text{CE}}}

\theoremstyle{plain}
\newtheorem{theorem}{Theorem}
\newtheorem{corollary}{Corollary}

\theoremstyle{definition}
\newtheorem{definition}{Definition}

\title{Compute Aligned Training:\\
Optimizing for Test Time Inference}

\author{
  Adam Ousherovitch \\
  Department of Statistics \\
  University of Michigan \\
  \texttt{aoushero@umich.edu}
  \and
  \textbf{Ambuj Tewari} \\
  Department of Statistics \\
  University of Michigan \\
  \texttt{tewaria@umich.edu} 
}
\date{\today}

\begin{document}
\maketitle

\begin{abstract}
Scaling test-time compute has emerged as a powerful mechanism for enhancing Large Language Model (LLM) performance. However, standard post-training paradigms, Supervised Fine-Tuning (SFT) and Reinforcement Learning (RL), optimize the likelihood of individual samples under a base policy, creating a misalignment with test time procedures that rely on aggregated or filtered outputs. In this work, we propose Compute Aligned Training, which aligns training objectives with test-time strategies. By conceptualizing inference strategies as operators on the base policy, we derive new loss functions that maximize performance when said strategies are applied. We instantiate such loss functions for SFT and RL across common test time strategies. Finally, we provide empirical evidence that this training method substantially improves test time scaling over standard training. 
\end{abstract}

\section{Introduction}
In his foundational essay, \textit{The Bitter Lesson}, Richard Sutton observed that "search and learning are the two most important classes of techniques for utilizing massive amounts of computation in AI research" \cite{sutton2019bitter}. Historically, the generative AI community has leaned heavily into the latter, driving progress by scaling the learning phase \cite{kaplan2020scaling}. Recently, the frontier is shifting towards scaling test-time compute \cite{snell2024scaling}: a form of search. Yet, current training paradigms treat learning and search as isolated, optimizing models for one-shot performance rather than operating in conjunction with a search procedure. However, ultimately, search and learning are not isolated processes. They complement each other in a continuous cycle. Search explores the environment to find high-quality answers, and learning internalizes these discoveries, enabling the model to search more effectively in the future. Therefore, learning benefits from a model optimized for search. Thus, to fully align with modern deployment, a model must be trained to support the inference procedures used at test-time. \footnote{All code for this paper can be found at https://github.com/adamoosya/compute-aligned-training.}

Standard training, such as Supervised Fine-Tuning (SFT) and Reinforcement Learning (RL), explicitly optimizes for the \textit{expected} reward of a single output. This implicitly encourages the model to converge on a narrow set of "safe," high-likelihood outputs. The result is an actively detrimental misalignment with test-time scaling -- specifically, using Large Language Models (LLMs) to generate multiple candidate solutions, particularly Chain-of-Thought (CoT) reasoning traces \cite{wei2022chainofthought}, and applying aggregation strategies like Pass@$N$ \cite{chen2021codex}, Majority Vote \cite{wang2023selfconsistency}, or Best-of-$N$ verification \cite{cobbe2021verifiers}. Strategies like Best-of-$N$ and Pass@$N$ allow the model to benefit when it produces low-probability yet high-quality solutions. Standard training has been shown to actively degrade Pass@$N$ performance by effectively collapsing the exploration required for search \cite{chen2025rethinking, yue2025does}. Similarly, \citet{dang2025weight} found that SFT can yield suboptimal test-time scaling under majority vote. The model is optimized for single output performance, so at inference, when it is forced to act as an ensemble, it performs suboptimally.

Despite the relevance of test-time strategies, standard training protocols remain largely agnostic to how models are actually deployed. Prior work has only identified and addressed isolated instances of this misalignment. Solutions are typically restricted to either a certain test time strategy or training modality (\Cref{sec: related}). However, this misalignment is a general phenomenon that persists across broad training paradigms, from SFT to RL and various test-time strategies. This raises the question, can we extract better performance under test time scaling by aligning training with test time strategy without substantially increasing training-time compute costs? Solving a problem this broad requires a generalized, rather than piecemeal, approach. To that end, we provide a unified theoretical framework as the mechanism to bridge the gap between training and inference.

In this work, we propose \textbf{Compute Aligned Training (CAT)}, a framework for aligning training objectives with test-time. The key observation enabling this method is that the inputs to inference strategies are the generative model and a set of known hyperparameters, ultimately outputting a single solution. Thus, we can conceptualize inference strategies not as heuristics applied \textit{post-hoc}, but as operators $\mathcal{T}$ that transform the base policy $\pi_\theta$ into an effective test-time distribution $\tilde{\pi}_\theta$. By performing gradient descent to minimize the loss with respect to this \textit{transformed} distribution, we can align the model training with its deployment environment. Whether for SFT or RL, CAT re-weights gradients based on the "marginal utility" of a sample to the aggregate outcome with no added overhead. We will show empirically that this approach broadly improves performance under test-time scaling.

\subsection{Related Works}
\label{sec: related}
Efforts to align training objectives with inference strategies are fragmented across modalities. 

In the domain of \textbf{Supervised Fine-Tuning}, \citet{chen2025rethinking} pioneered the identification of the misalignment between Cross-Entropy and Pass@$N$, proposing \textit{Direct Coverage Optimization} (DCO). While the most foundational paper to our work, DCO is restricted to SFT and does not extend to strategies outside of Pass@$N$. Our primary contribution is extending their ideas across Test Time Strategies, to RL, and beyond LLMs.

In the \textbf{Reinforcement Learning} setting, \citet{tang2025optimizing} propose gradient estimators for Pass@$k$ and Majority Voting. However, their method operates on a fundamentally different mechanism which introduces issues that CAT avoids (\Cref{app: Full TTS Rollout}). Alternatively, \citet{chow2025inferenceaware} addresses Best-of-$N$ for both SFT and RL, but relies on variational approximations of the verification process rather than exact policy optimization.

BoNBoN \cite{gui2024bonbon} fine-tunes a model to approximate the Best-of-$N$ distribution with a single sample at deployment. CAT instead trains the base policy for use with a specified multi-sample inference strategy. This is related insofar as they similarly calculate the marginal utility of a given sample to the BoN procedure in the same way we do.

We propose a unified framework that overcomes these modality and efficiency constraints. Our operator-based formulation allows for computationally efficient, exact gradient scaling. This allows us to recover DCO as a special case while generalizing to RL and complex strategies beyond Pass@N.

\subsection{Contributions}

Our primary contributions are as follows:

\begin{itemize}
\item \textbf{A Unified Operator Framework:} We formalize test-time strategies as operators on the base policy, deriving a unified gradient scaling mechanism that aligns training objectives with arbitrary inference strategies. We then instantiate that framework for common cases in which it may be used.
    
\item \textbf{Empirical Validation across Paradigms:} We provide empirical evidence that CAT generalizes the work of \citet{chen2025rethinking} in three directions. We push test time alignment to strategies beyond Pass@N, beyond SFT, and beyond LLMs to Protein Language Models (PLMs), successfully improving performance once an inference strategy is aligned.

\end{itemize}

\section{Method: Compute Aligned Training}
\label{sec:method}

As shown in \citet{chen2025rethinking} and \citet{yue2025does}, the mismatch between training and inference decreases performance under test time scaling. The mechanism via which this happens is that standard objectives like Cross-Entropy optimize the likelihood of a sample regardless of whether that improvement actually changes the test-time outcome. This can cause effective overtraining once the Test Time Strategy (TTS) is applied.

Consider two scenarios that highlight this inefficiency:
\begin{enumerate}
    \item \textbf{The Efficiency Gap (Pass@$N$):} If a model has a $50\%$ probability of generating the correct answer ($p=0.5$), increasing this to $90\%$ provides negligible benefit to a Pass@100 metric, where success is already virtually guaranteed. Yet, standard SFT continues to direct the model to maximize this likelihood.
   \item \textbf{The Consensus Gap (Majority Vote):} Majority Vote is a competition: a probability of $0.35$ is likely sufficient to win if the strongest rival has $0.1$ probability. Standard training, however, treats $0.35$ as high error and pushes the model toward $p \to 1.0$. This trains the model to optimize overwhelming ("landslide") victories.
\end{enumerate}

The broader theme is that standard training misallocates gradient pressure towards solutions that are already likely to be solved once the inference strategy is applied, without consideration of the utility of a single generation in context (\Cref{app:gradient_alignment}). To address this mismatch, we need to quantify the marginal utility of a sample to the TTS being used.

\subsection{Test-Time Strategies as Differentiable Operators}
\label{sec:method_tts_operator}

Let $\mathcal{X}$ be the space of all possible inputs (prompts) and $\mathcal{Y}$ be the space of all possible generated sequences (solutions). For a given input $x \in \mathcal{X}$, the base policy $\pi_{\theta}(\cdot|x)$ produces a probability distribution over the output space. This distribution resides in a standard probability simplex $\Delta^{|\mathcal{Y}|-1}$.

For any TTS which takes as input a generative model and hyperparameters, we can represent it as a functional operator $\mathcal{T}: \Delta^{|\mathcal{Y}|-1} \times \Phi \rightarrow \Delta^{|\mathcal{Y}|-1}$, mapping the base policy and hyperparameters $\phi \in \Phi$ (e.g., a sample budget $N \in \mathbb{N}$) to an effective system-level policy $\tilde{\pi}_{\theta}$ \footnote{For a strategy like Pass@N, if the model doesn't find the correct output, it doesn't output any answer. We preserve the mapping to the simplex by conceptualizing $\tilde{\pi}_\theta$ as an augmented distribution. The target sequence receives probability $\tilde{p}$, and all remaining mass $1-\tilde{p}$ is assigned to a consolidated "failure" state. See \Cref{app:sft_passn} for details.}::

\begin{equation}
    \tilde{\pi}_{\theta}(y|x) = \mathcal{T}(\pi_{\theta}(\cdot|x), \phi)(y)
\end{equation}

where $y \in \mathcal{Y}$. Sampling from this effective distribution demands greater test-time compute (e.g., a sampling budget of $N$ scales inference cost by a factor of $N$ relative to $\pi_\theta$). As we will show, for many TTS, we can derive the functional form of $\mathcal{T}$. Furthermore, that form is often differentiable with respect to the base policy probabilities $\pi_\theta(y'|x)$.

At deployment, predictions do not come from $\pi_\theta$ but from $\tilde\pi_\theta$. Thus, we must select $\theta$ to maximize the likelihood of the correct output $y^* \in \mathcal{Y}$ given $x$ under the effective test-time distribution $\tilde{\pi}_\theta$ (rather than $\pi_\theta$). So, in the SFT setting, our goal is to minimize the induced loss over our training set $\mathcal{D}$, defined as the negative log-likelihood of the deterministic ground truth output $y^*$ under the induced distribution:\looseness-1

\begin{equation}
    \mathcal{L}_{CAT}(\theta) = -\mathbb{E}_{(x,y^*) \sim \mathcal{D}}\left[ \log \tilde{\pi}_{\theta}(y^*|x) \right]
\end{equation}

Let $p = \pi_{\theta}(y^*|x) \in [0,1]$ denote the base probability of the correct answer, and $\tilde{p} = \tilde{\pi}_{\theta}(y^*|x) \in [0,1]$ denote the effective test-time probability. Differentiating the objective requires applying the chain rule over the entire output space  $\mathcal{Y}$ (more details in \Cref{app:sft}):
\begin{equation}
    \nabla_{\theta}\mathcal{L}_{CAT} = -\frac{1}{\tilde{p}} \sum_{y' \in \mathcal{Y}} \frac{\partial \tilde{p}}{\partial \pi_{\theta}(y'|x)} \nabla_{\theta}\pi_{\theta}(y'|x)
\end{equation}

Importantly, we do not have the full vector $\pi_\theta(\cdot|x)$, only $p$. To avoid additional sampling, we need to approximate this gradient as a function of $p$ and $\phi$. Thus, we don't compute this full Jacobian, as it requires evaluating the gradient through every possible alternative sequence $y'$. We introduce a diagonal approximation of the Jacobian.

\begin{theorem}[Diagonal Gradient Decomposition]
The exact gradient of the CAT objective can be strictly decomposed into a scaled standard Cross-Entropy gradient ($\nabla_{\theta}\mathcal{L}_{CE} = -\nabla_{\theta} \log p$) and an off-diagonal residual vector $\epsilon_{\text{vec}}$:
\begin{equation}
    \nabla_{\theta}\mathcal{L}_{CAT} = \underbrace{\left( \frac{p}{\tilde{p}} \frac{\partial \tilde{p}}{\partial p} \right)}_{w} \cdot \nabla_{\theta}\mathcal{L}_{CE} + \epsilon_{\text{vec}}
\end{equation}
where the exact off-diagonal residual is defined as:
\begin{equation}
    \epsilon_{\text{vec}} = -\frac{1}{\tilde{p}} \sum_{y' \neq y^*} \frac{\partial \tilde{p}}{\partial \pi_{\theta}(y'|x)} \nabla_{\theta}\pi_{\theta}(y'|x)
\end{equation}
In general, $\tilde{p}$ and thus $w$ depend on the full base probability vector $\pi_{\theta}(\cdot|x)$. However, if we can approximate $\tilde{p}$ as a function of only the ground truth probability $p$ and strategy hyperparameters $\phi$, the scaling coefficient simplifies to a tractable scalar function $w(p, \phi)$. Consequently, optimizing this scalar approximation $w(p, \phi) \cdot \nabla_{\theta}\mathcal{L}_{CE}$ serves as a reliable surrogate gradient for $\mathcal{L}_{CAT}$ provided the magnitude of the off-diagonal residual $\|\epsilon_{\text{vec}}\|$ does not overpower the diagonal update (formally bounded and proven in \Cref{sec:diagonal_sensitivity}).
\end{theorem}

The intuition is simple: often, the most important of the base probabilities $\pi_\theta(\cdot|x)$ for selecting $y^*$ after search is the base probability of $y^*$. We present sensitivity analysis for this approximation in \Cref{sec:diagonal_sensitivity} and theoretical and empirical justification in all the cases we considered in  \Cref{app:scalar_justification}, \Cref{app: empirical validation 1}, \Cref{app: BoN Symmetry Guarantee}, and \Cref{sec:bon_diagonal}. For an arbitrary search strategy, this is not guaranteed to be a good approximation; however, it allows us to use no additional compute to perform CAT.

This theorem provides our unified mechanism: the strategy communicates its requirements to the model by re-weighting the standard gradient based on the marginal test-time utility of the sample. Specific scaling factors become corollaries of this theorem with a reference table in \Cref{tab:strategies}.

\begin{table}[h]
\centering
\caption{Scaling factors for different strategies and post-training regimes.}
\label{tab:strategies}
\renewcommand{\arraystretch}{1.8} 
\begin{tabular}{@{}lll@{}}
\toprule
\textbf{Strategy} & \textbf{SFT Scaling Factor} & \textbf{RL Scaling Factor} \\ \midrule
\textbf{Pass@$N$} & $\frac{N p (1 - p)^{N-1}}{1 - (1 - p)^N}$ & $N(1-p)^{N-1}$ \\
\textbf{Majority Vote} & $\frac{k \binom{N}{k} p^k (1-p)^{N-k}}{\sum_{i=k}^{N} \binom{N}{i} p^i (1-p)^{N-i}}$ & $N \binom{N-1}{k-1} p^{k-1} (1-p)^{N-k}$ \\
\textbf{Best-of-$N$} & $\frac{N p (1 - p)^{N-1}}{1 - (1 - p)^N}$ & $N (P_{<y_i})^{N-1}$ \\ \bottomrule
\end{tabular}
\end{table}

\begin{corollary}[Pass@N Scaling Factor]
    For Pass@N, success requires generating the correct answer at least once in $N$ attempts: $\tilde{p} = 1 - (1-p)^N$. Because Pass@N success is invariant to the distribution of incorrect answers, $\epsilon = 0$ and $\frac{\partial \tilde{p}}{\partial p}$ only depends on $p$ and $N$, so the factor is exact. The SFT scaling factor is:
    \begin{equation}
    w_{pass, SFT}(p, N) = \frac{N p (1-p)^{N-1}}{1 - (1-p)^N}
    \end{equation}
\end{corollary}

\begin{corollary}[Majority Vote Scaling Factor]
    For Majority Vote, success requires the correct answer to achieve plurality. To make this a function of $p$ and $\phi$, we approximate this as reaching a fixed consensus threshold $k$ within $N$ attempts, which yields $\tilde{p} = \sum_{i=k}^{N} \binom{N}{i} p^i (1-p)^{N-i}$. By setting $k$ to a fixed hyperparameter, $\frac{\partial \tilde{p}}{\partial p}$ is a function of $p$ and $\phi$. In Majority Vote, the true k is a function of terms in $\pi_\theta(\cdot|x)$ other than p. Thus, off-diagonal terms exist, but the diagonal approximation is safely bounded (see \Cref{app:scalar_justification}). The SFT scaling factor is:
    \begin{equation}
        w_{maj, SFT}(p, N, k) = \frac{k \binom{N}{k} p^k (1-p)^{N-k}}{\sum_{i=k}^{N} \binom{N}{i} p^i (1-p)^{N-i}}
    \end{equation}
\end{corollary}

A similar derivation (\Cref{app:rl}) shows that for RL, we reweigh gradients by a similar factor:

\[w(p, \phi) = \frac{\partial \tilde{p}}{\partial p}\]

An intuition for why this is different than the SFT formulation is that RL optimizes utility, which is linear with respect to the policy, while SFT optimizes log probability, so the gradient passes through a log term. 
RL has the key distinction that we can use the empirical distribution from rollouts to dynamically estimate strategy parameters, such as the $k$ in Majority Vote, from the batch.

\begin{corollary} [Best-of-$N$ Scaling Factor for RL] For Best-of-$N$ in Reinforcement Learning, the strategy selects the candidate with the maximum reward from a batch of $N$ samples. Let $P_{<y}$ denote the probability of generating an output with a reward strictly less than $R(y|x)$. The effective probability that candidate $y$ wins the selection is $\tilde{\pi}_\theta(y|x) = (P_{<y} + \pi_\theta(y|x))^N - (P_{<y})^N$. Invoking the diagonal assumption and a local probability conservation assumption with justification in \Cref{app: BoN Symmetry Guarantee}, $\frac{\partial \tilde{p}}{\partial p} = N(P_{<y})^{N-1}$. Since the RL scaling factor is this derivative, the weight is:
\[w_{bon, RL}(\pi_\theta(\cdot|x), N) = N(P_{<y})^{N-1}\]
We can estimate $P_{<y}$ from the RL rollout batch via the empirical quantile $\hat{P_{<y}}$.
\end{corollary}

\subsection{Interpreting the Scaling Factors} 


$w(p,\phi)$ determines how strongly a sample influences training, based on its \textit{marginal contribution to test-time success}. Gradients are amplified when a change meaningfully affects the final outcome, and suppressed when more confidence is redundant. Visualizations of each $w$ are in \Cref{app: interpereting scaling factors}.

\paragraph{Pass@$N$} For Pass@N, $w$ decreases monotonically with the base probability $p$. Once the model is likely to solve a problem within $N$ attempts, further increasing $p$ yields negligible improvement in Pass@$N$ performance. $w$ suppresses gradients on “easy” examples, allocating gradients only towards hard problems where marginal gains still improve coverage.

\paragraph{Majority Vote} Majority Vote success depends on surpassing competing answers. Accordingly, the scaling factor concentrates gradient mass near the consensus boundary. Samples that are already likely to win a plurality receive little to no signal.

\paragraph{Best-of-N} In Best-of-N RL, the scaling factor depends on a sample’s rank within the reward distribution rather than its absolute likelihood. This induces winner-take-all dynamics: gradients reinforce the highest-reward samples, while lower-reward outputs are ignored.



\section{Experiments}
\label{sec:experiments}

To empirically validate CAT, we show how it extends \citet{chen2025rethinking} by evaluating the method on three progressive axes to demonstrate its generality: moving \textbf{beyond Pass@$N$} to majority vote, \textbf{beyond SFT} to Reinforcement Learning, and \textbf{beyond LLMs} to PLMs.

For our language modeling experiments, we fine-tune the \texttt{Mistral-7B-Instruct-v0.2} model \cite{jiang2023mistral} on the MATH benchmark \cite{hendrycksmath2021}. To isolate the effect of strategy-aware loss from general feature learning and save compute, all models first undergo a shared two-epoch "warmup" phase using standard cross-entropy. This stabilizes the base policy $\pi_\theta$ into a regime of non-trivial success probability, which helps to ensure the strategy-aware gradient scaling factors $w(p, \phi)$ remain numerically stable (see \Cref{app:sft_variance}) before branching into the experimental CAT regimes.

\subsection{Beyond Pass@$N$: Maj Vote in SFT}
\label{sec:beyond_pass}

Prior work has demonstrated that standard cross-entropy is misaligned with Pass@$N$ \cite{chen2025rethinking}. We first validate our CAT framework by recovering this result. As shown in \Cref{tab:pass_results} and \Cref{fig:pass_k_delta}, applying the Pass@$N$ CAT objective actively suppresses overconfidence on easy samples. Models trained with larger budgets (e.g., $N=64$) incur a "tax" on single-sample accuracy (Pass@1 drops by 2.28\%) but yield a 7.80\% absolute improvement in Pass@64 coverage over the SFT baseline.

\begin{table}[h]
\centering
\caption{Absolute Pass@$k$ performance on the MATH benchmark. Strategy-aware models trade off single-sample accuracy (Pass@1) for superior coverage at large inference budgets (Pass@64). Standard errors (SE) are reported in parentheses.}
\label{tab:pass_results}
\resizebox{\linewidth}{!}{%
\begin{tabular}{@{}lcccccc@{}}
\toprule
\textbf{Model} & \textbf{@1} & \textbf{@4} & \textbf{@8} & \textbf{@16} & \textbf{@32} & \textbf{@64} \\ \midrule
SFT Baseline & \textbf{15.8\%} (1.6) & 29.6\% (2.0) & 37.9\% (2.2) & 46.3\% (2.2) & 53.7\% (2.2) & 59.8\% (2.2) \\
Pass@$N=4$ & 15.2\% (1.6) & \textbf{30.5\%} (2.1) & 39.9\% (2.2) & 49.2\% (2.2) & 57.5\% (2.2) & 63.8\% (2.1) \\
Pass@$N=16$ & 14.5\% (1.6) & 30.1\% (2.1) & \textbf{40.3\%} (2.2) & \textbf{50.9\%} (2.2) & \textbf{59.8\%} (2.2) & 66.2\% (2.1) \\
Pass@$N=64$ & 13.5\% (1.5) & 28.5\% (2.0) & 38.7\% (2.2) & 49.4\% (2.2) & 59.3\% (2.2) & \textbf{67.6\%} (2.1) \\ \bottomrule
\end{tabular}%
}
\end{table}

\begin{figure}[!h]
    \centering
    \includegraphics[width=0.85\textwidth]{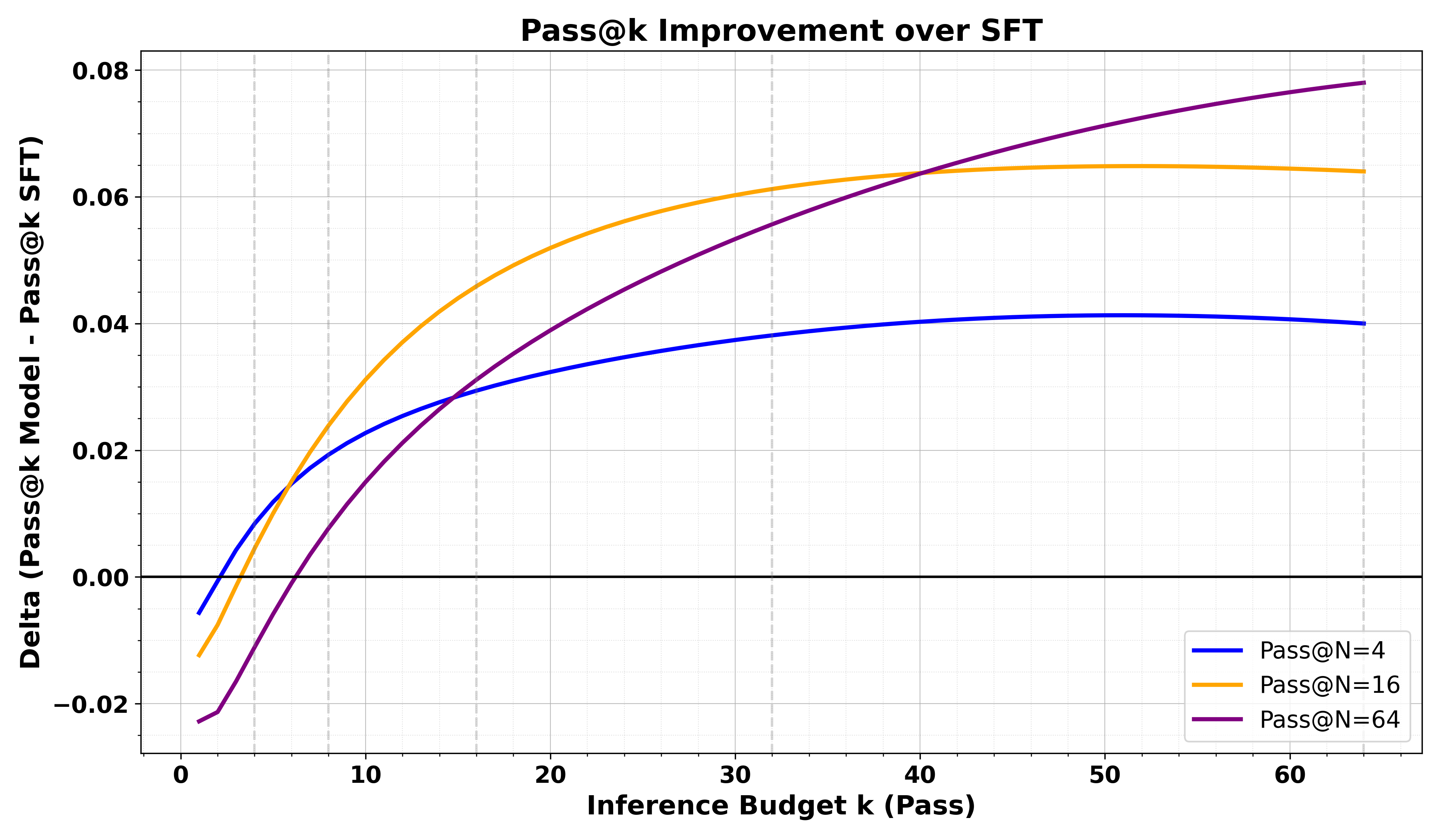}
    \caption{\textbf{Pass@$k$ Improvement over SFT.} The performance difference ($\text{Pass}@k_{\text{model}} - \text{Pass}@k_{\text{SFT}}$). High-$N$ models (Purple) sacrifice low-budget accuracy to achieve superior performance at scale.}
    \label{fig:pass_k_delta}
\end{figure}

 To validate CAT more broadly, we transition to Majority Vote. We trained models across varying budgets $N \in \{8, 16, 64\}$ selecting $k$ for each model via sweeps (\Cref{app:maj_sweep}). In the previous experiment, we trained the models on just answers, whereas in this experiment, we used both answers and solution traces. Our rationale is that Majority Vote relies on CoT while Pass@N does not.

As detailed in \Cref{tab:maj_results} and \Cref{fig:final_champions}, CAT strictly outperforms the baseline when Majority Vote is applied at its target inference budget. Interestingly, we observe a "sweet spot" at $N=16$ where this model broadly outperforms all the others. This highlights a bias-variance tradeoff inherent to CAT (characterized theoretically in \Cref{app:sft_variance}). For some strategies, gradient variance increases as intended compute budget increases. Furthermore, the choice of hyperparameter $k$ may have been more accurate for $N=16$ than the other strategies. Regardless, all models show improved performance when the TTS is applied.

\begin{table}[h]
\centering
\caption{Absolute Maj@$k$ performance. All CAT models outperform SFT once Majority Vote is applied at its target inference budget. Standard errors (SE) are reported in parentheses.}
\label{tab:maj_results}
\resizebox{\linewidth}{!}{%
\begin{tabular}{@{}lcccccc@{}}
\toprule
\textbf{Model} & \textbf{@1} & \textbf{@4} & \textbf{@8} & \textbf{@16} & \textbf{@32} & \textbf{@64} \\ \midrule
SFT Baseline & 21.9\% (1.8) & 22.9\% (1.9) & 23.3\% (1.9) & 23.6\% (1.9) & 23.8\% (1.9) & 23.9\% (1.9) \\
MajVote $N=8$ & 22.2\% (1.9) & 23.7\% (1.9) & 24.8\% (1.9) & 25.4\% (1.9) & 25.6\% (2.0) & 25.7\% (2.0) \\
MajVote $N=16$ & \textbf{22.4\% (1.9)} & \textbf{24.3\% (1.9)} & \textbf{25.2\% (1.9)} & \textbf{25.9\% (2.0)} & \textbf{26.1\% (2.0)} & \textbf{26.2\% (2.0)} \\
MajVote $N=64$ & 21.4\% (1.8) & 22.7\% (1.9) & 23.7\% (1.9) & 24.2\% (1.9) & 24.5\% (1.9) & 24.9\% (1.9) \\
\bottomrule
\end{tabular}%
}
\end{table}

\begin{figure}[h]
    \centering
\includegraphics[width=0.85\textwidth]{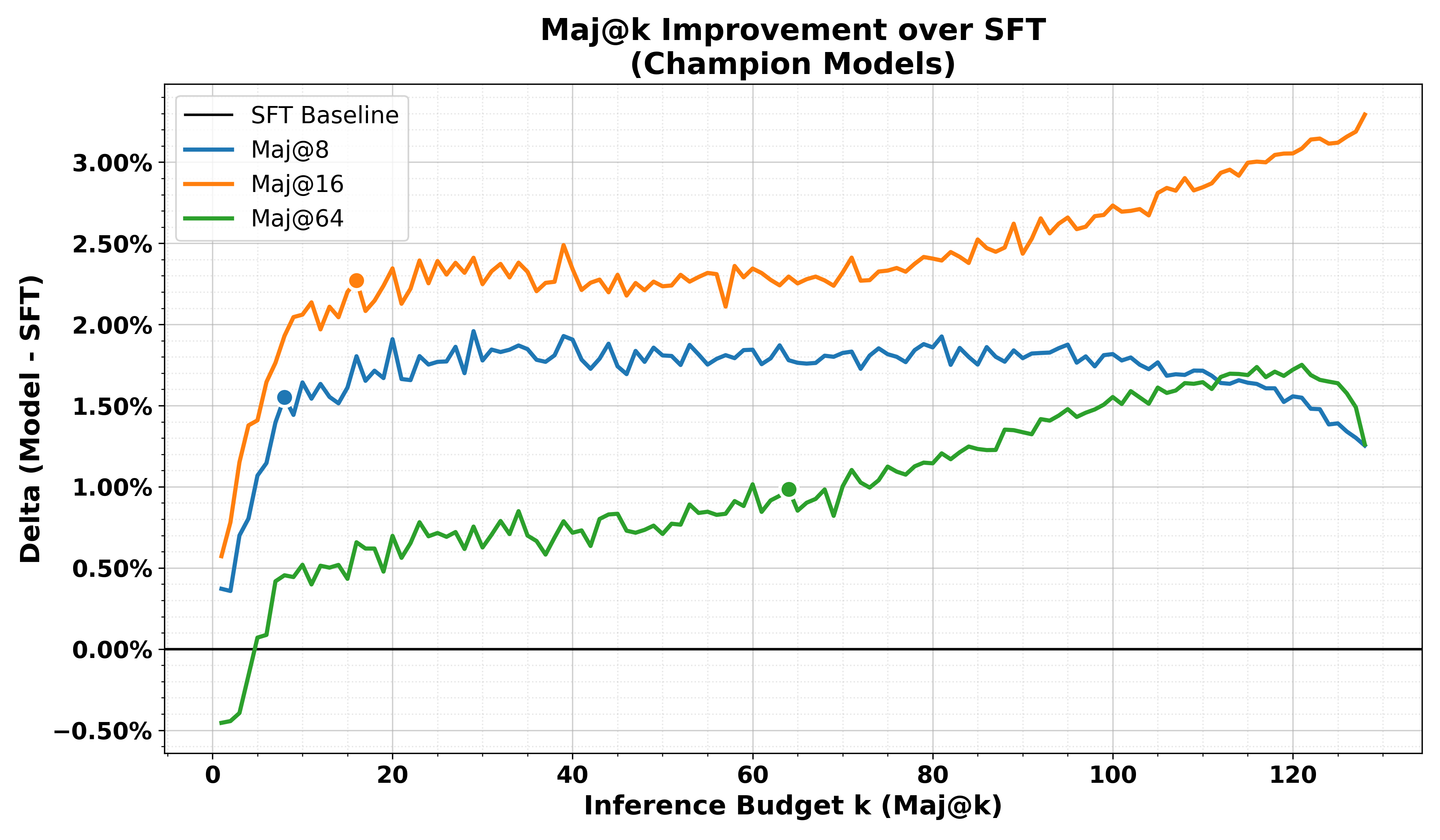}
    \caption{\textbf{Comparison of Best Models vs SFT.} All models demonstrate strong test time scaling, confirming that CAT modifies the distribution to support aggregation.}
    \label{fig:final_champions}
\end{figure}

\subsection{Beyond SFT: Reinforcement Learning}
\label{sec:beyond_sft}

Having established CAT's efficacy for SFT, we extend the framework to RL. To align RL, we modify the GRPO objective by applying our scaling factors directly to the normalized advantages (derivation in \Cref{sec:grpo_ppo}). We evaluate the CAT models on both Pass@$N$ and Majority Vote, utilizing the MATH benchmark. The models undergo a brief SFT warmup before transitioning to one epoch of strategy-aligned GRPO.

\paragraph{Pass@$N$ Coverage Validation.} As shown in \Cref{fig:rl_scaling_laws} and \Cref{tab:rl_sft_comparison}, standard RL exhibits poor test-time scaling. While it achieves a respectable Pass@1 (8.27\%), its performance saturates quickly due to policy collapse, reaching only 35.80\% at $k=32$. In contrast, our CAT models explicitly preserve the variance necessary for test-time discovery. The Pass@16 model reaches 40.00\% at $k=32$, a substantial relative improvement over the baseline. Furthermore, the Pass@4 model improved Pass@1 coverage over the base model. The reason for this is how CAT interacts with the RL generation process. We offer a deeper analysis of this in \Cref{app: improve_pass@1}.

\paragraph{Majority Vote Consensus Validation.} A fundamental difference between RL and SFT CAT for Majority Vote is that the rollouts in RL allow us to estimate the threshold $k$ from the rollout batch, while SFT forces us to set that as a hyperparameter. As detailed in \Cref{fig:main_performance} and \Cref{tab:maj_rl_results}, Standard RL scales suboptimally under Majority Vote aggregation. Conversely, the CAT-trained models trade off a marginal amount of Pass@1 accuracy to successfully shape the distribution for consensus.

\begin{figure}[h]
    \centering
    \begin{minipage}{0.48\textwidth}
        \centering
        \includegraphics[width=\linewidth]{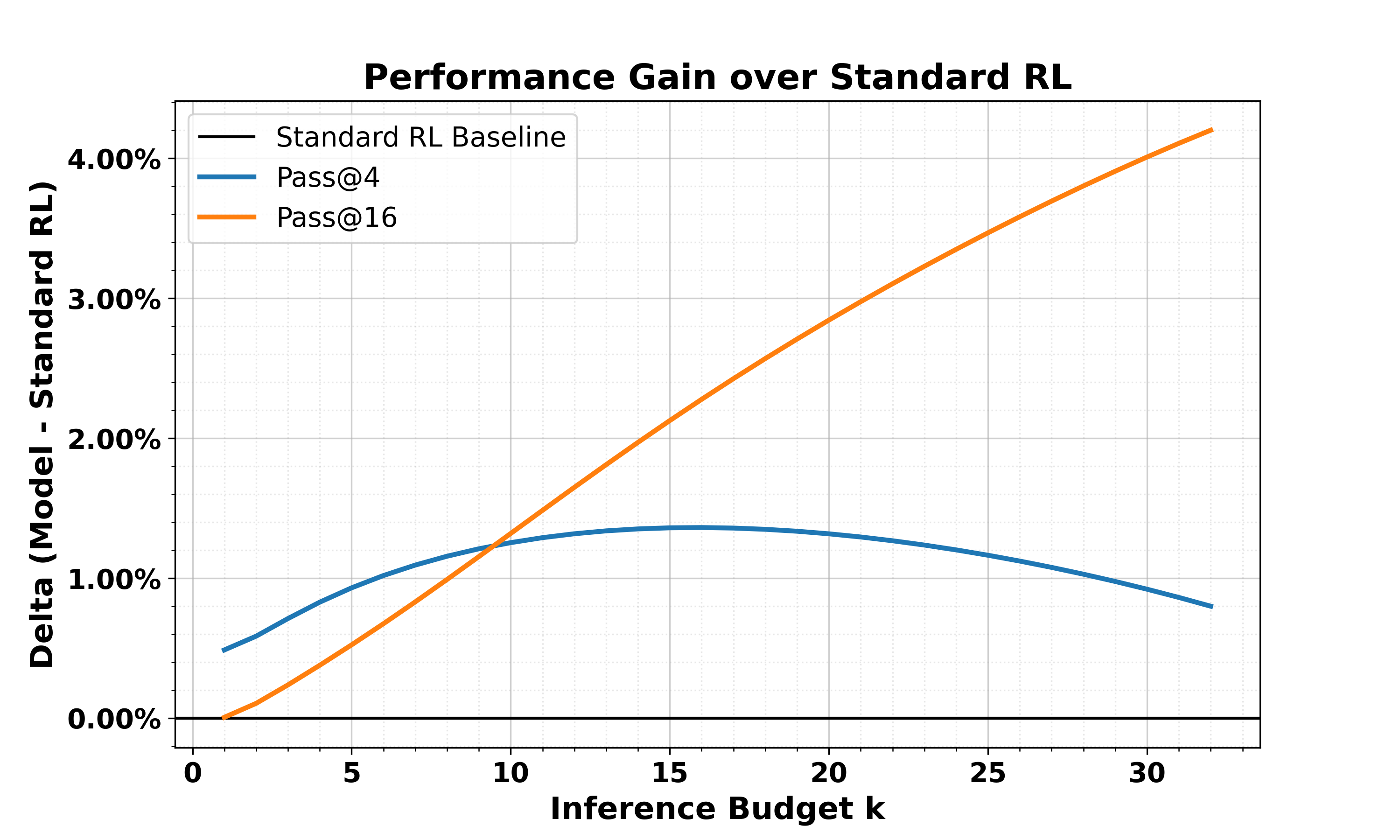}
        \caption{\textbf{Pass@$k$ RL Scaling.} Models trained with strategy-aware objectives (Pass@16) scale better than the baseline.}
        \label{fig:rl_scaling_laws}
    \end{minipage}\hfill
    \begin{minipage}{0.48\textwidth}
        \centering
        \includegraphics[width=\linewidth]{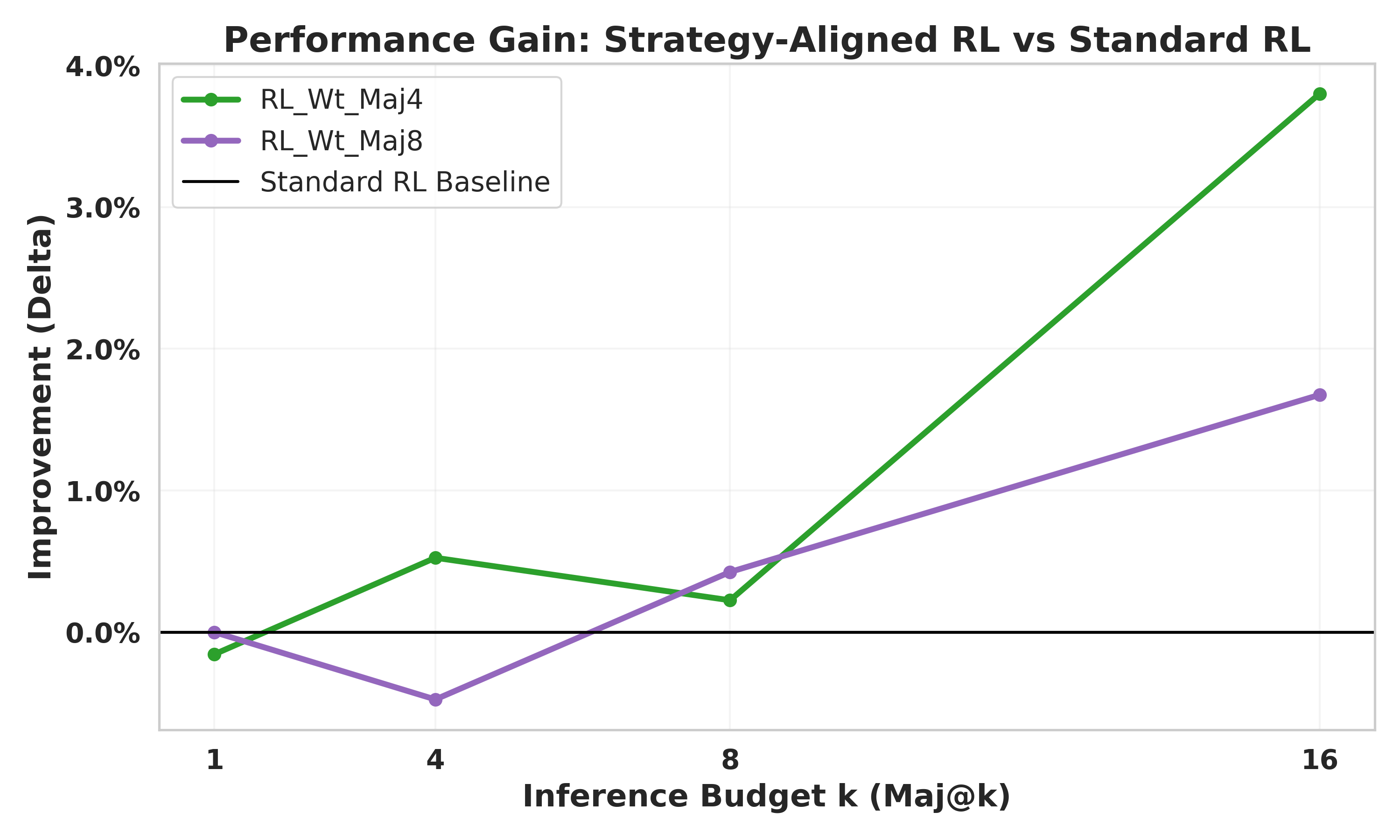}
        \caption{\textbf{Majority Vote RL Scaling.} The CAT models significantly outperform the baseline at higher inference budgets.}
        \label{fig:main_performance}
    \end{minipage}
\end{figure}
\begin{table}[h]
    \centering
    \setlength{\tabcolsep}{3pt} 
    
    \begin{minipage}{0.49\textwidth}
        \centering
        \caption{Pass@$k$ performance of RL models. Standard errors are reported in parentheses.}
        \label{tab:rl_sft_comparison}
        \resizebox{\linewidth}{!}{%
        \begin{tabular}{@{}lcccc@{}}
        \toprule
        \textbf{Model} & \textbf{@1} & \textbf{@8} & \textbf{@16} & \textbf{@32} \\ \midrule
        Standard & 8.3\% (1.2) & 24.1\% (1.9) & 30.3\% (2.1) & 35.8\% (2.1) \\
        Pass@4 & \textbf{8.8\% (1.3)} & \textbf{25.3\% (1.9)} & 31.6\% (2.1) & 36.6\% (2.2) \\
        Pass@16 & 8.3\% (1.2) & 25.1\% (1.9) & \textbf{32.5\% (2.1)} & \textbf{40.0\% (2.2)} \\ \bottomrule
        \end{tabular}%
        }
    \end{minipage}\hfill
    \begin{minipage}{0.49\textwidth}
        \centering
        \caption{Majority Vote RL performance. Standard errors are reported in parentheses.}
        \label{tab:maj_rl_results}
        \resizebox{\linewidth}{!}{%
        \begin{tabular}{@{}lcccc@{}}
        \toprule
        \textbf{Model} & \textbf{@1} & \textbf{@4} & \textbf{@8} & \textbf{@16} \\ \midrule
        Standard & \textbf{11.1\% (1.4)} & 14.4\% (1.6) & 17.6\% (1.7) & 19.2\% (1.8) \\ 
        Maj@4 & 10.9\% (1.4) & \textbf{15.0\% (1.6)} & 17.9\% (1.7) & \textbf{23.0\% (1.9)} \\ 
        Maj@8 & \textbf{11.1\% (1.4)} & 14.0\% (1.5) & \textbf{18.1\% (1.7)} & 20.9\% (1.8) \\ \bottomrule
        \end{tabular}%
        }
    \end{minipage}
\end{table}

\subsection{Beyond LLMs: Best-of-$N$ in Protein Generation}
\label{sec:beyond_llms}

To demonstrate that our framework is modality-agnostic, we investigate Best-of-$N$ (BoN) RL using Protein Language Models (PLMs). Standard RL optimizes the expected reward, making it inherently risk-averse. Conversely, BoN RL optimizes the maximum reward of a batch, acting as a filter. The objective signals that generating failures is permissible, provided the variance is high enough to produce a single high-reward output. A qualitative difference between BoN and the other strategies is that BoN has no single best answer. As such, we have to estimate the quantile of a given answer within the rollout batch \Cref{app:rl_bon}. We test CAT in two settings using the \texttt{ProtGPT2} \cite{ferruz2022protgpt2} model.

\paragraph{Experiment 1: Unconditional Hydrophobicity.}  We defined the reward $R(x)$ as a function of protein hydrophobicity $h(x)$ (see \Cref{app:bon_conditional_implementation} for the details). In protein engineering, hydrophobicity is often a useful property in designing stable transmembrane domains. From an optimization perspective, this can create non-convex rewards. In these scenarios, materials with low-hydrophobicity create safe water-soluble structures, medium-hydrophobicity structures tend to misfold, and high-hydrophobicity structures are the most stable \cite{lu2018accurate, dobson2003protein}. The resulting landscape is multi-modal, containing a local hydrophilic maximum ($h \approx 0.35$), a heavily penalized aggregation zone ($h \approx 0.50$), and a narrow global maximum ($h \approx 0.75$). As shown in \Cref{fig:bon_distributions} (Left), Standard RL converges into the local maximum, while BoN-trained models ($N \in \{2, 4, 8\}$) bifurcate their distributions, improving both average reward and the expected maximum reward.

\paragraph{Experiment 2: Complementary Hydrophobicity.} In therapeutic and biomaterials design, practitioners frequently need to generate complementary sequences, such as solubility tags, that perfectly counterbalance the properties of a target protein to maintain stability \cite{esposito2006enhancement, costa2014fusion}. As such, we evaluated the framework on a task requiring the model to generate a sequence with inverse hydrophobicity to an input ($h_{target} = 1.0 - h_{in}$). Standard RL clusters around a narrow, safe mean, failing to capture the full relational range. \Cref{fig:bon_distributions} (Right) shows that BoN training induces a "stretching" effect. The model learns to output a wider variance of candidates, increasing the probability that at least one sample will perfectly complement the input target. 

\begin{figure}[!h]
    \centering
    \begin{minipage}{0.48\textwidth}
        \centering
        \includegraphics[width=\linewidth]{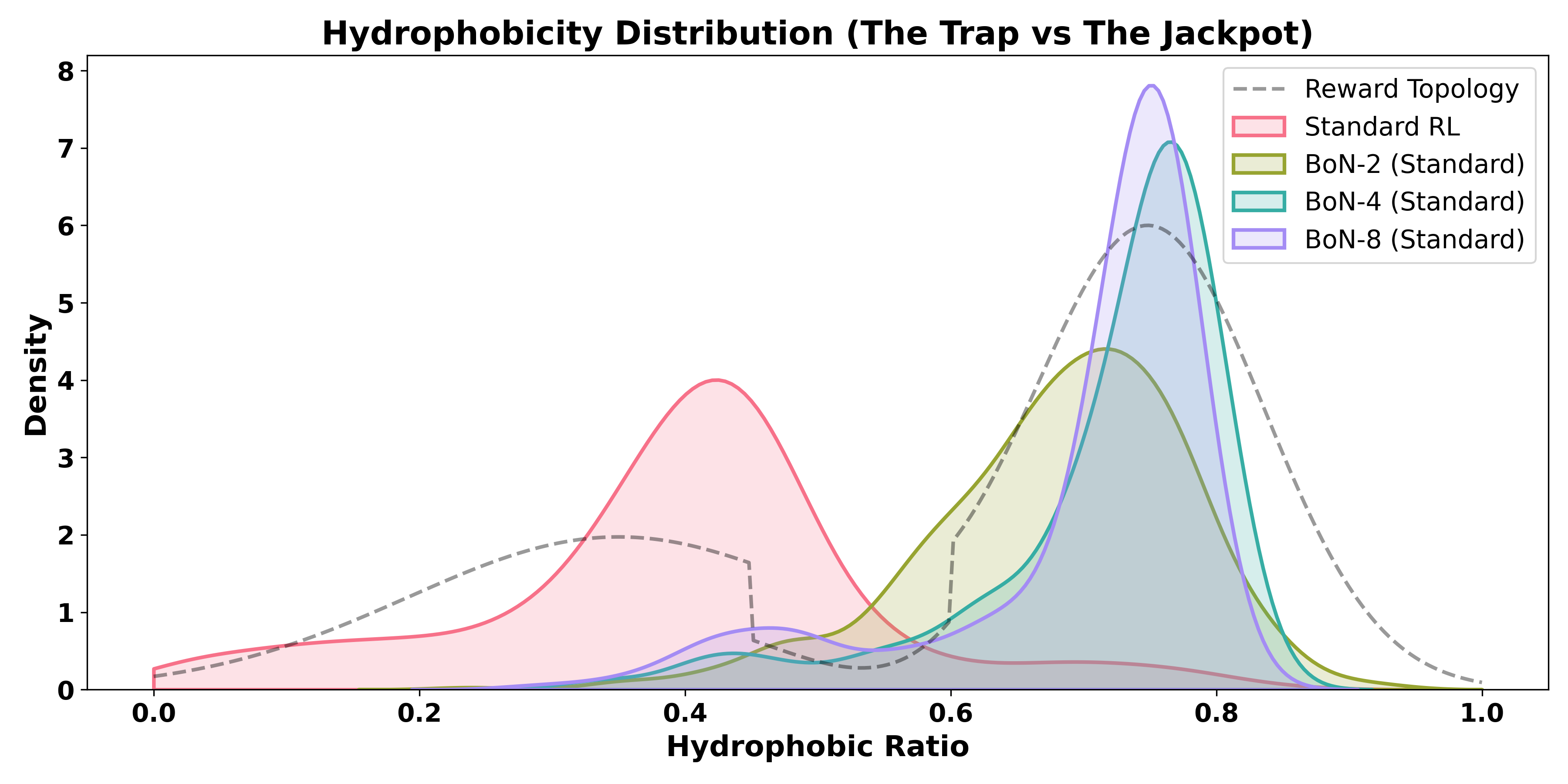}
        \caption{\textbf{Unconditional Shift.} Standard RL (Blue) is stuck in the local "Trap." BoN models (Green/Red) cross the "Valley" to reach the "Jackpot."}
        \label{fig:bon_density}
    \end{minipage}\hfill
    \begin{minipage}{0.4\textwidth}
        \centering
        \includegraphics[width=\linewidth]{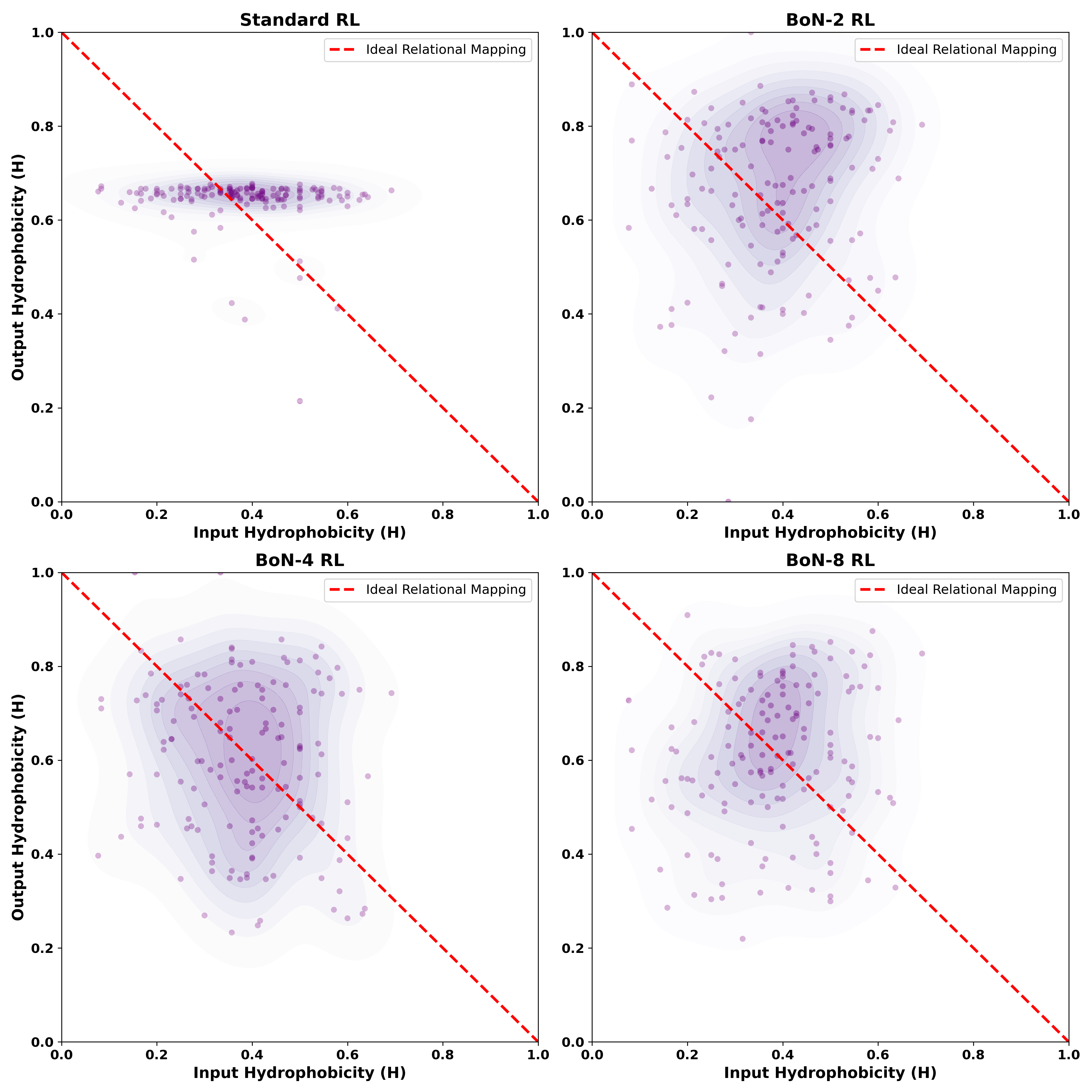}
        \caption{\textbf{Conditional Shift.} Density of Output vs. Input. BoN models stretch to align with the ideal identity line (red dashed) better than the clustered baseline.}
        \label{fig:bon_density_grid}
    \end{minipage}
    \label{fig:bon_distributions}
\end{figure}

\paragraph{Results.} \Cref{tab:bon_scaling}, \Cref{fig:bon_unconditional_scaling}, and \Cref{fig:bon_conditional_scaling} detail the expected maximum reward at varying test-time budgets. In both settings, Standard RL underperforms under test time scaling because its narrow distribution lacks the diversity required for a search process. The most informative result is in the more complex second experiment. While BoN-trained models pay a "Search Tax" (lower single-sample quality at $N=1$), they exhibit superior scaling curves, saturating near the theoretical maximums when allowed to search.

\begin{table}[h]
\centering
\caption{Expected Max Reward scaling. Standard errors (SE) are reported in parentheses.}
\label{tab:bon_scaling}
\vspace{0.5em}

\textbf{Exp 1: Unconditional}\par\vspace{0.2em}
\resizebox{\linewidth}{!}{%
\begin{tabular}{l|c|cccccc}
\toprule
\textbf{Model} & \textbf{$N=1$} & \textbf{$N=2$} & \textbf{$N=4$} & \textbf{$N=8$} & \textbf{$N=16$} & \textbf{$N=32$} & \textbf{$N=64$} \\ \midrule
Standard RL & 3.18 (0.15) & 4.20 (0.15) & 5.30 (0.14) & 6.54 (0.13) & 8.25 (0.11) & 10.06 (0.08) & 11.40 (0.05) \\ 
BoN-2 & 7.97 (0.12) & 10.16 (0.08) & 11.41 (0.05) & 11.93 (0.03) & 12.10 (0.01) & 12.15 (0.01) & 12.16 (0.01) \\
BoN-4 & 9.73 (0.09) & 11.33 (0.05) & \textbf{11.92} (0.03) & \textbf{12.10} (0.01) & \textbf{12.15} (0.01) & \textbf{12.16} (0.01) & \textbf{12.16} (0.01) \\ 
BoN-8 & \textbf{9.82} (0.09) & \textbf{11.53} (0.04) & 12.04 (0.02) & 12.14 (0.01) & 12.16 (0.01) & 12.16 (0.01) & 12.16 (0.01) \\ \bottomrule
\end{tabular}%
}

\vspace{1.5em} 

\textbf{Exp 2: Conditional Generation}\par\vspace{0.2em}
\resizebox{\linewidth}{!}{%
\begin{tabular}{l|c|cccccc}
\toprule
\textbf{Model Strategy} & \textbf{$N=1$} & \textbf{$N=2$} & \textbf{$N=4$} & \textbf{$N=8$} & \textbf{$N=16$} & \textbf{$N=32$} \\ \midrule
Standard RL & \textbf{7.26} (0.24) & \textbf{7.70} (0.23) & 8.02 (0.21) & 8.31 (0.20) & 8.60 (0.18) & 8.89 (0.16) \\ 
BoN-2 & 5.36 (0.28) & 7.64 (0.23) & 8.84 (0.17) & 9.48 (0.12) & \textbf{9.86} (0.08) & \textbf{9.96} (0.06) \\ 
BoN-4 & 5.82 (0.27) & 7.49 (0.24) & \textbf{9.05} (0.16) & \textbf{9.61} (0.11) & 9.86 (0.08) & 9.94 (0.06) \\ 
BoN-8 & 5.67 (0.27) & 7.55 (0.23) & 8.96 (0.16) & 9.40 (0.13) & 9.80 (0.09) & 9.94 (0.06) \\ \bottomrule
\end{tabular}%
}
\end{table}

\begin{figure}[!h]
    \centering
    \begin{minipage}{0.48\textwidth}
        \centering
        \includegraphics[width=\linewidth]{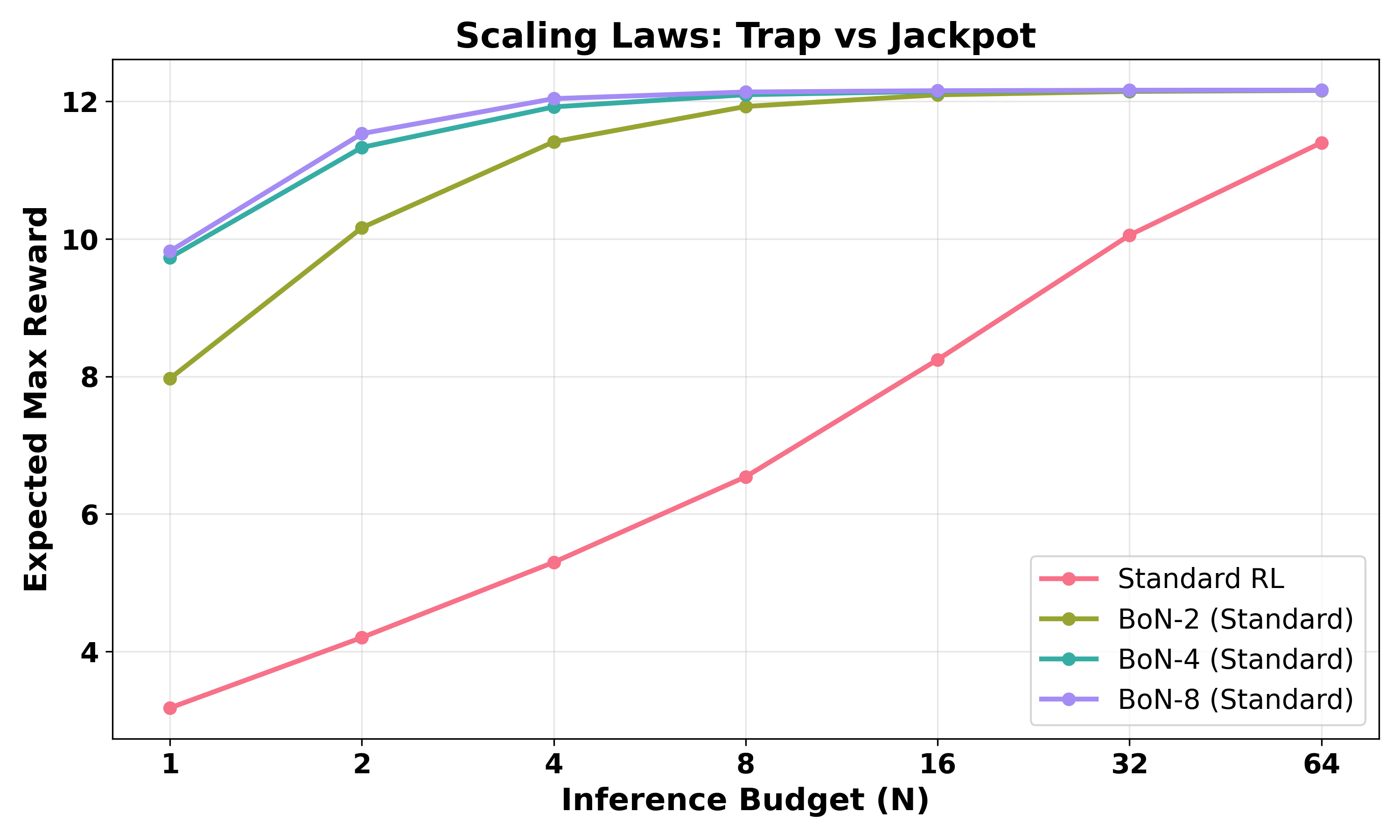}
        \caption{\textbf{Unconditional Scaling.} Models trained with BoN objectives exhibit superior scaling behavior at test time.}
        \label{fig:bon_unconditional_scaling}
    \end{minipage}\hfill
    \begin{minipage}{0.48\textwidth}
        \centering
        \includegraphics[width=\linewidth]{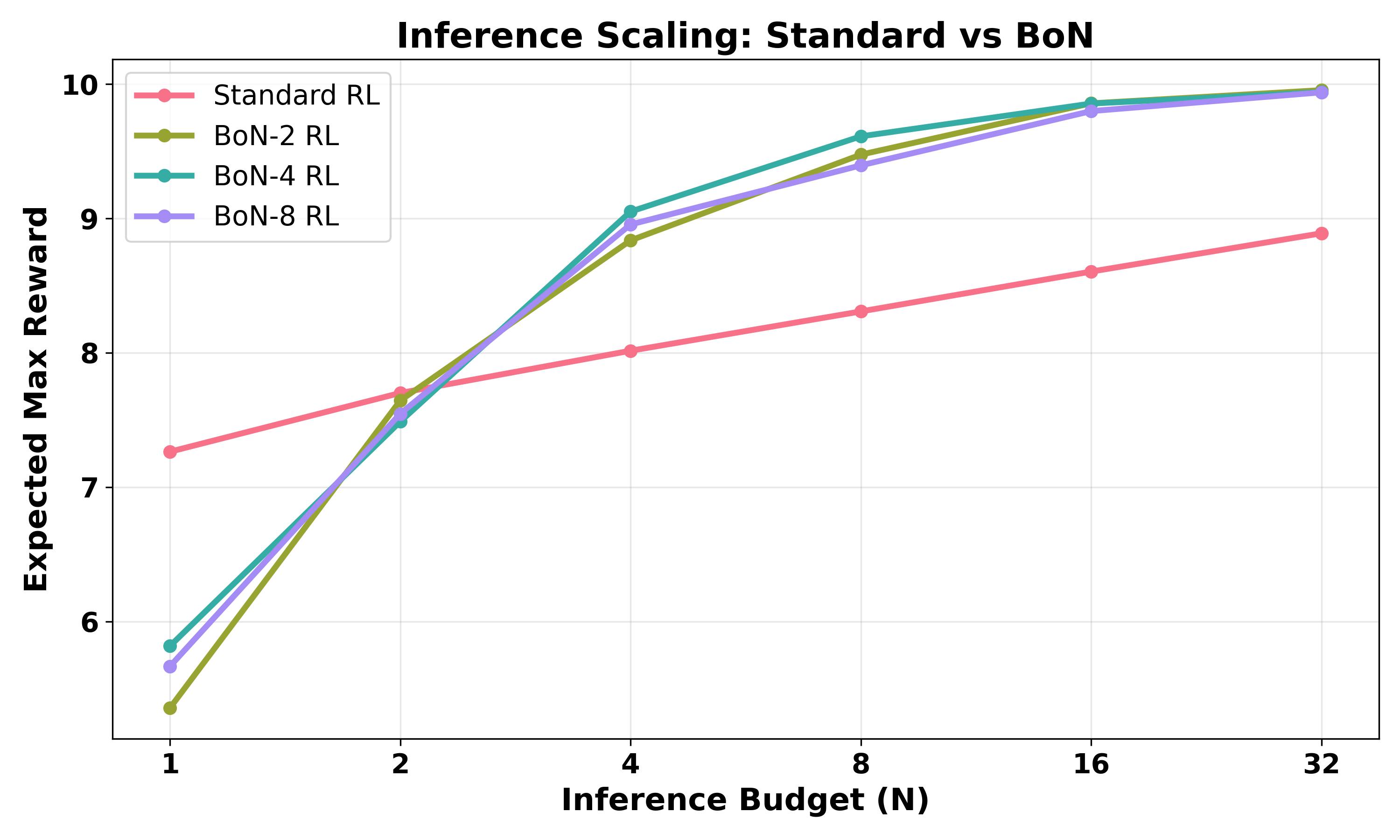}
        \caption{\textbf{Conditional Scaling.} BoN-aligned models exhibit superior test-time scaling properties compared to the Standard RL baseline.}
        \label{fig:bon_conditional_scaling}
    \end{minipage}
\end{figure}

\section{Discussion}

By formalizing inference as an operator, we show that training should be aligned with the distribution induced by the test-time procedure rather than the base policy. This addresses a fundamental mismatch in modern generative modeling: while deployment generates multiple outputs, standard objectives optimize single-output performance. CAT provides a simple and efficient mechanism to correct this mismatch without increasing training cost. Across SFT and RL, we observe that alignment with test-time strategies consistently improves scaling with inference compute. Furthermore, results on protein generation demonstrate that this principle extends beyond language models. Together, these findings suggest that training objectives should be designed with the downstream inference procedure in mind, especially as test-time compute becomes increasingly central to enhancing model capability. Conceptualizing inference as an operator provides a principled mechanism to achieve this goal.


\paragraph{Limitations}
While CAT provides a theoretically sound bridge between training and deployment, it introduces specific optimization challenges. First, as analyzed in our variance analysis (\Cref{app:sft_variance} and \Cref{app:rl_variance}), strategy-aware gradients can exhibit severe variance inflation, which can degrade performance. Developing principled variance reduction techniques for these objectives is a promising direction for future work. Second, our framework relies on the \textit{diagonal approximation} of the effective test-time probability. While we show this approximation is well-justified in the settings we consider (\Cref{app:scalar_justification}), it may underestimate the full learning signal. Extending the framework to account for cross-output dependencies, or understanding the benefits of allocating additional compute to approximate the full gradient, remains an important avenue for further study. More broadly, we have focused on a set of relatively simple test-time strategies. A natural extension is to apply this operator-based alignment to more complex, multi-step search algorithms, such as Genetic Algorithms \cite{holland1975adaptation} or Monte Carlo Tree Search \cite{kocsis2006bandit}, where $\tilde{\pi}_\theta$ may not admit a simple functional dependence on $\pi_\theta$ and $\phi$ or more complex agentic settings \cite{kim2026scalingtesttimecomputeagentic} for which functional forms may be impossible to derive whatsoever. Finally, understanding how CAT affects the underlying generation process during RL (\Cref{app: improve_pass@1}) presents an additional direction for future work.

\section*{Acknowledgments}

We thank Feng Chen for early discussions about this work. This research used NCSA Delta's GPU cluster, made available through allocation CIS260915 (``Directly Optimizing Generative Models For Search'') from the National Science Foundation's Advanced Cyberinfrastructure Coordination Ecosystem: Services \& Support (ACCESS) program. 

\newpage
\bibliographystyle{unsrtnat}
\bibliography{reference}
\newpage

\appendix
\section{Derivations and Explanations for SFT}
\label{app:sft}
This section adds detail to the derivations for the SFT CAT losses, as well as some additional useful theoretical results about these factors, including an analysis of the accuracy of the diagonal approximation. We generally presume SFT aims to minimize standard Cross-Entropy loss ($\Lce$), defined as the negative log-likelihood of the ground truth labels under the base policy:
\begin{equation}
    \Lce(\theta) = -\mathbb{E}_{(x, y) \sim \mathcal{D}} \left[ \log \pi_\theta(y \mid x) \right]
\end{equation}
We define a Test-Time Strategy (TTS) as a functional operator $\mathcal{T}$ applied to the base policy $\pi_\theta$. Let $\tilde{\pi}_\theta(y \mid x) = \mathcal{T}(\pi_\theta, \phi)(y \mid x)$ denote the \textit{system-level probability} that the generative model ultimately yields the target output $y$ given the input $x$ after the TTS is applied. The term $\phi$ represents the hyperparameters of the strategy, such as the sample budget $N$ in Pass@$N$.

We define the \textit{Compute Aligned Training} (CAT) loss ($\Lcat$) as the negative log-likelihood of the ground truth data under this effective system-level policy:
\begin{equation}
    \Lcat(\theta) = -\mathbb{E}_{(x, y) \sim \mathcal{D}} \left[ \log \tilde{\pi}_\theta(y \mid x) \right] = -\mathbb{E}_{(x, y) \sim \mathcal{D}} \left[ \log \left( \mathcal{T}(\pi_\theta, \phi)(y \mid x) \right) \right]
\end{equation}

Note that a very similar setup would work for loss functions other than cross-entropy, but we limit our scope to cross-entropy for simplicity.

\subsection{Gradient of $\Lcat$} \label{app:sft_gradient}
We derive the gradient of the new objective $\Lcat$ with respect to the model parameters $\theta$.

Let $\tilde{p} = \tilde{\pi}_{\theta}(y^* \mid x)$ denote the probability of success (generating the ground truth $y^*$) under the test-time strategy. $\tilde{p}$ is a function of the entire policy vector $\pi_\theta(\cdot \mid x)$.

Applying the chain rule over the output space $\mathcal{Y}$:
\begin{align}
    \nabla_\theta \Lcat &= - \nabla_\theta \log \tilde{p} \nonumber \\
    &= - \frac{1}{\tilde{p}} \nabla_\theta \tilde{p} \nonumber \\
    &= - \frac{1}{\tilde{p}} \sum_{y' \in \mathcal{Y}} \frac{\partial \tilde{p}}{\partial \pi_\theta(y' \mid x)} \nabla_\theta \pi_\theta(y' \mid x)
\end{align}

The above equation presents a practical challenge. In the standard SFT regime, we efficiently compute $\pi_\theta(y^* \mid x)$ (the ground truth probability), but we do not have access to the probabilities of arbitrary alternative outputs $\pi_\theta(y' \mid x)$ without sampling.

We briefly touch on ways to incorporate some of these missing diagonal elements in \Cref{subsection: no_fp}. However, we assume that the majority of the change in probability of success is accounted for by the change in the probability of outputting the correct answer. Theoretical justification is given in \Cref{sec:diagonal_sensitivity}. Furthermore, take $\tilde{p}$ to be well-approximated as a function of $p =\pi_\theta(y^* \mid x)$ and the sum to be $0$ for terms where $y' \neq y^*$.

 Applying these assumptions, we get:
\begin{align}
    \nabla_\theta \Lcat &\approx - \frac{1}{\tilde{p}} \frac{\partial \tilde{p}}{\partial p} \nabla_\theta p \nonumber \\
    & = - \frac{p}{\tilde{p}} \frac{\partial \tilde{p}}{\partial p} \nabla_\theta log(p) \nonumber
    \\
    & = \underbrace{\left( \frac{p}{\tilde{p}} \frac{\partial \tilde{p}}{\partial p} \right)}_{w(p, \phi)} \cdot \nabla_\theta \Lce
\end{align}

The Test-Time Aware update is simply the standard gradient scaled by a scalar coefficient $w(p, \phi)$, keeping the computational cost of training comparable to standard SFT.

\subsection{Deriving $w$ for Pass@$N$} \label{app:sft_passn} \label{app:sft_strategies}
To preserve the mathematical rigor of the operator mapping to the probability simplex $\Delta^{|\mathcal{Y}|-1}$ defined in \Cref{sec:method_tts_operator}, we conceptualize the test-time distribution for Pass@$N$ as an augmented categorical distribution. Because Pass@$N$ is a strict success/failure metric, we assign the effective probability of finding the correct answer $y^*$ to $\tilde{p}$. To ensure the distribution strictly sums to 1, we consolidate all remaining probability mass into a single "no correct answer found" state, which naturally receives the remaining probability $1 - \tilde{p}$.

It is worth noting that this instance of CAT is Direct Coverage Optimization from \citet{chen2025rethinking}. We credit their ideas for inspiring this work.

The hyperparameter $\phi$ is the sample budget $N$. Pass@$N$ samples $N$ independent solutions from the base model $\pi_\theta$. The strategy is successful if \textit{at least one} of the $N$ samples is correct.

Let $p = \pi_\theta(y^* \mid x)$ be the probability that a single sample is correct (i.e., Pass@1). The probability of generating an incorrect answer is $1 - p$. The probability that \textit{all} $N$ samples fail is $(1 - p)^N$. Therefore, the effective probability of success, $\tilde{p}$, is:
\begin{equation}
    \tilde{p} = 1 - (1 - p)^N
\end{equation}

Under our augmented distribution, the probability of the "no correct answer found" state is exactly $(1 - p)^N$.

We first compute the partial derivative of $\tilde{p}$ with respect to the base probability $p$:
\begin{equation}
    \frac{\partial \tilde{p}}{\partial p} = \frac{\partial}{\partial p} \left[ 1 - (1 - p)^N \right] = - \left( N(1 - p)^{N-1} \cdot (-1) \right) = N(1 - p)^{N-1}
\end{equation}
Recall the definition of the scaling factor derived in~\Cref{app:sft_gradient}:
\[
    w(p, N) = \frac{p}{\tilde{p}} \frac{\partial \tilde{p}}{\partial p}
\]
Substituting our expressions for $\tilde{p}$ and $\frac{\partial \tilde{p}}{\partial p}$, we obtain the closed-form scaling coefficient for Pass@$N$:
\begin{align}
    w(p, N) &= \frac{p}{1 - (1 - p)^N} \cdot N(1 - p)^{N-1} \nonumber \\
    &= \frac{N p (1 - p)^{N-1}}{1 - (1 - p)^N}
\end{align}
 For the trivial case where $N=1$, the term simplifies to $w(p, 1) = 1$, recovering the standard cross-entropy objective.

The behavior of the scaling factor $w_{\text{pass}}$ provides intuitive insight into why this objective prevents model overconfidence. It is monotonically decreasing as $p \rightarrow 1$.
     
\textbf{Intuition:} The Pass@$N$ objective acts as a natural regularizer. It signals to the optimizer: \textit{"The model is already likely enough to generate the correct answer within $N$ tries; stop updating parameters for this sample and focus on harder examples."} This prevents the model from wasting capacity on memorizing "easy" prompts. A more thorough discussion of the meaning of this weighing factor can be found in \cite{chen2025rethinking}.

\subsection{Majority Vote in SFT} \label{app:sft_maj}

Coming from the work of \citet{wang2023selfconsistency}, Majority Vote is a plurality vote: the correct answer $y^*$ is selected if it appears more frequently than any specific incorrect answer $y'$. Thus $\tilde{p}$ clearly depends on more than just $p$ but on the full distribution of $\pi_{\theta}(\cdot|x)$. To render this tractable and dependent solely on the scalar $p$, we approximate the strategy using a fixed threshold relaxation.

We define success as the correct answer appearing at least $k$ times in $N$ samples. In truth, this $k$ is a function of the distribution of incorrect answers. By setting a fixed threshold, so long as a single incorrect answer doesn't appear more than $k$ times out of $N$ draws, this objective will cause the majority vote to get the right answer with less overconfidence than what SFT would give it. $k = \lfloor N/2 \rfloor + 1$ serves as a strict upper bound on $k$ on what this threshold could be, so it serves as a worst case for what we should set $k$ to. The count of correct answers follows a Binomial distribution $n_{y^*} \sim \text{Bin}(N, p)$. The effective probability is the regularized incomplete beta function:
\begin{equation}
    \tilde{p}_{\text{maj}} = \Pr(n_{y^*} \ge k) = \sum_{i=k}^{N} \binom{N}{i} p^i (1-p)^{N-i}
\end{equation}
The derivative of this cumulative sum with respect to $p$ is proportional to the probability mass at the decision boundary $k$. The scaling factor becomes:
\begin{equation}
    w_{\text{maj}}(p, N, k) = \frac{k \binom{N}{k} p^k (1-p)^{N-k}}{\sum_{i=k}^{N} \binom{N}{i} p^i (1-p)^{N-i}}
\end{equation}

\paragraph{Interpretation:}
The numerator represents the probability of being \textit{exactly} at the threshold of consensus (having $k$ correct answers). The denominator is the total probability of having achieved consensus.
\begin{itemize}
    \item If the model is far from consensus ($p \ll k/N$), the gradient is boosted effectively to push the mass over the decision boundary.
    \item If the model has achieved stable consensus ($p \gg k/N$), the denominator approaches 1 and the numerator approaches 0, annealing the gradient to preserve the solution.
\end{itemize}

\subsection{Best-of-$N$ SFT}
In an SFT setting with nonbinary scores, there's many ways data can be presented to use. In the setting where we are always given a ground truth best possible answer, the problem reduces to Pass@$N$, so we do not provide any experiments or derivations for this. The scope of this paper will not go into all of the different ways nonbinary reward data can be presented for fine tuning.

\subsection{Reasoning Traces and Data Constraints}
The derivations in the prior subsections operate under the simplified assumption that $\pi_\theta(y \mid x)$ represents the probability of generating the correct final answer $y$. However, in modern reasoning models utilizing COT, the generation process is more complex. The model produces a reasoning trace $z$ (a sequence of intermediate tokens) followed by a final answer $y$.

Formally, the probability of generating a specific answer $y$ is the marginal probability over all possible latent reasoning paths $\mathcal{Z}$:
\begin{equation}
    \pi_\theta(y \mid x) = \sum_{z \in \mathcal{Z}} \pi_\theta(y, z \mid x) = \sum_{z \in \mathcal{Z}} \pi_\theta(y \mid z, x) \pi_\theta(z \mid x)
\end{equation}
Test-time strategies like Majority Vote or Pass@N are concerned exclusively with this marginal probability $\pi_\theta(y \mid x)$. We care whether the final answer is correct, regardless of the specific reasoning path taken.

This distinction introduces a practical difficulty for Supervised Fine-Tuning. In a standard SFT setting, our training dataset $\mathcal{D}$ typically consists of triplets $(x, z^*, y^*)$, where $z^*$ is a single "golden" reasoning trace provided by human annotators or synthetic data generation.

We do not have access to the full marginal distribution $\pi_\theta(y^* \mid x)$ during training, as marginalizing over the exponentially large space of all possible reasoning paths $\mathcal{Z}$ is computationally intractable. Instead, since $\pi_\theta(y^*, z^* \mid x) \le \pi_\theta(y^* \mid x)$, the likelihood of the provided trace serves as a tractable lower bound for the true marginal probability. In practice, we use this single-trace likelihood as a direct proxy for $p$ when computing the scaling factor $w(p, \phi)$. This approximation allows us to apply test-time aligned gradient scaling efficiently without the need for expensive sampling or marginalization over the latent space. This can be interpreted as training the model for the specific reasoning trace $(z^*, y^*)$ to be selected as the TTS. This is what is done for standard SFT under Pass@$1$ and it works empirically, so we stick to this formulation to avoid sampling.

Furthermore, in practice, we begin by warming up the model on classic SFT. This is useful because it teaches the model to associate the provided reasoning trace $z^*$ with the correct answer $y^*$ before strategy-aware modulation takes over. In the low-probability regime ($p \to 0$), the scaling factors typically a constant (e.g., $w(p) \approx 1$ for Pass@$N$), rendering the CAT gradient directionally equivalent to the standard SFT gradient. Warmup leverages this alignment to 'ground' the reasoning process, ensuring that subsequent strategy-aware updates optimize a policy that generates the answer via a valid chain of thought, rather than exploiting spurious correlations in the marginal distribution.
\subsection{Variance Analysis: The Variance Inflation Factor}
\label{app:sft_variance}
To quantify the stability of the CAT objective, we define the Variance Inflation (VI) factor. Because the strategy-aware SFT gradient is formulated as $\nabla_\theta \mathcal{L}_{\text{CAT}} = w(p, \phi) \nabla_\theta \mathcal{L}_{\text{CE}}$, we can cleanly isolate the impact of our objective on optimization stability. We measure this via the ratio of the squared gradient norms for a specific sample $x$:

\begin{equation}
    \text{VI}(x) = \frac{\lVert \nabla_\theta \Lcat \rVert^2}{\lVert \nabla_\theta \Lce \rVert^2} = |w(p, \phi)|^2 = \left| \frac{p}{\tilde{p}} \frac{\partial \tilde{p}}{\partial p} \right|^2
\end{equation}

To identify potential stability risks, we focus our analysis on the asymptotic behavior of the VI for "hard" samples, taking the limit as the base probability approaches zero ($p \to 0$). We examine this specific limit because as models master "easy" samples ($p \to 1$), strategy-aware scaling factors naturally decay to zero (or remain strictly bounded), suppressing gradient variance. Optimization catastrophes in this framework almost exclusively originate from the aggressive up-weighting of highly uncertain predictions.

\paragraph{Pass@$N$: Stability.}
For Pass@$N$, $\tilde{p} = 1 - (1-p)^N$. The scaling factor is:
\begin{equation}
    w_{\text{pass}}(p) = \frac{p \cdot N(1-p)^{N-1}}{1 - (1-p)^N}
\end{equation}
Analyzing the limit as $p \to 0$ using the first-order Taylor expansion $(1-p)^N \approx 1 - Np$:
\begin{equation}
    \lim_{p \to 0} F_{\text{pass}}(p) \approx \frac{p \cdot N(1)}{1 - (1 - Np)} = \frac{Np}{Np} = 1
\end{equation}
\textbf{Conclusion:} Pass@$N$ SFT is inherently stable. The VI approaches 1 for hard samples and decays to 0 for easy samples. It does not introduce variance explosion.

\paragraph{Majority Vote: Instability.}
For Majority Vote, we approximate the success probability near $p \to 0$ using the leading term of the Binomial sum (where $k \approx N/2$ is the threshold):
\begin{equation}
    \tilde{p} \approx \binom{N}{k} p^k (1-p)^{N-k} \implies \frac{\partial \tilde{p}}{\partial p} \approx k \binom{N}{k} p^{k-1}
\end{equation}
Substituting this into the definition of $w$:
\begin{equation}
    \lim_{p \to 0} w_{\text{maj}}(p) = \frac{p}{\binom{N}{k} p^k} \cdot k \binom{N}{k} p^{k-1} = \frac{p \cdot k \cdot p^{k-1}}{p^k} = k
\end{equation}
\textbf{Conclusion:} For Majority Vote, the gradient for hard samples is boosted by a factor of $k \approx N/2$. The variance scales quadratically:
\begin{equation}
    \text{VI}_{\text{maj}} \approx \left( \frac{N}{2} \right)^2 = \frac{N^2}{4}
\end{equation}
This theoretical bound confirms that Majority Vote SFT is highly prone to variance explosion for hard examples as $N$ increases.

\subsection{Beyond the Diagonal Assumption: Token-Level Dependencies}
\label{subsection: no_fp}

The derivation of $w(p, \phi)$ in the previous sections relies on the simplifying assumption that the effective test-time probability $\tilde{p}$ is a function of the ground truth probability $p = \pi_\theta(y^* | x)$, instead of all of $\pi_\theta(\cdot|x)$. While this holds for Pass@$N$, it is an approximation for Majority Vote.

In a Majority Vote setting, the probability of the correct answer $y^*$ achieving a consensus depends not just on its own likelihood, but on the distribution of the \textit{incorrect} mass. A scenario where the remaining probability mass $1-p$ is concentrated on a single "strong rival" answer $y'$ is fundamentally different from a scenario where it is dispersed uniformly over the vocabulary. To address this without the computational cost of full Monte Carlo resampling during training, we propose examining the logits at the token level.

\subsubsection{The Greedy Bound via Token Margins}
We posit that the strength of the "rival" consensus can be bounded by analyzing the decision margin at each generation step. Let $y^* = (y^*_1, \dots, y^*_T)$ be the ground truth sequence. At any timestep $t$, let $z_t \in \mathbb{R}^{|V|}$ denote the logits where $V$ is the vocabulary. We define the \textit{token margin} $\delta_t$ as the distance between the logit of the correct token and the most likely incorrect token (the "greedy rival"):

\begin{equation}
    \delta_t = z_t[y^*_t] - \max_{v \neq y^*_t} z_t[v]
\end{equation}

This margin acts as a proxy for the local stability of the generation. A large positive $\delta_t$ implies that the model is locally robust against deviating to an alternative path. Conversely, a small or negative $\delta_t$ identifies a "weak link" where a rival answer could plausibly capture the search process.

We can approximate the probability ratio between the ground truth $y^*$ and the most likely rival sequence $y'$ via the minimum margin over the sequence:
\begin{equation}
    \frac{P(y')}{P(y^*)} \approx \exp\left( -\min_{t} \delta_t \right)
\end{equation}
This formulation allows us to estimate parameters such as the effective consensus difficulty without explicit sampling. For example, if the minimum margin is small, the rival is strong, implying that a higher effective count $k$ is required to ensure $y^*$ dominates the vote.

\subsubsection{Contrastive Margin Maximization}
To directly suppress the probability of the rival answer $y'$, we introduce a contrastive auxiliary loss. Rather than simply maximizing the likelihood of $y^*$ (which can be satisfied even if a rival $y'$ is also highly likely), we explicitly maximize the gap between the truth and the nearest incorrect trajectory.

We define the Contrastive Majority Loss as:
\begin{equation}
    \mathcal{L}_{\text{Contrast}} = \frac{1}{T} \sum_{t=1}^T \max(0, \eta - \delta_t)
\end{equation}
where $\eta$ is a target margin hyperparameter. By minimizing this objective, we force the model to not only generate the correct answer but to do so with sufficient distance from the "second-best" hypothesis. This effectively "clears the field" for the Majority Vote operator, ensuring that $p(y^*)$ is dominant relative to the most probable incorrect aggregate $\max_{y \neq y^*} p(y)$.

This is just a heuristic meant to show that a user can consider the full distribution, but practically we still recommend the strategy used in \Cref{app:maj_implementation}.

\subsubsection{Empirical Validation of Token-Level Margins}

To validate the theoretical benefit of addressing token-level dependencies and suppressing rival probability mass, we conducted a preliminary, small-scale proof-of-concept experiment on a subset of the MATH benchmark. We compared a standard SFT baseline against both our standard CAT Majority Vote implementation and a hybrid objective that incorporates Contrastive Margin Maximization.

\paragraph{Experimental Setup}
We utilized the 4-bit quantized \texttt{Mistral-7B-Instruct-v0.2} model. To ensure a stable initialization and isolate the effect of the strategy-aware losses, all models underwent a shared warmup phase consisting of one epoch of standard cross-entropy training. Following warmup, the training branched into three distinct one-epoch regimes:
\begin{enumerate}
    \item \textbf{SFT Baseline:} Continued standard cross-entropy training.
    \item \textbf{CAT Majority Vote:} Training using the CAT objective.
    \item \textbf{Pairwise Contrastive (Hybrid):} A joint objective combining standard SFT with a token-level contrastive loss, weighted equally ($0.5 \cdot \Lce + 0.5 \cdot \mathcal{L}_{\text{Contrast}}$). 
\end{enumerate}

\paragraph{Contrastive Implementation}
Instead of explicitly calculating the full margin over the sequence, we implemented a computationally efficient pairwise approximation. For each valid token in the ground truth sequence, we sampled $M=4$ alternative "rival" tokens from the model's current probability distribution (masking out the ground truth). Let $z^*$ be the logit of the correct token, and $z_{\text{bad}}$ be the logit of a sampled rival. The contrastive loss maximizes the difference between these logits using a log-sigmoid function:
\begin{equation}
    \mathcal{L}_{\text{Contrast}} = - \frac{1}{M} \sum_{i=1}^M \log \sigma\left(z^* - z_{\text{bad}}^{(i)}\right)
\end{equation}
This objective explicitly pushes down the probability mass of the most likely incorrect paths, "clearing the field" to ensure the ground truth answer dominates the inference-time voting process. 

\paragraph{Results and Discussion}
We evaluated the models by generating $N=64$ candidate solutions per problem. The results are summarized in \Cref{tab:token_level_results}.

\begin{table}[h]
\centering
\caption{\textbf{Token-Level Contrastive Loss Results ($N=64$).} A small-scale proof of concept demonstrating that explicitly suppressing rival tokens improves the aggregation gain at test time.}
\label{tab:token_level_results}
\begin{tabular}{@{}lccc@{}}
\toprule
\textbf{Model Strategy} & \textbf{Pass@1} & \textbf{Majority Vote} & \textbf{Aggregation Gain} \\ \midrule
Standard SFT & 12.90\% & 16.20\% & +3.30\% \\
CAT Majority Vote & \textbf{13.03\%} & 17.60\% & +4.57\% \\
Pairwise Contrastive & 12.47\% & \textbf{18.20\%} & \textbf{+5.72\%} \\ \bottomrule
\end{tabular}
\end{table}

The results align well with our theoretical framework. The CAT Majority Vote model successfully improves both single-sample precision ($13.03\%$) and overall consensus performance ($17.60\%$) compared to the standard SFT baseline. 

Crucially, the Pairwise Contrastive model highlights the specific dynamics of test-time search. It experiences a minor degradation in single-sample precision compared to the baseline (Pass@1 drops to $12.47\%$), demonstrating the "search tax" where optimizing for aggregation slightly degrades the mean trajectory. However, this tradeoff yields a distribution far better suited for aggregation. It achieves the highest absolute Majority Vote accuracy ($18.20\%$), and its "Aggregation Gain" (the absolute performance delta between Pass@1 and Majority Vote) nearly doubles the baseline, reaching $5.72\%$. 

While this remains a limited-scale proof of concept, it successfully supports the hypothesis that standard cross-entropy leaves test-time performance on the table by ignoring the distribution of incorrect answers. Actively penalizing strong rival tokens ensures that when the model generates the correct reasoning trace, it does so with enough local stability to win a plurality vote.

\subsection{Sensitivity Analysis of the Diagonal Approximation}
\label{sec:diagonal_sensitivity}

In \Cref{sec:method_tts_operator}, we introduced the diagonal approximation to avoid the computational intractability of evaluating the full Jacobian over the entire output space $\mathcal{Y}$. The previous subsection described a heuristic used to compensate the approximation. Here, we formally study the error introduced by this approximation.

Let the exact gradient of our objective be $g_{\text{exact}} = \nabla_{\theta}\mathcal{L}_{CAT}$, and our diagonal approximation be $g_{\text{diag}} = w(p, \phi) \nabla_{\theta}\mathcal{L}_{CE}$. The exact gradient can be decomposed into the diagonal term and an off-diagonal error vector $\epsilon_{\text{vec}}$:
\begin{equation}
    g_{\text{exact}} = g_{\text{diag}} + \epsilon_{\text{vec}} \quad \text{where} \quad \epsilon_{\text{vec}} = -\frac{1}{\tilde{p}} \sum_{y' \in \mathcal{R}} \frac{\partial \tilde{p}}{\partial \pi_{\theta}(y'|x)} \nabla_{\theta}\pi_{\theta}(y'|x)
\end{equation}
where $\mathcal{R} = \{y' \in \mathcal{Y} \mid y' \neq y^*\}$ is the set of incorrect (rival) sequences.

Throughout this analysis, we work with a fixed prompt-target pair $(x,y^*)$, so $\mathcal{L}_{CAT}=-\log\tilde{p}$ and $\mathcal{L}_{CE}=-\log p$. We assume $0<p<1$, $\tilde{p}>0$, and $\partial\tilde{p}/\partial p\ge0$. The diagonal coefficient $w=(p/\tilde{p})(\partial\tilde{p}/\partial p)$ is defined using the exact strategy success probability and may depend on the full policy. Replacing that success probability by the fixed-threshold surrogate introduces an additional approximation not controlled by the off-diagonal residual alone.

To determine how optimizing with our diagonal approximation $g_{\text{diag}}$ affects the true objective $\mathcal{L}_{CAT}$, we must analyze the objective's behavior under a standard gradient descent update. Let the model parameters at step $t$ be updated using the approximate gradient with a sufficiently small learning rate $\eta > 0$, such that $\theta_{t+1} = \theta_t - \eta g_{\text{diag}}$. 

Assuming $\mathcal{L}_{CAT}$ is locally differentiable, we can apply a first-order Taylor expansion to evaluate the resulting change in the exact loss:
\begin{align}
    \mathcal{L}_{CAT}(\theta_{t+1}) &\approx \mathcal{L}_{CAT}(\theta_t) + \nabla_{\theta}\mathcal{L}_{CAT}(\theta_t)^T (\theta_{t+1} - \theta_t) \nonumber \\
    \mathcal{L}_{CAT}(\theta_{t+1}) &\approx \mathcal{L}_{CAT}(\theta_t) - \eta \langle g_{\text{exact}}, g_{\text{diag}} \rangle
\end{align}
We must first validate that this update is directionally aligned with our goal. For the parameter update to yield a strict decrease in the true objective ($\mathcal{L}_{CAT}(\theta_{t+1}) < \mathcal{L}_{CAT}(\theta_t)$), the inner product $\langle g_{\text{exact}}, g_{\text{diag}} \rangle$ must be strictly positive.
\begin{theorem} [Approximation Descent Condition]
    The diagonal approximation $g_{\text{diag}}$ guarantees a strict decrease in the exact loss $\mathcal{L}_{CAT}$ if the magnitude of the off-diagonal error vector is strictly bounded by the magnitude of the diagonal gradient: $\|\epsilon_{\text{vec}}\| < \|g_{\text{diag}}\|$.
\end{theorem}

\noindent\textit{Proof.} Expanding the inner product yields $\langle g_{\text{diag}}, g_{\text{diag}} + \epsilon_{\text{vec}} \rangle = \|g_{\text{diag}}\|^2 + \langle g_{\text{diag}}, \epsilon_{\text{vec}} \rangle$. By the Cauchy-Schwarz inequality, the worst-case adversarial perturbation occurs when $\epsilon_{\text{vec}}$ points in the exact opposite direction of $g_{\text{diag}}$. Setting $\|g_{\text{diag}}\|^2 - \|g_{\text{diag}}\|\|\epsilon_{\text{vec}}\| > 0$ yields the strict descent bound $\|\epsilon_{\text{vec}}\| < \|g_{\text{diag}}\|$. 

\vspace{0.5em}

While Theorem 2 establishes a worst-case bound for descent, we can derive a much stronger typical-case alignment by analyzing the structural properties of the inner product $\langle g_{\text{diag}}, \epsilon_{\text{vec}} \rangle$. 

Let us expand the off-diagonal error vector using the log-derivative trick, $\nabla_{\theta}\pi_{\theta}(y'|x) = \pi_{\theta}(y'|x) \nabla_{\theta}\log \pi_{\theta}(y'|x)$:
\begin{equation}
    \epsilon_{\text{vec}} = -\frac{1}{\tilde{p}} \sum_{y' \in \mathcal{R}} \frac{\partial \tilde{p}}{\partial \pi_{\theta}(y'|x)} \pi_{\theta}(y'|x) \nabla_{\theta}\log \pi_{\theta}(y'|x)
\end{equation}

Taking the inner product with the diagonal approximation $g_{\text{diag}} = -w(p, \phi) \nabla_{\theta}\log \pi_{\theta}(y^*|x)$ yields:
\begin{equation}
    \label{eq:inner_product_expanded}
    \langle g_{\text{diag}}, \epsilon_{\text{vec}} \rangle = \frac{w(p, \phi)}{\tilde{p}} \sum_{y' \in \mathcal{R}} \pi_{\theta}(y'|x) \frac{\partial \tilde{p}}{\partial \pi_{\theta}(y'|x)} \mathcal{C}_{\theta}(y^*, y')
\end{equation}
where $\mathcal{C}_{\theta}(y^*, y') = \langle \nabla_{\theta}\log \pi_{\theta}(y^*|x), \nabla_{\theta}\log \pi_{\theta}(y'|x) \rangle$.

To evaluate this inner product, we introduce the Proportional Decay Assumption: an increase in the ground truth probability $p = \pi_{\theta}(y^*|x)$ draws mass from rival answers in proportion to their current probabilities. This is an idealized description of probability redistribution, not a general property of autoregressive models. For the following bounds, we assume: 

\[
\nabla_{\theta}\pi_{\theta}(y'|x)
\approx -\frac{\pi_{\theta}(y'|x)}{1-p}\nabla_{\theta}p.
\]

Dividing by $\pi_{\theta}(y'|x)$ and applying the log-derivative trick, we establish a deterministic relationship between the log-gradients:

\[\nabla_{\theta}\log \pi_{\theta}(y'|x) \approx -\frac{p}{1-p} \nabla_{\theta}\log p\]

Substituting this relationship back into our definition of $\mathcal{C}_{\theta}(y^*, y')$ yields:

\[\mathcal{C}_{\theta}(y^*, y') \approx \left\langle \nabla_{\theta}\log p, -\frac{p}{1-p} \nabla_{\theta}\log p \right\rangle = -\frac{p}{1-p} \|\nabla_{\theta}\log p\|^2\]

This confirms our first structural property: under proportional decay, the inner product of the ground truth and rival log-gradients is strictly negative. The parameter updates inherently oppose each other. We can now substitute this back into the full expansion (Equation \ref{eq:inner_product_expanded}):

\[\langle g_{\text{diag}}, \epsilon_{\text{vec}} \rangle \approx -\underbrace{\left( \frac{w(p, \phi)}{\tilde{p}} \frac{p}{1-p} \|\nabla_{\theta}\log p\|^2 \right)}_{\text{Strictly Non-Negative Constant } C} \sum_{y' \in \mathcal{R}} \pi_{\theta}(y'|x) \frac{\partial \tilde{p}}{\partial \pi_{\theta}(y'|x)}\]

This formulation explicitly isolates the effect of the proportional decay assumption. Because the term inside the parentheses, $C$, is strictly non-negative, the leading negative sign dictates that the off-diagonal error aligns with our descent direction ($\langle g_{\text{diag}}, \epsilon_{\text{vec}} \rangle \ge 0$) \textit{if and only if} the remaining summation is non-positive. In other words, proportional decay structurally guarantees that the error vector points in the correct direction, up to the sign of the strategy's sensitivity to incorrect answers.

To resolve this final sign and guarantee directional alignment, we must categorize the behavior of the test-time strategy itself.

\begin{definition}[Competitive Strategy]A strategy is Competitive if increasing the probability of any incorrect answer strictly decreases (or leaves unchanged) the probability of success. That is, $\frac{\partial \tilde{p}}{\partial \pi_\theta(y' \mid x)} \le 0$ for all $y' \in \mathcal{R}$. Here, the partial derivatives are unconstrained: all other probability coordinates are held fixed.\end{definition}

For any competitive strategy, the partial derivative $\frac{\partial \tilde{p}}{\partial \pi_{\theta}(y'|x)}$ is non-positive. Because the leading negative sign from our proportional decay derivation cancels this non-positive sum, the entire inner product is strictly non-negative: $\langle g_{\text{diag}}, \epsilon_{\text{vec}} \rangle \ge 0$.

This synthesis of proportional decay and competitive strategy dynamics yields two powerful guarantees for our optimization process:
\begin{enumerate}
\item \textbf{Directional Alignment:} The ignored off-diagonal error actively points in the same general direction as our approximation ($\langle g_{\text{diag}}, \epsilon_{\text{vec}} \rangle \ge 0$). We are guaranteed not to step in the wrong direction.
\item \textbf{Conservative Magnitude:} Because the error vector aligns with our approximation, our calculated scalar weight strictly underestimates the true learning signal.
\end{enumerate}

\begin{theorem}[Directional Alignment and Conservative Bound]
    Assuming proportional decay and a competitive test-time strategy, the off-diagonal error aligns with the diagonal approximation ($\langle g_{\text{diag}}, \epsilon_{\text{vec}} \rangle \ge 0$). Consequently, the diagonal approximation safely underestimates the exact gradient magnitude: $\|g_{\text{diag}}\| \le \|g_{\text{exact}}\|$.
\end{theorem}

\paragraph{Implications for Optimization: The Conservative Gradient Guarantee.} 
Together, Theorem 2 and Theorem 3 establish the conditions required for the theoretical safety of the diagonal approximation. Theorem 2 dictates that for the approximation to decrease the true loss, the off-diagonal error must not overpower the diagonal gradient. Theorem 3 satisfies and exceeds this requirement by proving that under proportional decay, the error vector actually points in the \textit{same} general direction as our approximation ($\langle g_{\text{diag}}, \epsilon_{\text{vec}} \rangle \ge 0$). 

This yields a powerful conclusion: as long as a Test-Time Strategy is \textbf{Competitive}, meaning rival answers actively hinder the ground truth's chance of success, ignoring the off-diagonal terms does not send the optimizer in the wrong direction. Instead, it creates a \textit{conservative underestimate} of the true gradient ($\|g_{\text{diag}}\| \le \|g_{\text{exact}}\|$). In the context of deep reinforcement learning and fine-tuning, where large variance and overshooting are primary causes of policy collapse, this underestimation acts as a natural stabilizer. It guarantees that our computationally efficient scalar approximation is a safe, valid descent direction.

\subsection{Diagonal Approximation Analysis for Specific Strategies}
\label{app:scalar_justification}

In \Cref{sec:diagonal_sensitivity}, Theorems 2 and 3 established the geometric safety of the diagonal approximation. Specifically, we proved that for competitive strategies under proportional decay, the approximation actively aligns with the descent direction, guaranteeing that it strictly underestimates the true gradient magnitude ($\|g_{\text{diag}}\| \le \|g_{\text{exact}}\|$).

However, a conservative underestimate is only practically useful if it captures a meaningful portion of the true learning signal. In this section, we transition from high-dimensional vector geometry to a purely scalar analysis to explicitly quantify the magnitude of this off-diagonal error. By formalizing a taxonomy of inference strategies, we derive exact mathematical bounds on our scalar weights $w(p, \phi)$. This allows us to prove exactly how much of the true gradient's magnitude is preserved across the specific test-time strategies utilized in our experiments.

\subsubsection{The Unified Gradient Equation: From Vectors to Scalars}

Under proportional decay, we can compare the magnitudes of the diagonal and exact gradients through their scalar coefficients.

Recall the full expansion of the exact gradient from Section \ref{sec:diagonal_sensitivity}, partitioned into the ground truth probability $p = \pi_\theta(y^* \mid x)$ and the set of incorrect rivals $\mathcal{R}$:

\[g_{\text{exact}} = -\frac{1}{\tilde{p}} \left[ \frac{\partial \tilde{p}}{\partial p} \nabla_\theta p + \sum_{y' \in \mathcal{R}} \frac{\partial \tilde{p}}{\partial \pi_\theta(y' \mid x)} \nabla_\theta \pi_\theta(y' \mid x) \right]\]

To evaluate the summation, we apply the same Proportional Decay Assumption used to prove Theorem 3, substituting the rival gradient $\nabla_\theta \pi_\theta(y' \mid x) \approx -\frac{\pi_\theta(y' \mid x)}{1-p} \nabla_\theta p$:

\[g_{\text{exact}} \approx -\frac{1}{\tilde{p}} \left[ \frac{\partial \tilde{p}}{\partial p} \nabla_\theta p - \sum_{y' \in \mathcal{R}} \frac{\partial \tilde{p}}{\partial \pi_\theta(y' \mid x)} \frac{\pi_\theta(y' \mid x)}{1-p} \nabla_\theta p \right]\]

By factoring out $\nabla_\theta p$, we reveal a crucial structural simplification. Under proportional decay, the exact parameter-space gradient is strictly proportional to a single scalar bracket:

\[g_{\text{exact}} \approx -\frac{1}{\tilde{p}} \underbrace{\left[ \frac{\partial \tilde{p}}{\partial p} - \frac{1}{1-p}\sum_{y' \in \mathcal{R}} \pi_\theta(y' \mid x)\frac{\partial \tilde{p}}{\partial \pi_\theta(y' \mid x)} \right]}_{\text{Total Scalar Derivative } \frac{d\tilde{p}}{dp}} \nabla_\theta p\]

This bracketed term is the total derivative of the success probability, $\frac{d\tilde{p}}{dp}$, accounting for both the direct increase in the ground truth and the proportional decay of the rivals.

Since our diagonal approximation is defined purely by the first term ($g_{\text{diag}} = -\frac{1}{\tilde{p}} \frac{\partial \tilde{p}}{\partial p} \nabla_\theta p$), the relationship between the magnitude of our approximation and the exact gradient is entirely captured by this scalar equation. We formalize this as the \textbf{Unified Gradient Equation}:

\begin{equation}
\label{eq:unified_gradient}
\frac{d\tilde{p}}{dp} = \underbrace{\frac{\partial \tilde{p}}{\partial p}}_{\text{Direct Sensitivity}} - \underbrace{\frac{1}{1-p} \sum_{y' \in \mathcal{R}} \pi_\theta(y' \mid x) \frac{\partial \tilde{p}}{\partial \pi_\theta(y' \mid x)}}_{\text{Rivalry Correction } \mathcal{C}{\text{rival}}}
\end{equation}

To see how this connects to the previous section:

\[\epsilon_{\text{vec}} \approx \left( \frac{1}{\tilde{p}} \mathcal{C}_{\text{rival}} \right) \nabla_\theta p\]

By analytically bounding $\mathcal{C}_{\text{rival}}$ for specific test-time strategies, we can do more than just measure the gradient magnitude; we can bound the Optimization Efficiency of our training step.

In standard gradient descent with learning rate $\eta$, the expected first-order reduction in the true loss when taking a step with our diagonal approximation is $\Delta \mathcal{L}_{\text{diag}} \approx -\eta \langle g_{\text{exact}}, g_{\text{diag}} \rangle$. If we were to take a step with the exact, computationally intractable gradient, the corresponding reduction would be $\Delta \mathcal{L}_{\text{exact}} \approx -\eta \|g_{\text{exact}}\|^2$.

Because both $g_{\text{diag}}$ and $g_{\text{exact}}$ are strictly proportional to $\nabla_\theta p$ under the proportional decay assumption, the high-dimensional dot products cleanly factor out $\|\nabla_\theta p\|^2$. The relative efficiency of our approximation $\rho$ reduces to the simple ratio of the scalar derivatives:

\begin{equation}
\rho = \frac{\Delta \mathcal{L}_{\text{diag}}}{\Delta \mathcal{L}_{\text{exact}}} \approx \frac{\langle g_{\text{exact}}, g_{\text{diag}} \rangle}{\|g_{\text{exact}}\|^2} = \frac{\frac{\partial \tilde{p}}{\partial p}}{\frac{d\tilde{p}}{dp}} = \frac{\text{Direct Sensitivity}}{\text{Direct Sensitivity} - \mathcal{C}_{\text{rival}}}
\end{equation}

This ratio compares the first-order loss reductions of the diagonal and exact-gradient steps at the same learning rate. Values above one reflect a larger diagonal step. By bounding $\mathcal{C}_{\text{rival}}$, we can mathematically guarantee the worst-case optimization efficiency of CAT across different inference strategies.

\subsubsection{Strategy Classes}
\label{app:strategy_classes}

In the previous section, we established that Competitive strategies guarantee a conservative underestimate of the true gradient. To precisely quantify the optimization efficiency $\rho$ and bound the rivalry correction $\mathcal{C}_{\text{rival}}$, we introduce two further classifications based on a strategy's sensitivity to the distribution of incorrect answers $\boldsymbol{\pi}_{\mathcal{R}}$.

\begin{definition}[Orthogonal Strategy]
A strategy is Orthogonal if the success probability is completely invariant to the distribution of incorrect answers. That is, $\nabla_{\boldsymbol{\pi}_{\mathcal{R}}} \tilde{p} = \mathbf{0}$.
\end{definition}

\begin{definition}[Adversarial Symmetry]
A strategy exhibits Adversarial Symmetry if the sensitivity of the success probability to the strongest possible rival is bounded by its sensitivity to the ground truth. This means that changing the probability of the correct answer has more effect on test-time success than changing any other answer. That is:
\begin{equation}
    \max_{y' \in \mathcal{R}} \left| \frac{\partial \tilde{p}}{\partial \pi_\theta(y' \mid x)} \right| \le \frac{\partial \tilde{p}}{\partial p}
\end{equation}
\end{definition}

\subsubsection{Approximation Theorems}
\label{app:approximation_theorems}

We now apply these classifications to the Unified Gradient Equation to show exactly how much signal is captured, or lost, by the scalar approximation, directly connecting these bounds to the optimization efficiency $\rho$.

\begin{theorem}[Exactness for Orthogonal Strategies]
For any Orthogonal strategy, the total derivative exactly equals the partial derivative ($\frac{d\tilde{p}}{dp} = \frac{\partial \tilde{p}}{\partial p}$), resulting in perfect optimization efficiency ($\rho = 1$).
\end{theorem}
\begin{proof}
By Definition 2, the partial derivative with respect to any specific rival $y'$ is zero. Consequently, the rivalry correction term vanishes ($\mathcal{C}_{\text{rival}} = 0$), leaving only the direct sensitivity. Substituting this into our efficiency ratio yields $\rho = 1$. Pass@$N$ is a clear example of an Orthogonal strategy, as success depends exclusively on the presence of the ground truth.
\end{proof}

\begin{theorem}[The Bounded Gradient Theorem]
For any strategy that is both Competitive and Adversarially Symmetric, the true total derivative of the success probability is strictly bounded by the partial derivative and twice its value:
\[\frac{\partial \tilde{p}}{\partial p} \le \frac{d\tilde{p}}{dp} \le 2 \frac{\partial \tilde{p}}{\partial p}\]
Consequently, the optimization efficiency is mathematically guaranteed to remain above 50\% ($\rho \ge 0.5$) and safely not overestimate the gradient.
\end{theorem}
\begin{proof}
From the Unified Gradient Equation (\Cref{eq:unified_gradient}), we have $\frac{d\tilde{p}}{dp} = \frac{\partial \tilde{p}}{\partial p} - \mathcal{C}_{\text{rival}}$. 

\textbf{Lower Bound:} As established in Theorem 3, the partial derivative for any rival in a Competitive strategy is non-positive ($\frac{\partial \tilde{p}}{\partial \pi_\theta(y' \mid x)} \le 0$). Because probabilities are strictly non-negative, the entire rivalry correction term $\mathcal{C}_{\text{rival}} \le 0$. Therefore, $\frac{d\tilde{p}}{dp} \ge \frac{\partial \tilde{p}}{\partial p}$. This is the scalar realization of Theorem 3: the error aligns with our objective, proving the diagonal approximation is a safe, conservative underestimate.

\textbf{Upper Bound:} The magnitude $-\mathcal{C}_{\text{rival}}$ is a weighted average of the absolute rival sensitivities, with weights
$\pi_\theta(y' \mid x)/(1-p)$. Therefore,
\[
-\mathcal{C}_{\text{rival}}
\le
\max_{y'\in\mathcal{R}}
\left|
\frac{\partial\tilde{p}}
     {\partial\pi_\theta(y'\mid x)}
\right|
\le
\frac{\partial\tilde{p}}{\partial p},
\]
where the last inequality uses Adversarial Symmetry. The single-rival case (``The Duel'') gives the intuition for this bound. Thus, the total derivative cannot exceed $2\,\partial\tilde{p}/\partial p$.

Substituting these derivative bounds back into our efficiency ratio $\rho$ demonstrates that $\rho = (\frac{\partial \tilde{p}}{\partial p}) / (\frac{d\tilde{p}}{dp}) \ge 0.5$.
\end{proof}

\begin{corollary}[Majority Vote Approximation Bounds]
Majority Vote is Adversarially Symmetric. Under proportional decay,
\[
0
\le
\frac{d\tilde{p}_{\text{maj}}}{dp}
\le
\frac{\partial\tilde{p}_{\text{maj}}}{\partial p}.
\]
Consequently, the diagonal approximation may overestimate the exact gradient magnitude, with $\rho\ge1$ whenever $d\tilde{p}_{\text{maj}}/dp>0$.
\end{corollary}
\begin{proof}

To prove that Majority Vote satisfies Adversarial Symmetry, let index $0$ denote the ground truth $y^*$, and indices $j \in \{1, \dots, M\}$ denote rival answers. Let $W$ be the set of vote count configurations $\mathbf{n}$ that result in a plurality win for the ground truth. We first derive an exact identity for the partial derivatives of the success probability $\tilde{p}$.

\paragraph{Derivation of the Gradient Identity.}
The success probability $\tilde{p}(\boldsymbol{\pi})$ for an inference budget $N$ is a polynomial sum over all possible outcome count vectors $\mathbf{n}$ such that $\sum n_i = N$:
\begin{equation}
    \tilde{p}(\boldsymbol{\pi}) = \sum_{\substack{\mathbf{n} \ge 0 \\ |\mathbf{n}|=N}} \frac{N!}{\prod n_i!} \left( \prod_{i=0}^M \pi_i^{n_i} \right) \mathbb{I}(\mathbf{n} \in W)
\end{equation}
Differentiating with respect to a specific component $\pi_k$ yields:
\begin{equation}
    \frac{\partial \tilde{p}}{\partial \pi_k} = \sum_{\substack{\mathbf{n} \ge 0 \\ |\mathbf{n}|=N, n_k \ge 1}} \frac{N!}{\prod n_i!} \mathbb{I}(\mathbf{n} \in W) (n_k \pi_k^{n_k-1}) \prod_{i \neq k} \pi_i^{n_i}
\end{equation}
Using the factorial identity $n_k \frac{N!}{n_k!} = N \frac{(N-1)!}{(n_k-1)!}$, we can perform a change of variables by shifting the index to $\mathbf{m} = \mathbf{n} - e_k$, where $e_k$ is the one-hot vector for candidate $k$. This transforms the summation into an expectation over an $(N-1)$ sample process:
\begin{align}
    \frac{\partial \tilde{p}}{\partial \pi_k} &= N \sum_{\substack{\mathbf{m} \ge 0 \\ |\mathbf{m}|=N-1}} \underbrace{\frac{(N-1)!}{\prod m_i!} \prod \pi_i^{m_i}}_{\text{PMF of Multi}(N-1, \boldsymbol{\pi})} \mathbb{I}(\mathbf{m} + e_k \in W) \nonumber \\
    &= N \cdot \Pr(\mathbf{m} + e_k \in W)
\end{align}
This identity is remarkably intuitive: the sensitivity of the success probability to candidate $k$ is exactly proportional to the probability that the system wins \textit{given} that we force the $N$-th vote to be for candidate $k$. This identity holds for any strategy using aggregation by redefining $W$.

\paragraph{Establishing Symmetry.}
We now apply this identity to compare the sensitivity of the ground truth ($k=0$) against a rival ($k=j$).

Let $\mathcal{E}_{0} = \{ \mathbf{m} \mid \mathbf{m} + e_0 \in W \}$ be the event that the ground truth wins if it receives the forced $N$-th vote. Similarly, let $\mathcal{E}_{j} = \{ \mathbf{m} \mid \mathbf{m} + e_j \in W \}$ be the event that the ground truth wins \textit{even if} the rival $j$ receives the forced $N$-th vote.

Consider any vote configuration $\mathbf{m} \in \mathcal{E}_{j}$. For the ground truth to win despite the rival receiving an extra vote, the ground truth count $m_0$ must strictly exceed the rival's augmented count ($m_0 > m_j + 1$). This is a strictly stronger condition than what is required for $\mathcal{E}_{0}$, which only requires the ground truth's augmented count to exceed the rival's count ($m_0 + 1 > m_j$). 

Because satisfying the conditions for $\mathcal{E}_{j}$ inherently satisfies the conditions for $\mathcal{E}_{0}$, we have $\mathcal{E}_{j} \subseteq \mathcal{E}_{0}$. Since winning under a disadvantage is a subset of winning under an advantage, the probabilities satisfy $\Pr(\mathcal{E}_{j}) \le \Pr(\mathcal{E}_{0})$. 

Multiplying by $N$ recovers the gradient inequality, proving that Majority Vote satisfies Adversarial Symmetry:
\begin{equation}
    \left| \frac{\partial \tilde{p}}{\partial \pi_j} \right| \le \frac{\partial \tilde{p}}{\partial \pi_0}
\end{equation}

The gradient identity also gives $\partial\tilde{p}/\partial\pi_j\ge0$. Because the weights $\pi_j/(1-p)$ sum to one over the rivals,
\[
0\le\mathcal{C}_{\text{rival}}
\le\frac{\partial\tilde{p}}{\partial p}.
\]
The stated bounds now follow from
\Cref{eq:unified_gradient}.
\end{proof}

\subsection{Relaxing the Competitive Assumption: Guarantees under Adversarial Symmetry}
\label{app:non_competitive_guarantees}

In Theorem 3, we established that for Competitive strategies, the off-diagonal error vector strictly aligns with our diagonal approximation ($\langle g_{\text{diag}}, \epsilon_{\text{vec}} \rangle \ge 0$), guaranteeing a safe underestimate of the true gradient. However, certain complex test-time strategies may not be strictly Competitive. If increasing the probability of a rival answer $y'$ might increase the overall success probability ($\frac{\partial \tilde{p}}{\partial \pi_\theta(y' \mid x)} > 0$), the error vector may actively oppose the diagonal update. 

We must therefore ensure that this opposing force does not overpower the intended descent direction. We show that Adversarial Symmetry bounds the error by the diagonal gradient.

\begin{theorem}[Gradient Alignment under Adversarial Symmetry]
\label{thm:adversarial_symmetry_bound}
Assume the generation process follows Proportional Decay. For any test-time strategy satisfying Adversarial Symmetry, the magnitude of the off-diagonal error vector is bounded by the magnitude of the diagonal gradient approximation:
\begin{equation}
    \|\epsilon_{\text{vec}}\| \le \|g_{\text{diag}}\|
\end{equation}
Consequently, $\langle g_{\text{exact}},g_{\text{diag}}\rangle\ge0$, so the diagonal update does not increase the exact loss
to first order.
\end{theorem}

\begin{proof}
Recall the full expansion of the exact gradient under the Proportional Decay assumption:
\begin{equation}
    g_{\text{exact}} = g_{\text{diag}} + \epsilon_{\text{vec}}
\end{equation}
where the diagonal gradient is defined as $g_{\text{diag}} = -\frac{1}{\tilde{p}} \frac{\partial \tilde{p}}{\partial p} \nabla_\theta p$, and the off-diagonal error vector is:
\begin{equation}
    \epsilon_{\text{vec}} = \frac{1}{\tilde{p}} \left[ \frac{1}{1-p} \sum_{y' \in \mathcal{R}} \pi_\theta(y' \mid x) \frac{\partial \tilde{p}}{\partial \pi_\theta(y' \mid x)} \right] \nabla_\theta p
\end{equation}
To determine if the error vector can overpower our approximation, we analyze its maximum magnitude using the triangle inequality:
\begin{equation}
    \|\epsilon_{\text{vec}}\| \le \frac{1}{\tilde{p}} \frac{1}{1-p} \sum_{y' \in \mathcal{R}} \pi_\theta(y' \mid x) \left| \frac{\partial \tilde{p}}{\partial \pi_\theta(y' \mid x)} \right| \|\nabla_\theta p\|
\end{equation}
By adversarial symmetry, the absolute sensitivity of the success probability to any rival is strictly bounded by its sensitivity to the ground truth: $\left| \frac{\partial \tilde{p}}{\partial \pi_\theta(y' \mid x)} \right| \le \frac{\partial \tilde{p}}{\partial p}$. Substituting this upper bound into the summation yields:
\begin{equation}
    \|\epsilon_{\text{vec}}\| \le \frac{1}{\tilde{p}} \frac{1}{1-p} \left( \frac{\partial \tilde{p}}{\partial p} \right) \|\nabla_\theta p\| \sum_{y' \in \mathcal{R}} \pi_\theta(y' \mid x)
\end{equation}
The sum of the probabilities of all rival answers is exactly the remaining probability mass: $\sum_{y' \in \mathcal{R}} \pi_\theta(y' \mid x) = 1 - p$. Substituting this into the inequality perfectly cancels the denominator:
\begin{equation}
    \|\epsilon_{\text{vec}}\| \le \frac{1}{\tilde{p}} \frac{\partial \tilde{p}}{\partial p} \|\nabla_\theta p\| = \|g_{\text{diag}}\|
\end{equation}

Finally, by Cauchy--Schwarz,
\[ \langle g_{\text{exact}},g_{\text{diag}}\rangle \ge \|g_{\text{diag}}\|^2 -\|\epsilon_{\text{vec}}\|\,\|g_{\text{diag}}\| \ge0.
\]
This completes the proof.
\end{proof}

\paragraph{Implications for Optimization.}
Under proportional decay and Adversarial Symmetry, the off-diagonal error cannot reverse the diagonal direction. When $\|\epsilon_{\text{vec}}\|<\|g_{\text{diag}}\|$, Theorem 2 guarantees strict descent for a sufficiently small learning rate. At equality, the first-order loss change is nonpositive, but a finite-step decrease is not guaranteed.

\subsection{Empirical Validation of the Diagonal Approximation for Majority Vote}
\label{app: empirical validation 1}
We empirically measure the optimization efficiency ($\rho$) of the Majority Vote diagonal approximation under proportional decay. Corollary 4 gives $\rho\ge1$ whenever the total derivative is positive; the experiment measures the size of this overestimation.

\paragraph{Methodology.} We evaluate a Mistral-7B-Instruct model fine tuned on 2000 problems on a validation subset of the MATH benchmark. For each problem, we generate $M=32$ candidate solutions to establish an empirical distribution over the sequence-level probabilities of both the ground truth ($\pi_\theta(y^*|x)$) and the set of incorrect rival answers ($\mathcal{R}$). To compute the intractable off-diagonal sensitivities $\frac{\partial \tilde{p}}{\partial \pi_\theta(y'|x)}$, we utilize the combinatorial identity $\frac{\partial \tilde{p}}{\partial \pi_\theta(y'|x)} = N \cdot Pr(y^* \text{ wins} \mid y' \text{ receives 1 free vote})$ and estimate it via the Monte Carlo simulation for a test-time budget of $N=8$. The empirical optimization efficiency is then calculated as the ratio of the diagonal term to the full exact derivative:

\[
\rho =
\frac{\frac{\partial\tilde{p}}{\partial p}}
     {\frac{\partial\tilde{p}}{\partial p}
      - \mathcal{C}_{\text{rival}}}.
\]

\paragraph{Results \& Analysis.}
As shown in \Cref{fig:diagonal_histogram}, the estimated mean ratio is $1.2357$, with a median of $1.1255$ across the 81 retained problems. Under proportional decay, the diagonal term therefore overestimates the total derivative by $23.6\%$ on average. The rival correction reduces, rather than amplifies, the diagonal signal.

This high $\rho$ value close to $1$ is highly consequential for test-time alignment. It demonstrates that the computationally prohibitive task of evaluating the full parameter-space Jacobian across all rival sequences is unnecessary. By employing the diagonal approximation, we achieve a highly efficient surrogate objective for SFT and RL that faithfully represents the true test-time Majority Vote objective, simultaneously reducing variance and eliminating the need for exhaustive rival sampling during training.

\begin{figure}[htbp]
    \centering
    \includegraphics[width=0.8\textwidth]{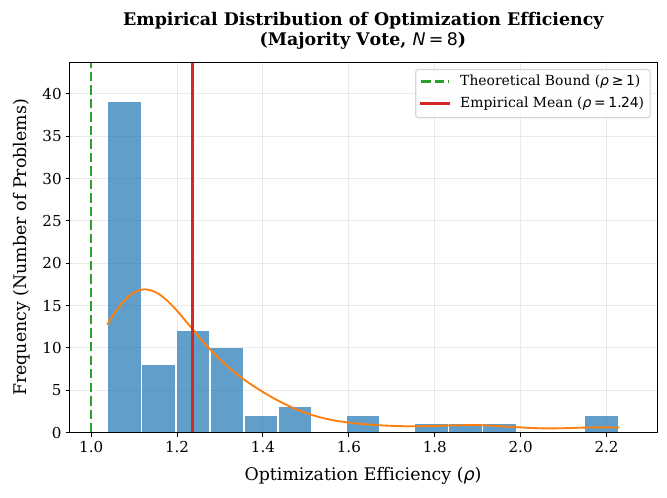}
\caption{Empirical distribution of optimization efficiency ($\rho$) for the Majority Vote diagonal approximation under proportional decay ($N=8$), across 81 retained MATH problems. The dashed line marks the lower bound $\rho=1$, and the solid line marks the empirical mean $\rho\approx1.236$. The diagonal term overestimates the total derivative by $23.6\%$ on average in this diagnostic.}
    \label{fig:diagonal_histogram}
\end{figure}

\section{Derivations and Explanations for RL}
\label{app:rl}

\subsection{RL Gradient Estimation} 
\label{app:rl_gradient}

For RL, we do not assume access to labeled data but instead learn from a general reward function $R(y|x)$. Our derivation is very similar to \cite{sutton1999policy}.

The classic RL REINFORCE objective \cite{williams1992simple} maximizes the expected reward of a single generation under the base policy:
\begin{equation}
 J(\theta) = \mathbb{E}_{y \sim \pi_{\theta}(\cdot \mid x)}[R(y|x)] = \sum_{y \in \mathcal{Y}} \pi_{\theta}(y \mid x) R(y|x)
 \label{eq:rl_objective}
\end{equation}

To align this objective with a Test-Time Strategy (TTS), we elevate the expectation to operate over the effective system-level policy. Let $\tilde{\pi}_\theta(y \mid x) = \mathcal{T}(\pi_\theta, \phi)(y \mid x)$ denote the probability that the overall system selects $y$ as its final output. 

The CAT objective for RL is the expected reward under the test-time strategy:
\begin{equation}
J_{\text{CAT}}(\theta) = \mathbb{E}_{y \sim \tilde{\pi}_{\theta}(\cdot \mid x)}[R(y|x)] = \sum_{y \in \mathcal{Y}} \tilde{\pi}_{\theta}(y \mid x) R(y|x)
\label{eq:rl_tts_objective}
\end{equation}

We compute the gradient with respect to $\theta$. 

Let $\boldsymbol{\pi}_\theta$ denote the vector of base probabilities over the output space $\mathcal{Y}$. The gradient of the objective is:
\begin{equation}
    \nabla_\theta J_{\text{CAT}} = \sum_{y \in \mathcal{Y}} R(y|x) \nabla_\theta \tilde{\pi}_{\theta}(y|x)
\end{equation}
Expanding $\nabla_\theta \tilde{\pi}_{\theta}(y|x)$ :
\begin{equation}
    \nabla_\theta \tilde{\pi}(y|x) = \sum_{y' \in \mathcal{Y}} \frac{\partial \tilde{\pi}_{\theta}(y|x)}{\partial \pi_\theta(y'|x)} \nabla_\theta \pi_\theta(y'|x)
\end{equation}

Substituting this back into the objective and swapping the order of summation:
\begin{align}
    \nabla_\theta J_{\text{CAT}} &= \sum_{y \in \mathcal{Y}} R(y|x) \left( \sum_{y' \in \mathcal{Y}} \frac{\partial \tilde{\pi}_{\theta}(y|x)}{\partial \pi_\theta(y'|x)} \nabla_\theta \pi_\theta(y'|x) \right) \nonumber \\
    &= \sum_{y' \in \mathcal{Y}} \left( \sum_{y \in \mathcal{Y}} R(y|x) \frac{\partial \tilde{\pi}_{\theta}(y|x)}{\partial \pi_\theta(y'|x)} \right) \nabla_\theta \pi_\theta(y'|x)
\end{align}
Using the log-derivative trick $\nabla_\theta \pi_\theta(y') = \pi_\theta(y') \nabla_\theta \log \pi_\theta(y')$, we can express this as an expectation over the base policy:
\begin{equation} \label{eq:general_gradient}
    \nabla_\theta J_{\text{CAT}} = \mathbb{E}_{y' \sim \pi_\theta} \left[ \underbrace{\left( \sum_{y \in \mathcal{Y}} R(y|x) \frac{\partial \tilde{\pi}_{\theta}(y|x)}{\partial \pi_\theta(y'|x)} \right)}_{\text{Marginal Test-Time Reward}} \nabla_\theta \log \pi_\theta(y'|x) \right]
\end{equation}
Equation \ref{eq:general_gradient} is the exact gradient. We will make the same assumption as the SFT case, that $\tilde{\pi}_{\theta}(y|x)$ is a function of $\pi_{\theta}(y|x)$:
\begin{equation}
    \frac{\partial \tilde{\pi}_{\theta}(y|x)}{\partial \pi_\theta(y'|x)} \approx 
    \begin{cases} 
        \frac{\partial \tilde{\pi}_{\theta}(y|x)}{\partial \pi_\theta(y|x)} & \text{if } y = y' \\
        0 & \text{if } y \neq y'
    \end{cases}
\end{equation}

Letting $p = \pi_\theta(y|x)$ and $\tilde{p} = \tilde{\pi}_{\theta}(y|x)$, the gradient simplifies to:
\begin{equation}
    \nabla_\theta J_{\text{CAT}} \approx \mathbb{E}_{y \sim \pi_\theta} \left[ \left( R(y|x) \frac{\partial \tilde{p}}{\partial p} \right) \nabla_\theta \log \pi_\theta(y|x) \right]
\end{equation}

The standard reward $R(y|x)$ is effectively multiplied by the marginal increase in test-time probability $\frac{\partial \tilde{p}}{\partial p}$. We define $\tilde{R}$ such that $R(y|x) \frac{\partial \tilde{p}}{\partial p} = \tilde{R}(y|x)$.

\subsection{RL Strategy Instantiations} \label{app:rl_strategies}

We derive the specific gradient estimators for the three primary test-time strategies with $M$ rollouts. The update takes the form of a weighted policy gradient:
\begin{equation}
    \nabla_\theta J \approx \frac{1}{M} \sum_{i=1}^M \tilde{R}(y_i|x) \cdot \nabla_\theta \log \pi_\theta(y_i \mid x)
\end{equation}

\subsubsection{Pass@$N$}

Recall that $\tilde{p}$ for Pass@$N$ is $1-(1-p)^{N}$ but this is only if the answer we chose is a correct one which was implicitly always true for SFT, otherwise the probability of the answer getting chosen is $0$. Luckily, if the answer is incorrect, $R(y_i|x)$ is $0$ anyways.

Thus the RL update weight for a sample $y_i$ is:
\begin{equation}
    \tilde{R}_{\text{pass}}(y_i|x) = 
        R(y_i|x) \cdot N(1 - p)^{N-1}
\end{equation}

\subsubsection{Majority Vote (Dynamic Threshold)}
Previously, we computed $\tilde{p} = \sum_{i=k}^{N} \binom{N}{i} p^i (1-p)^{N-i}$ and $\frac{\partial \tilde{p}}{{\partial p}} = N \binom{N-1}{k-1} p^{k-1} (1-p)^{N-k}$ with $k$ being the threshold required for the answer to be chosen. Previously, we had chosen $k$ as a hyperparameter. We could still do that. However, here, since we are sampling alternative answers, we can estimate $k$ using the batch.

Let $\hat{q_i}$ be the empirical probability of the most common alternate answer to $y_i$ in our $M$ rollouts. For the correct answer to win in an inference budget of $N$, it must appear at least $k$ times, where:
\[\hat{k}_i \approx \lfloor N \cdot \hat{q} \rfloor + 1\]
This establishes a dynamic threshold $k$ based on the "strength of the opposition."

\subsubsection{Best-of-$N$ (Empirical Max Reward)} \label{app:rl_bon}

In this setting, the strategy samples $N$ candidates $\{y_1, \dots, y_N\}$ and selects the candidate with the maximum reward $\tilde{y} = \argmax_{y_i} R(y_i|x)$.

To compute $\tilde{p}$ as a function of $p$. Let the rollout batch of $M$ candidates be sorted by reward. For a specific output $y$, let $P_{<y}$ denote the probability of generating an answer with a reward strictly lower than $R(y|x)$. The probability that a specific candidate $y$ is the selected maximum corresponds to the likelihood that all $N$ sampled candidates have rewards less than or equal to $R(y|x)$, excluding the scenario where all candidates have rewards strictly less than $R(y|x)$. The probability that $y$ wins  is given by:
\begin{equation}
    \tilde{p}(y) = (P_{<y} + p(y))^N - (P_{<y})^N
\end{equation}

To derive the gradient with respect to $p(y)$, we must enforce the constraint that probability distributions sum to 1. We assume that an increase in $p(y)$ is achieved by transferring probability mass from worse answers ($P_{<y}$), while the probability of strictly better answers ($P_{>y}$) remains constant.

Under this assumption, the cumulative probability of answers $\le y$ is constant relative to local changes in $p(y)$:
\[ P_{\le y} = P_{<y} + p(y) = \text{const} \]
Rearranging for $P_{<y}$, we have $P_{<y} = P_{\le y} - p(y)$. Substituting this into the objective function:
\begin{equation}
    \tilde{p}(y) = (P_{\le y})^N - (P_{\le y} - p(y))^N
\end{equation}
Differentiating with respect to $p(y)$:
\begin{align}
    \frac{\partial \tilde{p}(y)}{\partial p(y)} &= 0 - N(P_{\le y} - p(y))^{N-1} \cdot (-1) \nonumber \\
    &= N(P_{<y})^{N-1}
\end{align}
This result implies that the gradient signal is proportional to the probability that all competing samples are strictly worse than the current candidate $y$.

We estimate $P_{<y}$ empirically using the rollout batch. Let $c_{>y}$ be the count of samples with reward greater than $y$, and $c_{=y}$ be the count equal to $y$. Then $P_{<y} \approx 1 - \frac{c_{>y} + c_{=y}}{M}$.

The advantage weighting for the policy gradient update becomes:
\begin{equation}
    \tilde{R}_{\text{BoN}}(y_i|x) = R(y_i|x) \cdot N \left( 1 - \frac{c_{>y_i} + c_{=y_i}}{M} \right)^{N-1}
\end{equation}

\paragraph{Reduction to Pass@$N$:}
Consider the binary case where $R(y) \in \{0, 1\}$.
\begin{itemize}
    \item \textbf{Correct Answer ($R=1$):} Here, $c_{>y} = 0$ and $c_{=y} = c_{\text{correct}}$. The estimated lower-tail mass is $1 - \hat{p}$. The weight becomes:
    \[
        A = 1 \cdot N (1 - \hat{p})^{N-1}
    \]
    This matches Pass@$N$ but with $p$ computed using the empirical probabilities of the batch.
    \item \textbf{Incorrect Answer ($R=0$):} Here, all samples are either equal (0) or better (1), so $c_{>y} + c_{=y} = M$. The term becomes $(1 - 1)^{N-1} = 0$, correctly resulting in zero reward for incorrect answers.
\end{itemize}

\subsection{Best-of-N: Diagonal Approximation Guarantees}
\label{app: BoN Symmetry Guarantee}

Under the assumptions we have made, BoN demonstrates Adversarial Symmetry, allowing us to claim that the diagonal approximation is a safe underestimate. To see this, we must analyze the sensitivity of its test-time success probability to both the ground truth sequence $y$ and any rival sequence $y'$.

Recall the effective probability of selecting sequence $y$ under Best-of-N:
\begin{equation}
    \tilde{p}(y) = (P_{\le y})^N - (P_{< y})^N
\end{equation}

The direct sensitivity with respect to the ground truth probability $p(y)$ is:
\begin{equation}
    \frac{\partial\tilde{p}(y)}{\partial p(y)} = N(P_{\le y})^{N-1}
\end{equation}

Next, we evaluate the partial derivative with respect to a rival sequence $y'$. We consider two cases based on the reward $R(y')$:
\begin{itemize}
    \item \textbf{Stronger Rival ($R(y') > R(y)$):} The probability $p(y')$ does not contribute to $P_{\le y}$ or $P_{<y}$. Therefore, the derivative is $0$.
    \item \textbf{Weaker Rival ($R(y') < R(y)$):} The probability $p(y')$ is included in both $P_{\le y}$ and $P_{<y}$. Taking the derivative yields:
    \begin{equation}
        \frac{\partial\tilde{p}(y)}{\partial p(y')} = N(P_{\le y})^{N-1} - N(P_{< y})^{N-1}
    \end{equation}
\end{itemize}

A strategy exhibits \textit{Adversarial Symmetry} if the absolute sensitivity to any rival is strictly bounded by the sensitivity to the ground truth:
\begin{equation}
    \max_{y'\in\mathcal{R}} \left| \frac{\partial\tilde{p}}{\partial p(y')} \right| \le \frac{\partial\tilde{p}}{\partial p(y)}
\end{equation}

Comparing the maximum rival sensitivity to the direct sensitivity, we have:
\begin{equation}
    N(P_{\le y})^{N-1} - N(P_{< y})^{N-1} \le N(P_{\le y})^{N-1}
\end{equation}

Because probability distributions are strictly non-negative ($P_{< y} \ge 0$), this inequality strictly holds. 

Therefore, Best-of-N inherently satisfies Adversarial Symmetry. Consequently, the diagonal approximation for Best-of-N mathematically guarantees fail-safe descent ($||\epsilon_{vec}|| \le ||g_{diag}||$). This ensures that the scalar approximation acts in the same direction of the true gradient, preventing the optimizer from stepping in the wrong direction.

\paragraph{Uncompetitiveness.} We can make further guarantees about the approximation. Because probability distributions are strictly non-negative, $P_{\le y} > P_{<y}$. Consequently, the sensitivity to any weaker rival is strictly positive ($\frac{\partial\tilde{p}}{\partial p(y')} > 0$). Intuitively, increasing the likelihood of generating a weaker answer effectively "pads" the batch with easy opponents, increasing the ground truth's chance of surviving the tournament. Because rival sensitivities are strictly nonnegative, Best-of-$N$ is formally classified as an \textbf{Uncompetitive Strategy}. 

Because the strategy is uncompetitive, the rivalry correction term $\mathcal{C}_{rival}$ from the Unified Gradient Equation (\Cref{eq:unified_gradient}) evaluates to a strictly positive value. Since the exact total derivative is $\frac{d\tilde{p}}{dp} = \frac{\partial \tilde{p}}{\partial p} - \mathcal{C}_{rival}$, the exact gradient magnitude is smaller than our diagonal term. Thus, the diagonal approximation strictly overestimates the true gradient ($\rho \ge 1.0$). This is a worse situation than a conservative underestimate. Thus, it is advised to decrease learning rate with BoN CAT training.

\subsection{Empirical Accuracy of BoN Diagonal Approximation}
\label{sec:bon_diagonal}

We may have a theoretical guarantee that the diagonal approximation points in the right direction; however, we have no such guarantee of efficiency. We empirically demonstrate the efficacy of this approximation. To do this experiment, we must quantify the magnitude of the discarded rivalry correction $\mathcal{C}_{rival}$.

\subsection{Derivation of Rival Sensitivity}
To circumvent the intractable Jacobian evaluation, we exploit the mathematical structure of independent test-time sampling. Because BoN draws $N$ independent samples from the categorical distribution defined by $\pi_\theta(\cdot|x)$, the resulting counts of each generated sequence follow a multinomial distribution. This derivation is virtually identitcal to the one done for \Cref{app:approximation_theorems}.

Let $C = (c_1, c_2, \dots, c_K)$ be the counts of all possible sequences, such that $\sum c_i = N$. The probability of any specific batch configuration is:
\begin{equation}
    \mathbb{P}(C) = \frac{N!}{\prod_{i} c_i!} \prod_{i} \pi_\theta(y_i|x)^{c_i}
\end{equation}

Let $W$ be the set of all batch configurations where $y^*$ wins the BoN tournament (i.e., $c_{y^*} \ge 1$ and $R(y^*|x) \ge R(y_i|x)$ for all drawn $y_i$). The total success probability is $\tilde{p} = \sum_{C \in W} \mathbb{P}(C)$. 

When we take the partial derivative of $\tilde{p}$ with respect to a specific rival sequence's probability $\pi_\theta(y'|x)$, the exponent $c_{y'}$ is brought down as a multiplier. Through factorial cancellation, the $N$ terms rearrange perfectly to describe an $(N-1)$ multinomial distribution, mapping the derivative exactly to a conditional expectation:

\begin{equation}
\label{eq:combinatorial_identity}
\begin{aligned}
    \frac{\partial \tilde{p}}{\partial \pi_\theta(y'|x)} &= \sum_{C \in W} c_{y'} \left( \frac{N!}{c_1! \dots c_{y'}! \dots c_K!} \right) \pi_\theta(y_1|x)^{c_1} \dots \pi_\theta(y'|x)^{c_{y'}-1} \dots \\
    &= N \sum_{C \in W} \left( \frac{(N-1)!}{c_1! \dots (c_{y'}-1)! \dots c_K!} \right) \prod_i \pi_\theta(y_i|x)^{\tilde{c}_i} \\
    &= N \cdot \mathbb{P}\big(y^* \text{ wins BoN} \mid y' \text{ receives 1 guaranteed draw}\big)
\end{aligned}
\end{equation}

Note again the similarity to the derivation in \Cref{app:approximation_theorems}. 

\subsubsection{Empirical Approximation of Optimization Efficiency}
Using Equation \ref{eq:combinatorial_identity}, we can precisely measure the magnitude of the rivalry correction without computing the true Jacobian. By normalizing the probabilities across the empirical candidate space, the rivalry correction becomes the expected conditional sensitivity:
\begin{equation}
    \mathcal{C}_{rival} = \frac{1}{1-\pi_\theta(y^*|x)} \sum_{y' \neq y^*} \pi_\theta(y'|x) \left[ N \cdot \mathbb{P}\big(y^* \text{ wins} \mid y' \text{ guaranteed}\big) \right]
\end{equation}

\paragraph{Results}
To empirically validate that our diagonal approximation is not too far off, we approximated $\rho$ across empirical protein sequences using $N=8$ for our $N$ in BoN and batch size $M=64$ using 400 random inputs with the conditional reward experimental setup from \Cref{app:bon_conditional_implementation} on the warmed-up SFT model. As shown in \Cref{fig:bon_diagonal_histogram}, our empirical measurements strictly do not stray too far from $1$. The empirical average optimization efficiency is $\rho \approx 1.26$. This means that on average our gradient overestimates the true gradient by $1.26$.

\begin{figure}[!h]
    \centering
    \includegraphics[width=0.8\textwidth]{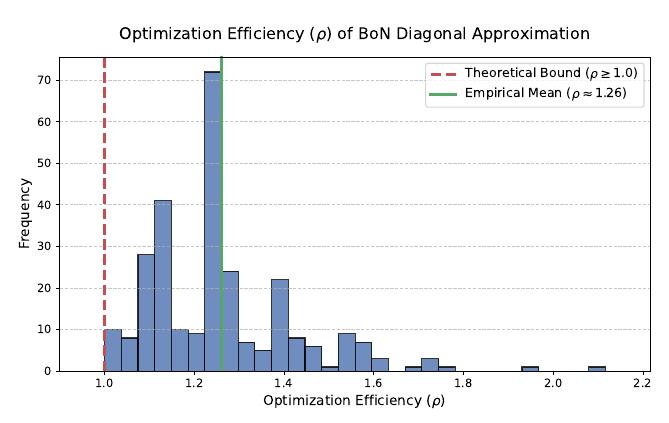}
    \caption{Empirical distribution of optimization efficiency ($\rho$) for the Best-of-$N$ diagonal approximation ($N=8, M=64$) across generated protein sequences. The empirical mean of $\rho \approx 1.26$ (green solid line) demonstrates that the computationally efficient diagonal approximation doesn't wildly overestimate the true signal.}
    \label{fig:bon_diagonal_histogram}
\end{figure}

\subsection{Variance of the Gradient Estimator}
\label{app:rl_variance}

In the Reinforcement Learning setting, the strategy-aware gradient can be viewed as an importance-weighted estimator. Let the standard RL gradient for a single sample be denoted as $\nabla_\theta J_{\text{base}} \approx R(y|x) \nabla_\theta \log \pi_\theta(y|x)$. The CAT gradient modifies this with a strategy-specific weight $w$:
\begin{equation}
    \nabla_\theta J_{\text{CAT}} \approx w \cdot \nabla_\theta J_{\text{base}}
\end{equation}
where $w = \frac{\partial \tilde{p}}{\partial p}$ is the scalar scaling factor derived in previous sections.

To quantify the stability of this estimator, we use the \textit{Variance Inflation} (VI) factor. Assuming the variance of the gradient is dominated by the magnitude of the scaling coefficients, the inflation relative to standard RL is proportional to the second moment of the weights:
\begin{equation}
    \text{VI} \approx \mathbb{E}[w^2] = \mathbb{E}\left[ \left| \frac{\partial \tilde{p}}{\partial p} \right|^2 \right]
\end{equation}
We analyze how this variance scales with the inference budget $N$ across the three primary strategies.

\subsubsection{Best-of-$N$: Linear Scaling ($O(N)$)}
For Best-of-$N$, the weight $w$ depends on the relative rank (quantile) $u \in [0, 1]$ of the generated sample's reward. The weight function is $w(u) = N u^{N-1}$.
Assuming a uniform distribution of quantiles in the rollout batch, the expected squared weight is:
\begin{align}
    \mathbb{E}[w^2] &= \int_0^1 \left( N u^{N-1} \right)^2 du = N^2 \int_0^1 u^{2N - 2} du \nonumber \\
    &= N^2 \left[ \frac{u^{2N - 1}}{2N - 1} \right]_0^1 = \frac{N^2}{2N - 1} \approx \frac{N}{2}
\end{align}
\textbf{Result:} The variance for Best-of-$N$ scales \textbf{linearly} with $N$. This confirms the "winner-take-all" instability: as $N$ increases, the effective sample size decreases, concentrating the signal on the top $1/N$ fraction of samples.

\subsubsection{Pass@$N$: Quadratic Scaling ($O(N^2)$)}
For Pass@$N$, the weight is determined by the prompt difficulty $p$. The weight is given by $w = N(1-p)^{N-1}$.
\begin{itemize}
    \item \textbf{Hard Regime ($p \ll 1/N$):} In this regime, $(1-p)^{N-1} \approx 1$. Consequently, the weight approaches $w \approx N$.
    \item \textbf{Variance:} The second moment scales as $\mathbb{E}[w^2] \approx N^2$.
\end{itemize}
\textbf{Result:} Pass@$N$ exhibits \textbf{quadratic} variance scaling for hard problems. This makes it the most unstable objective for RL. If the model encounters a batch of hard problems, the gradients scale by $N^2$, likely necessitating aggressive gradient clipping to prevent policy collapse.

\subsubsection{Majority Vote: Linear Scaling ($O(N)$)}
For Majority Vote, we recall that the derivative of the effective probability corresponds to the probability mass of the Binomial distribution at the decision threshold. Specifically, for a consensus threshold $k$, the weight is:
\begin{equation}
    w = \frac{\partial \tilde{p}}{\partial p} = N \binom{N-1}{k-1} p^{k-1} (1-p)^{N-k}
\end{equation}
The gradient is most active at the decision boundary where the model's base probability aligns with the requirement, i.e., $p \approx k/N$.

Let us assume the threshold scales linearly with the budget, such that $k \approx N/c$ for some constant $c > 1$ (e.g., $c=2$ for simple majority). At the decision boundary, the base probability is $p \approx 1/c$. We approximate the Binomial term using a Gaussian distribution $\mathcal{N}(\mu, \sigma^2)$ where the variance is:
\begin{equation}
    \sigma^2 = N p (1-p) \approx N \left(\frac{1}{c}\right) \left(1 - \frac{1}{c}\right)
\end{equation}
Crucially, despite the constant $c$, the variance scales linearly with the budget: $\sigma^2 \propto N$.

At the peak of this distribution (the decision boundary), the probability mass scales inversely with the standard deviation:
\begin{equation}
    \binom{N-1}{k-1} p^{k-1} (1-p)^{N-k} \approx \frac{1}{\sqrt{2\pi \sigma^2}} \propto \frac{1}{\sqrt{N}}
\end{equation}
Substituting this back into the expression for the weight $w$:
\begin{equation}
    w_{\text{peak}} \approx N \cdot \frac{C}{\sqrt{N}} \propto \sqrt{N}
\end{equation}
The Variance Inflation is determined by the square of this weight:
\begin{equation}
    \text{VI} \approx \mathbb{E}[w^2] \propto (\sqrt{N})^2 = N
\end{equation}

\subsection{Variance Reduction in Gradient Estimation}
\label{app:variance}

As derived in the previous subsection, the variance of standard estimators for strategies like Pass@$N$ can scale quadratically with the inference budget ($O(N^2)$) in low-probability regimes. To make training stable and efficient, we will describe some tricks we used to reduce variance.

\subsubsection{Pass@$N$ Normalized Weight}
\label{app:normalized_estimator_theory}

The raw RL gradient weight for Pass@$N$ is given by $w_{\text{RL}} = N(1-p)^{N-1}$. For "hard" samples where $p \to 0$, this weight causes the variance of the gradient to explode as $N$ increases.

To remedy this, we adopt the normalized scaling factor derived from our SFT formulation (Section \ref{app:sft_passn}). We replace the raw RL weight with the normalized weight $w_{\text{Norm}}$:
\begin{equation}
    w_{\text{Norm}} = \frac{p}{\tilde{p}} \cdot w_{\text{RL}} = \frac{p}{1 - (1-p)^N} \cdot N(1-p)^{N-1}
\end{equation}

From our previous analyses in \Cref{app:sft_variance}, we know that this weight doesn't experience a variance explosion.

But, the question remains, how much bias does this weight introduce.
To quantify the bias introduced by substituting the normalized weight $w_{\text{Norm}}$ for the true reinforcement learning weight $w_{\text{RL}}$, we can examine the ratio between the two estimators. This ratio is exactly the scaling factor from our SFT derivation:
\begin{equation}
    \frac{w_{\text{Norm}}}{w_{\text{RL}}} = \frac{p}{1 - (1-p)^N}
\end{equation}
Because this is a multiplicative factor, the normalized estimator introduces an attenuation bias relative to the exact gradient. The severity of this bias depends directly on the difficulty of the sample:

\paragraph{The Hard Regime ($p \to 0$):}
For highly uncertain generations, we can analyze the asymptotic bias using a first-order Taylor expansion where $(1-p)^N \approx 1 - Np$:
\begin{equation}
    \lim_{p \to 0} \frac{w_{\text{Norm}}}{w_{\text{RL}}} \approx \frac{p}{1 - (1 - Np)} = \frac{1}{N}
\end{equation}
In the low-probability regime, the normalized estimator systematically underestimates the true gradient magnitude by a factor of $N$.

\paragraph{The Easy Regime ($p \to 1$):}
As the model masters a prompt and the base probability approaches certainty, $(1-p)^N \to 0$, and the scaling ratio approaches $1$:
\begin{equation}
    \lim_{p \to 1} \frac{w_{\text{Norm}}}{w_{\text{RL}}} = 1
\end{equation}
Here, the estimator is asymptotically unbiased. 

\textbf{Conclusion:} The normalized weight removes the variance explosion but adds bias as $p \to 0$. However, this bias decreases the magnitude of the gradient, which is preferable to overestimating it to avoid overtraining. We show this emprically in \Cref{app:ablations}.

\subsubsection{Parameter Estimation with Small Batches ($M \ll N$)}

A significant practical challenge in CAT RL is the discrepancy between the training rollout size ($M$) and the intended test-time budget ($N$). Due to memory constraints, we often set $M \ll N$ (e.g., $M=4$ rollouts while optimizing for Pass@$64$).

This leads to high variance in the estimation of the strategy parameters, such as the thresholds (for Best-of-$N$ or Maj Vote). To mitigate this, we sometimes used the following techniques.

\paragraph{Smoothing Empirical Counts.}
 To prevent the estimates of thresholds from swinging wildly (e.g., between 0 and 1) due to small batch sizes $M$, we apply Laplace smoothing:
\begin{equation}
    \hat{k}_{\text{smooth}} = \frac{n + 1}{M + 2}
\end{equation}
where $n$ is the number of successful outcomes in the batch. This simple adjustment prevents numerical instability and dampens extreme gradient spikes derived from statistically anomalous batches.

\paragraph{Distribution Tracking.}
While smoothing addresses numerical stability, it does not solve the issue of coarse granularity. For strategies like Best-of-$N$, relying on a small batch size ($M$) forces the estimated quantiles $P_{<y}$ to take on discrete, coarse values (e.g., $0, 0.25, 0.5, \dots$ for $M=4$). This quantization introduces significant noise into the gradient.

To remedy this, we decouple the estimation of distributional statistics from the immediate mini-batch by maintaining a history buffer or moving average of recent trajectories.
\begin{itemize}

\item \textbf{For Best-of-$N$ (Time-Weighted Estimate):} 
To account for non-stationarity (the fact that the model's reward distribution shifts as it learns), we cannot treat all samples in the history buffer equally. Instead of a uniform average, we employ a \textit{time-weighted} estimation strategy.

We maintain a FIFO buffer $\mathcal{B}$ of pairs $(r_j, t_j)$, representing the reward and the training step of the $j$-th sample. When estimating the quantile $P_{<y_i}$ for a current sample with reward $R(y_i)$, we assign a decay weight $w_j = \lambda^{(T - t_j)}$ to each historical sample, where $T$ is the current step and $\lambda \in (0, 1)$ is a decay factor. The quantile is estimated via the weighted empirical CDF:
\begin{equation}
    \hat{P}_{<y_i} = \frac{\sum_{(r_j, t_j) \in \mathcal{B}} w_j \cdot \mathbb{I}(r_j < R(y_i))}{\sum_{(r_j, t_j) \in \mathcal{B}} w_j}
\end{equation}
This effectively creates an exponential moving average of the distribution, ensuring the quantile reference remains synchronized with the current policy $\pi_\theta$.
    \item \textbf{For Majority Vote:} We maintain an exponential moving average (EMA) of the consensus difficulty. If $\hat{k}_t$ is the ideal threshold estimated for batch $t$, we update a global tracking variable $\bar{k}$ via $\bar{k}_{t+1} = \gamma \bar{k}_t + (1-\gamma)\hat{k}_t$.
\end{itemize}
This technique allows us to estimate fine-grained global statistics required for the strategy-aware weights while keeping the active computational batch size $M$ small for efficiency.

\subsection{Empirical Variance Reduction: Decoupling Utility from Step Size}
\label{app:weight_normalization}

A critical challenge in applying the strategy-aware gradient $\nabla_\theta J_{\text{CAT}} \approx w_i \cdot \nabla_\theta J_{\text{Base}}$ is the high variance of the weighting factor $w_i$. For instance, in Pass@$N$, the raw marginal utility on "hard" prompts scales linearly with the budget ($w_i \approx N$). Directly applying these raw weights can result in extreme gradient magnitudes.

While it is theoretically correct that different prompts should command different gradient magnitudes, we must empirically bound the variance of the update step. 

A naive approach to stabilization is \textit{prompt-level normalization}, where weights are averaged within the $M$ rollouts of a single prompt $x$: $\hat{w}_i = w_i / \mathbb{E}[w|x]$. However, this is counterproductive. By forcing every prompt to exert an average gradient weight of 1, we inadvertently destroy the global capacity allocation mechanism, forcing the model to take full-sized gradient steps even on "solved" or "hopeless" prompts.

Instead, we employ \textit{Batch-Level Weight Normalization}. For a global mini-batch $B$ containing multiple prompts and their respective rollouts, we normalize the strategy weights against the batch mean:
\begin{equation}
    \tilde{w}_i = \frac{w_i}{\frac{1}{|B|} \sum_{j \in B} w_j}
\end{equation}

This rigorous formulation achieves two critical optimization goals simultaneously:
\begin{enumerate}
    \item \textbf{Step Size Stability:} The expected magnitude of the CAT multiplier across the batch is exactly 1 ($\mathbb{E}[\tilde{w}] = 1$). This completely decouples the scale of the learning rate from the test-time budget $N$, ensuring optimization stability regardless of the strategy used.
    \item \textbf{Preservation of Relative Capacity:} Because the normalization factor is shared across all prompts in the batch, the \textit{relative} weighting between different questions is perfectly preserved. A prompt at the Majority Vote "tipping point" will still exert proportionally more gradient pressure than a prompt that is safely solved, maintaining the core benefit of the strategy-aware objective.
\end{enumerate}

\subsection{On the Utilization of RL Rollouts for Exact Gradients} 

A natural extension to the diagonal approximation in the RL setting is to utilize the generated rollout batch to estimate the off-diagonal sensitivities $\frac{\partial \tilde{p}(y^*|x)}{\partial \pi_\theta(y'|x)}$ for rival sequences. However, calculating this is highly non-trivial.

For aggregation strategies like Majority Vote, the sensitivity to a specific rival $y'$ depends on the probability that $y'$ dictates the consensus threshold. Evaluating this requires marginalizing over the joint distribution of all other possible outputs. Estimating these complex, multi-way dependencies from a sparse RL mini-batch ($M \ll N$) introduces extreme variance that destabilizes updates.

Rigorously handling these combinatorial credit-assignment dependencies without variance explosion is relegated to future work. In this work, we stick to the more robust, computationally efficient diagonal approximation.

\section{Extending to GRPO and PPO}
\label{sec:grpo_ppo}

While \Cref{app:rl} derived the CAT gradients in the context of vanilla Policy Gradients (REINFORCE), modern LLM post-training relies on variance-reduced, trust-region methods such as Proximal Policy Optimization (PPO) \cite{schulman2017proximal} and Group Relative Policy Optimization (GRPO) \cite{shao2024deepseekmath}.

Integrating the CAT into these algorithms requires careful handling of the advantage normalization step. We posit that the strategy-aware scaling factor must be treated as an external importance weight applied \textit{after} the standard advantage estimation and normalization.

\subsection{Strategy-Aware GRPO}
GRPO eliminates the value function critic by using the group average reward as a baseline. For a prompt $x$, GRPO samples a group of $G$ outputs $\{y_1, \dots, y_G\}$ from the old policy $\pi_{\theta_{old}}$.

Standard GRPO calculates the advantage $A_i$ for the $i$-th sample by normalizing the rewards within the group:
\begin{equation}
    A_i = \frac{R(y_i) - \text{mean}(\{R(y_1), \dots, R(y_G)\})}{\text{std}(\{R(y_1), \dots, R(y_G)\}) + \epsilon}
\end{equation}

To align this with a test-time strategy $\mathcal{T}$, we recall our core derivation: $\nabla J_{\text{TTS}} \approx w_i \cdot \nabla J_{\text{Base}}$, where $w_i = \frac{\partial \tilde{p}}{\partial p_i}$ is the marginal utility of the sample under the strategy.

\paragraph{The Normalization Pitfall.} A naive approach might be to incorporate $w_i$ directly into the reward, defining $\tilde{R}_i = w_i \cdot R(y_i)$. However, applying GRPO normalization to $\tilde{R}_i$ would be catastrophic. The normalization step ($R - \mu) / \sigma$ is scale-invariant; it would cancel out the magnitude of $w_i$. For example, in Pass@$N$, as the model approaches success, $w_i \to 0$. Normalization would rescale these vanishing rewards back to unit variance, effectively negating the efficiency regularization intended by the framework.

\paragraph{Correct Formulation.}
We effectively treat the Test-Time strategy weight $w_i$ as a sample-specific learning rate scaler. We compute the standard GRPO advantage $A_i$ using the raw ground-truth rewards (or reward model scores), and apply $w_i$ multiplicatively to the surrogate objective:

\begin{equation}
    \mathcal{L}_{\text{GRPO-TTS}}(\theta) = -\frac{1}{G} \sum_{i=1}^G \underbrace{w_i(p_i, \phi)}_{\text{CAT Weight}} \cdot \min \left( \frac{\pi_\theta(y_i|x)}{\pi_{\theta_{old}}(y_i|x)} A_i, \text{clip}(\dots) A_i \right)
\end{equation}

\subsection{Implementation with PPO}
The same logic applies to PPO \cite{schulman2017proximal} with a learned Value function $V_\psi(x)$. We compute the Generalized Advantage Estimate (GAE) \cite{schulman2015high} $\hat{A}_t$ using the standard reward signal $R_t$. We then normalize the advantages across the batch: $\hat{A}_{norm} = (\hat{A} - \mu) / \sigma$.

The CAT weight is applied solely to the policy loss term:
\begin{equation}
    \mathcal{L}_{\text{PPO-TTS}} = \mathbb{E}_t \left[ w_t \cdot \mathcal{L}^{\text{CLIP}}_t(\theta, \hat{A}_{norm}) - c_1 \mathcal{L}^{\text{VF}}_t - c_2 S[\pi](y_t) \right]
\end{equation}
By decoupling the strategy weight $w_t$ from the advantage calculation, we ensure that the optimization landscape remains well-conditioned via normalization, while the gradient \textit{magnitude} correctly reflects the sample's contribution to the test-time metric.

\section{Interaction between CAT and RL Rollouts}
\label{app: improve_pass@1}

A foundational premise of this work is that search and learning operate in a continuous, mutually beneficial cycle. A piece of this cycle that we have yet to touch on is, in the RL paradigm, how the altered CAT gradient affects how the model receives signals in future rollouts. We do not give a complete picture of that dynamic in this work but will use this section to discuss key observations.

A keen observer may have noticed a compelling anomaly in \Cref{sec:beyond_sft}: the Pass@$4$ aligned model actually outperformed the baseline on single-shot accuracy (Pass@$1$). This is particularly noteworthy given that our training batch size was also $4$. We hypothesize that by aligning the training objective with the generation budget, the model discovered a higher volume of correct reasoning paths during its rollouts, thereby enriching its own learning signal.

To test our hypothesis in a more controlled setting, we isolate the effect of strategy-aware SFT as a warmup phase for standard RL. In deep reinforcement learning for reasoning tasks, a primary bottleneck is sparse rewards: if the policy does not naturally generate the correct answer during its exploratory rollouts, it receives no positive reinforcement, leading to stalled learning. Standard SFT, which optimizes for single-sample precision, often exacerbates this by collapsing the generation variance and discouraging exploration, making it harder for the model to stumble upon the right reasoning path for difficult problems during the RL phase. Our goal is to use CAT to correct for this in the warmup.

\paragraph{Experimental Setup} 
To validate whether CAT can overcome this bottleneck, we trained one model using normal SFT and one model with Pass@16 CAT SFT for 1 epoch for 2000 samples following the same hyperparameters as our CAT SFT experiment in \Cref{sec:beyond_pass}, followed by an epoch of GRPO on each model with a batch size of $M=16$ and the same hyperparameters as in \Cref{app:rl_pass_implementation}. We then evaluated both models using the Pass@$k$ calculations from \Cref{app:implementation} across varying test-time budgets over 100 questions.

\paragraph{Results} As we can see in \Cref{tab:sft_rl}, the CAT model beat the normal SFT model at every single Test-Time metric, including Pass@1 after RL training.

\begin{table}[h]
\centering
\caption{Pass@k Performance Comparison}
\label{tab:pass_at_k_results}
\begin{tabular}{l c c c c c}
\toprule
\textbf{Model} & \textbf{Pass@1} & \textbf{Pass@2} & \textbf{Pass@4} & \textbf{Pass@8} & \textbf{Pass@16} \\
\midrule
SFT & 4.06\% & 6.25\% & 9.19\% & 12.80\% & 17.00\% \\
\textbf{CAT} & \textbf{5.56\%} & \textbf{8.82\%} & \textbf{12.95\%} & \textbf{17.63\%} & \textbf{22.00\%} \\
\bottomrule
\end{tabular}
\caption{The model warmed up with CAT uniformly outperformed the standard model.}
\label{tab:sft_rl}
\end{table}

\paragraph{Analysis} By applying CAT (Pass@$16$) during the SFT warmup, we optimized the model for getting at least 1 answer right out of 16. When the CAT model transitions to the RL phase, it is more likely to get at least 1 of its 16 attempts per question right and receive signal. In this way, aligning the SFT warmup phase with a search strategy improves the sample efficiency and success rate of standard RL. We demonstrate that strategy-aware learning can actively improve the standard learning process itself. Here, CAT yields a more capable single-shot model (Pass@$1$).

However, to more broadly understand how CAT reshapes RL dynamics between successive rollouts, we must examine its effect on policy entropy. It is a well-established phenomenon that maintaining sufficient entropy is highly beneficial to robust RL training \cite{pmlr-v97-ahmed19a, eysenbach2022maximum}, as it prevents premature mode collapse. While the precise evolutionary dynamics of the policy depend heavily on the specific TTS to which the RL is aligned, these modified objectives generally act to preserve policy entropy. By design, they penalize redundant overconfidence and instead maintain the distributional variance necessary to seed the search process effectively. This helps the model retain enough exploratory capacity during training to consistently allow the search strategy to discover and correct for the optimal answer. A more in-depth, specific analysis of these dynamics is saved for future work.

\section{Interpreting Scaling Factors}
\label{app: interpereting scaling factors}
By visualizing the behavior of $w(p, \phi)$, we can understand the mechanisms via which CAT works for each strategy.

\subsection{Pass@$N$: The Efficiency Regularizer}
Pass@$N$ considers a prompt "solved" if \textit{any} of the $N$ samples are correct. 

\begin{figure}[h]
    \centering
    \begin{minipage}{0.35\textwidth} 
        \centering
        \includegraphics[width=\linewidth]{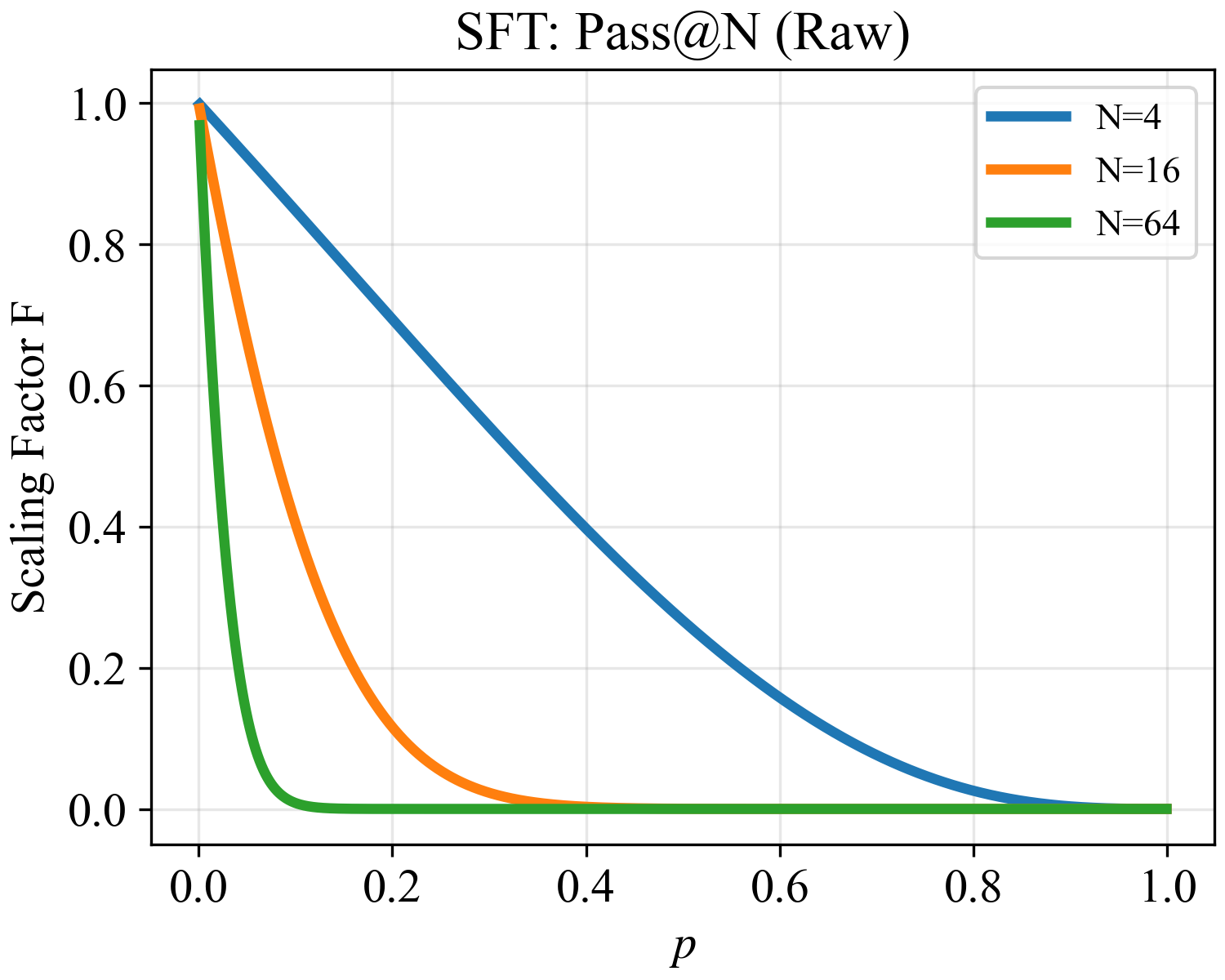}
        \caption{SFT Scaling Factor for Pass@$N$.}
        \label{fig:sft_pass}
    \end{minipage}
    \hspace{1cm} 
    \begin{minipage}{0.35\textwidth}
        \centering
        \includegraphics[width=\linewidth]{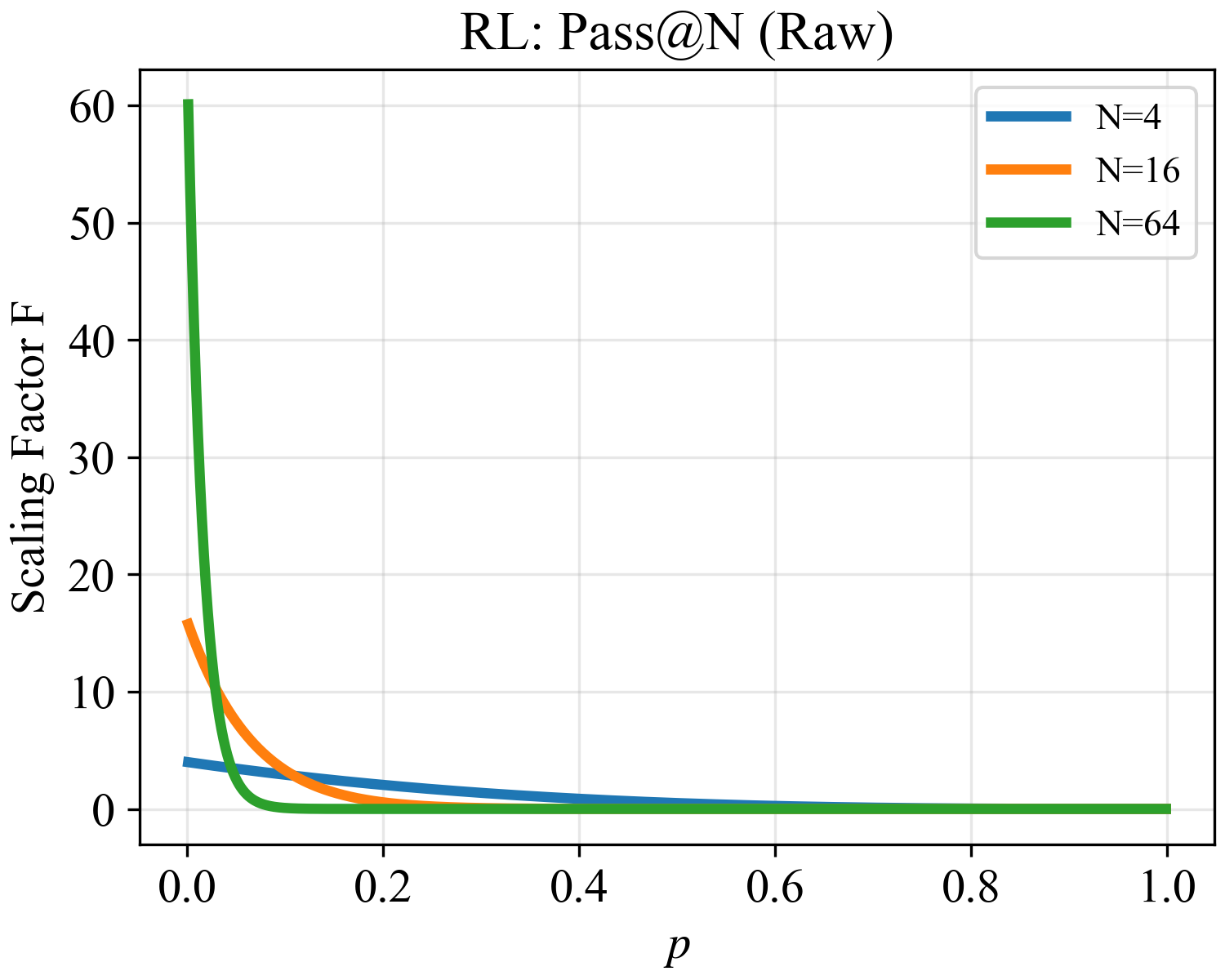}
        \caption{RL Scaling Factor for Pass@$N$.}
        \label{fig:rl_pass}
    \end{minipage}
\end{figure}

As shown in Figures \ref{fig:sft_pass} and \ref{fig:rl_pass}, the scaling factor for Pass@$N$ is monotonically decreasing with respect to the base model confidence $p$. 
\begin{itemize}
    \item \textbf{Diminishing Returns:} For a standard model ($N=1$), the gradient remains active even as $p \to 1$, encouraging the model to memorize the sample. However, with $N=64$ (Green line), the scaling factor collapses to near-zero once $p \approx 0.1$.
    \item \textbf{Training Dynamics:} This acts as an efficiency regularizer. The objective signals: \textit{"We are virtually guaranteed to find the answer within 64 tries; stop wasting capacity on this easy sample."} This preserves model capacity for "hard" samples.
    \item \textbf{Hard Example Boosting} The RL weighing factor has one key property that the SFT factor does not which is that for $p \to 0$, the RL factor boosts the gradient, optimizing for the tail of the distribution.
\end{itemize}

A deeper analysis of the SFT scaling factor (which is very similar to the RL factor) can be found in \citet{chen2025rethinking}.

\subsection{Majority Vote: The Tipping Point}
Majority Vote requires the correct answer to achieve a plurality of the $N$ outputs. We reframe this as the requirement that the correct answer $y^*$ be selected at least $k$ times out of $N$ samples where $k-1$ represents the number of times the most common incorrect answer gets chosen in the batch of $N$.

\begin{figure}[h]
    \centering
    \begin{minipage}{0.35\textwidth}
        \centering
        \includegraphics[width=\linewidth]{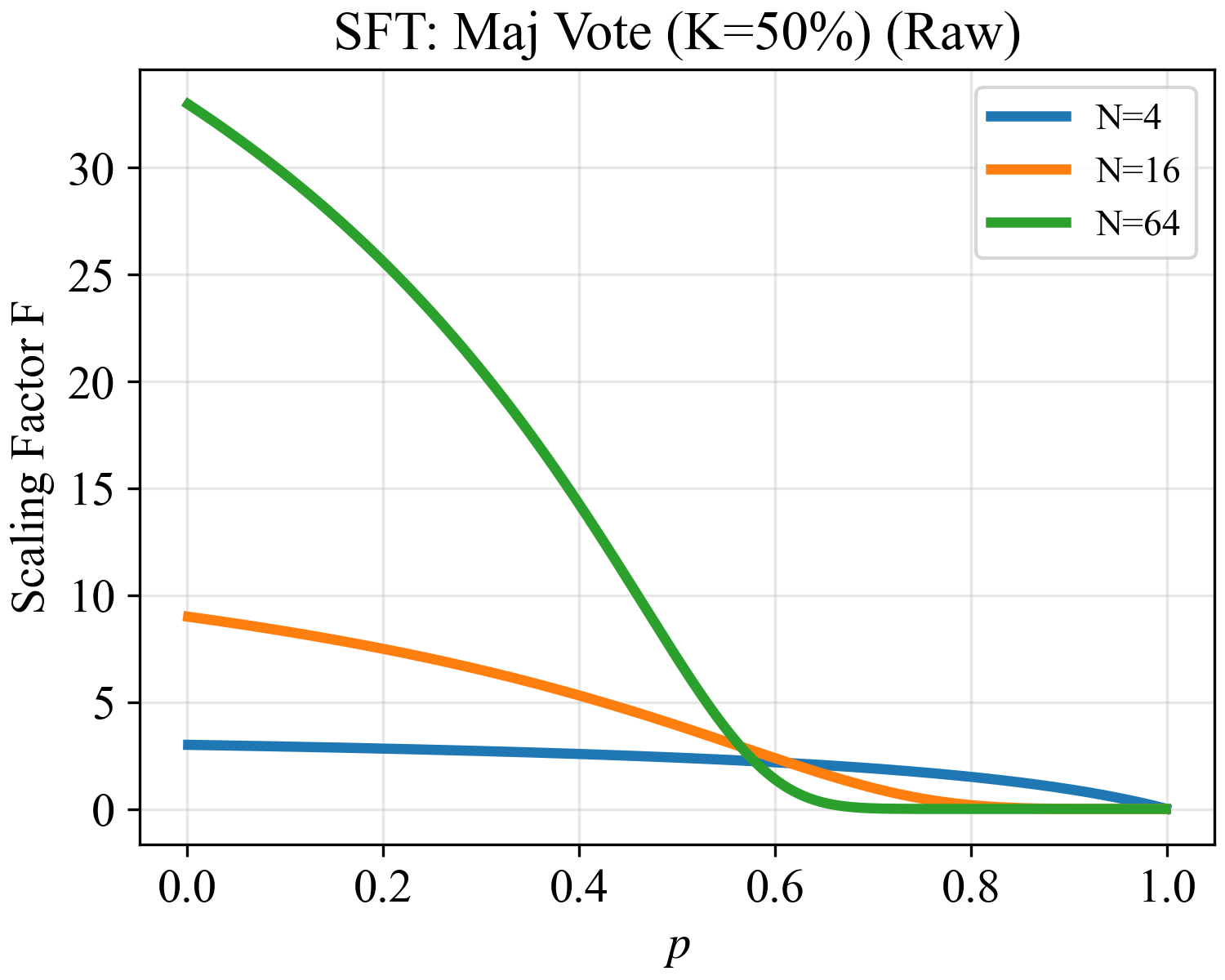}
        \caption{SFT Scaling Factor for Majority Vote. The gradient remains active ($F \approx k$) for low $p$, encouraging the model to build consensus from scratch.}
        \label{fig:sft_maj}
    \end{minipage}\hfill
    \begin{minipage}{0.35\textwidth}
        \centering
        \includegraphics[width=\linewidth]{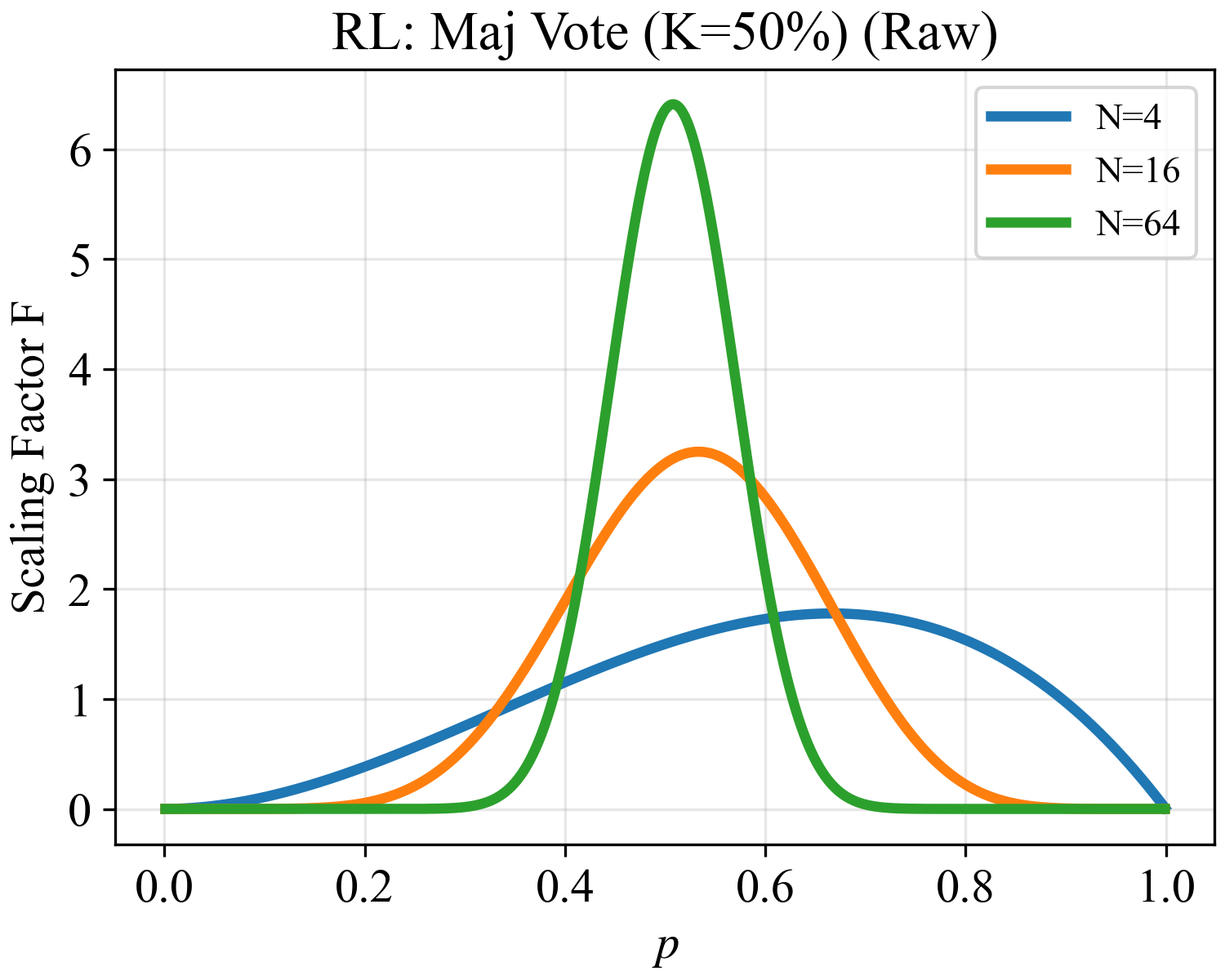}
        \caption{RL Scaling Factor for Majority Vote ($K=50\%$). The gradient focuses exclusively on the decision boundary, vanishing for "hopeless" or "secure" samples.}
        \label{fig:rl_maj}
    \end{minipage}
\end{figure}

\paragraph{SFT}
As illustrated in Figure \ref{fig:sft_maj}, the SFT scaling factor exhibits a monotonic decay, maximizing the gradient when $p \to 0$.

When the base probability $p$ is small, the probability of achieving a majority $\tilde{p}$ is exponentially small, resulting in a massive loss value. The scaling factor $w \approx k$ acts as a "bootstrap" signal. It effectively forces the model to prioritize samples where it is currently failing, utilizing the ground truth to "pull" the probability mass toward the decision boundary regardless of the current policy distribution. Once the problem has already been "solved" as in it will likely win a plurality vote, the model learns to stop providing any gradient to it, since it doesn't need to in order to get the problem right.

\paragraph{RL}
In contrast, Reinforcement Learning (Figure \ref{fig:rl_maj}) produces a "spotlight" centered on the decision boundary ($p \approx k/N$). It assigns near-zero weight to samples where the model is either "hopeless" ($p \approx 0$) or "secure" ($p \approx 1$).

\begin{figure}[h]
    \centering
    \includegraphics[width=0.35\textwidth]{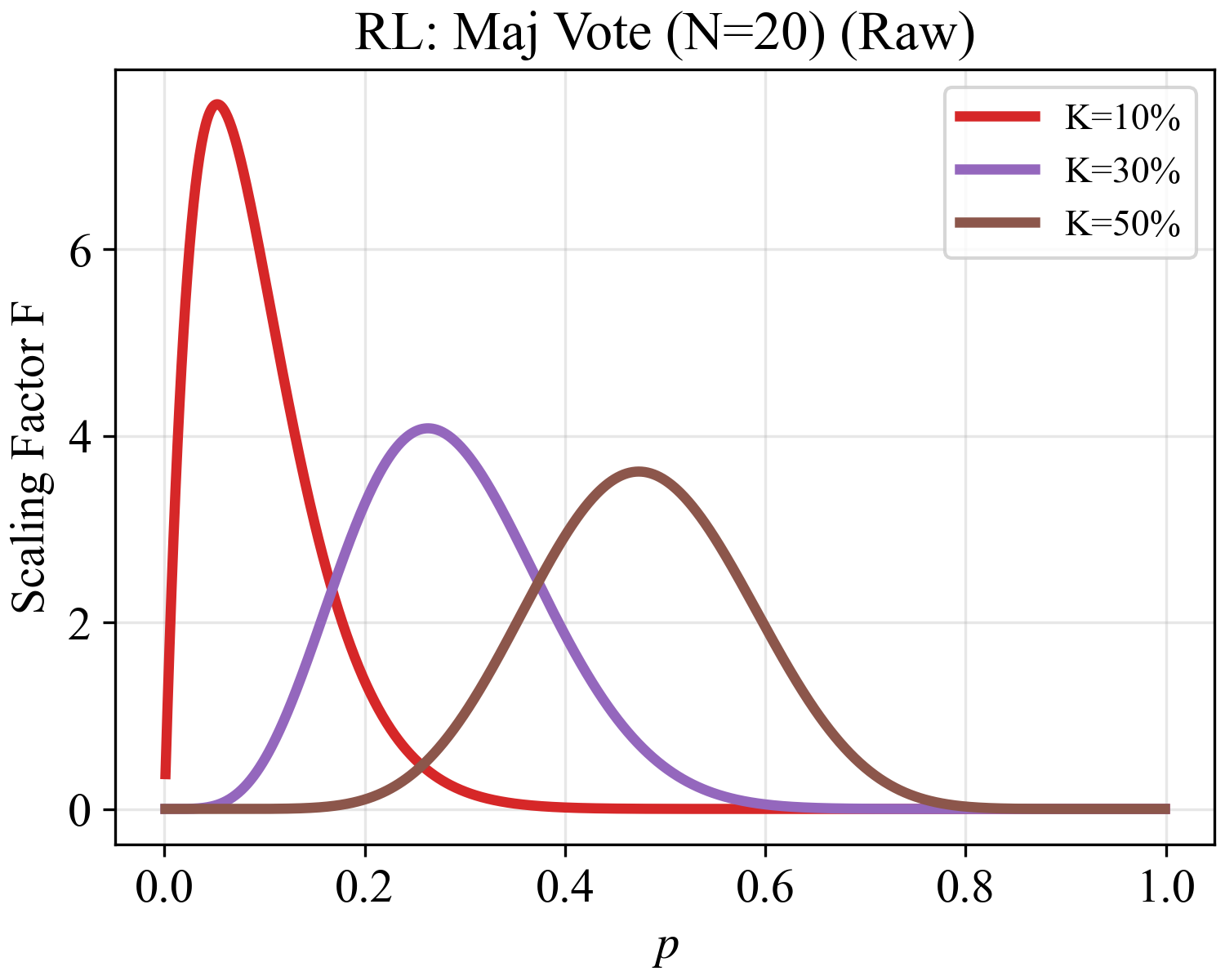}
    \caption{RL Scaling Factor for Majority Vote with varying thresholds $K$. The gradient acts as a "spotlight" centered on the current difficulty threshold. Early in training (Red), the threshold is low ($K \approx 10\%$), allowing the model to learn from low-$p$ samples. As the model improves (Purple $\to$ Brown), the threshold naturally rises, sliding the optimization window toward the final $50\%$ majority requirement.}
    \label{fig:rl_maj_dynamic}
\end{figure}

\paragraph{Solving the Cold Start Problem.}
The localized nature of the RL gradient creates a "cold start" vulnerability: if the model starts with $p \approx 0$, the gradient vanishes, and the model fails to learn.

As derived in Sec \ref{app:rl_strategies}, we resolve this not by altering the scaling curve, but by utilizing a dynamic threshold $k$. In a specific batch, the threshold $k$ is determined by the strength of the rival answers. Early in training, when the model is weak, the threshold $k$ drops (e.g., to $10\%$), shifting the gradient "spotlight" (The Red Curve in Figure \ref{fig:rl_maj_dynamic}) to overlap with the model's current low-$p$ distribution. This allows the model to progressively "slide" the window of optimization from chaos to consensus. A similar logic provides rationale to iteratively increase the $k$ for SFT.
\subsection{Best-of-$N$: The Upper-Tail Optimizer}

In the Best-of-$N$ setting, the system samples $N$ candidates and selects the single candidate with the maximum reward. While Best-of-$N$ for SFT reduces mathematically to Pass@$N$ (ensuring the ground truth appears once), in the Reinforcement Learning regime, the dynamics are distinct. The scaling factor, $N (P_{<y})^{N-1}$, depends not on absolute probability, but on the \textit{relative quantile} of the reward distribution.

This shift from absolute probability to relative rank introduces three distinct behaviors:

\begin{figure}[h]
    \centering
    \includegraphics[width=0.35\textwidth]{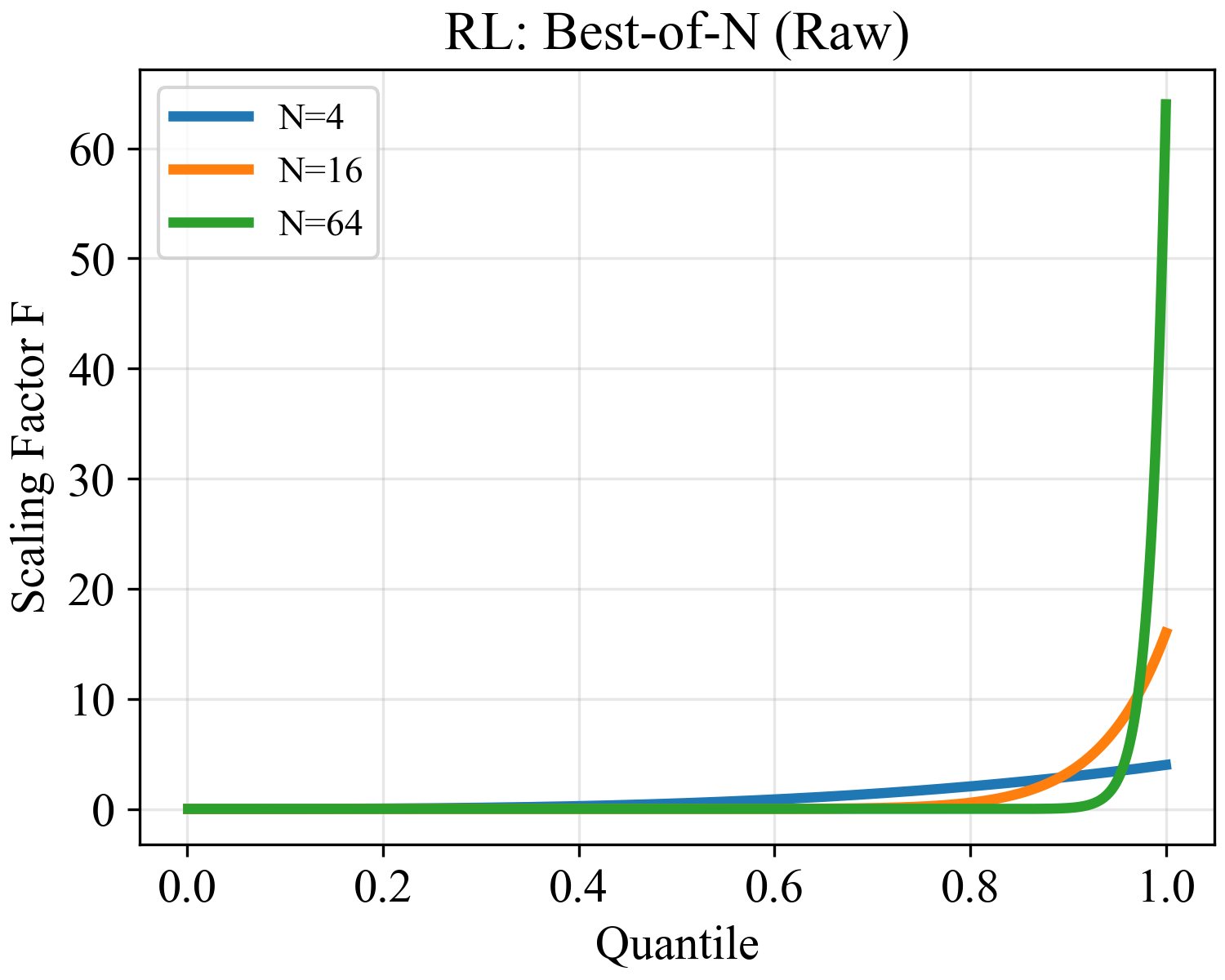}
    \caption{RL Scaling Factor for Best-of-$N$. The gradient weight grows exponentially with the sample's quantile, ignoring average outputs to focus on the top percentile.}
    \label{fig:rl_bon}
\end{figure}

\paragraph{1. Winner-Take-All Dynamics.} 
As visualized in Figure \ref{fig:rl_bon}, the gradient weight vanishes for the bottom percentile of samples and explodes for the top percentile. The objective effectively ignores "average" performance. If a sample is not in the top quantile of the model's potential outputs, it contributes zero signal. The model is only updated based on its best attempts, creating a "winner-take-all" feedback loop that aggressively reinforces the upper tail of the distribution.

\paragraph{2. Breaking Mode Collapse.}
Standard objectives maximize \textit{expected} utility. To minimize variance and maximize the average, these objectives encourage the model to collapse probability mass onto a single mode, usually a "safe," generic response. Diverging from this safe mode is typically penalized, as low-probability paths are treated as noise. 

Best-of-$N$ inverts this incentive. By validating risky, low-density paths, the operator prevents the model from getting stuck at the local maxima of "safe" responses. It encourages the distribution to develop \textit{heavy tails}, prioritizing the discovery of rare, high-reward solutions over the stability of the average output.

\paragraph{3. The "Safety Net" Effect.}
Because the inference strategy acts as a filter, the model is not penalized for generating $N-1$ failures, provided the $N$-th sample succeeds. This effectively creates a safety net during training. The objective signals to the model: \textit{"You are allowed to fail $N-1$ times, as long as your variance is high enough to produce one winner."} This theoretically justifies training models that are less coherent on average, but significantly more capable at peak performance.

\section{Quantifying Misalignment}
\label{app:gradient_alignment}

While the empirical results in Section \ref{sec:experiments} demonstrate that CAT improves performance, it is theoretically important to quantify \textit{when} standard objectives fail. In this section, we formalize the disconnect between training and testing as a resource allocation problem. We introduce the \textbf{Alignment Coefficient ($\mathcal{A}$)} as a global metric for transfer efficiency. The general subject of seeing how optimizing for one loss may affect success on a different loss is deeply discussed in \citet{gopalan2021omnipredictors}.

\subsection{Derivation: The Alignment Coefficient ($\mathcal{A}$)}

To quantify the "transfer efficiency" of a training objective to a test-time metric, we analyze the impact of a single gradient descent step.

Let $\mathcal{L}_{\text{train}}(\theta)$ be the training loss we minimize, and $\mathcal{U}_{\text{test}}(\theta)$ be the test-time utility we wish to maximize. We perform a parameter update with learning rate $\eta$:
\begin{equation}
    \theta_{t+1} = \theta_t - \eta \nabla_\theta \mathcal{L}_{\text{train}}
\end{equation}
We are interested in the resulting change in the test utility, $\Delta \mathcal{U}_{\text{test}} = \mathcal{U}_{\text{test}}(\theta_{t+1}) - \mathcal{U}_{\text{test}}(\theta_t)$. Applying a first-order Taylor expansion around $\theta_t$:
\begin{align}
    \Delta \mathcal{U}_{\text{test}} &\approx \langle \nabla_\theta \mathcal{U}_{\text{test}}, (\theta_{t+1} - \theta_t) \rangle \nonumber \\
    &= \langle \nabla_\theta \mathcal{U}_{\text{test}}, -\eta \nabla_\theta \mathcal{L}_{\text{train}} \rangle \nonumber \\
    &= -\eta \langle \nabla_\theta \mathcal{U}_{\text{test}}, \nabla_\theta \mathcal{L}_{\text{train}} \rangle
\end{align}
This result indicates that the immediate gain in test performance is proportional to the inner product of the training and test gradients.

As derived in Appendix \ref{app:sft_gradient}, for generative models, these gradients share the same directional component (the score function $\mathbf{g} = \nabla_\theta \log \pi_\theta(y|x)$) but differ in their scalar weighting functions $w(p)$:
\begin{align}
    \nabla \mathcal{L}_{\text{train}} &= -w_{\text{train}}(p) \cdot \mathbf{g} \\
    \nabla \mathcal{U}_{\text{test}} &= w_{\text{test}}(p) \cdot \mathbf{g}
\end{align}
Substituting this factorization into the inner product, we take the expectation over the data distribution $p \sim \mathcal{D}$. In a rigorous setting, this expectation depends on the Empirical Confidence Density $f_{\mathcal{D}}(p; \theta)$ (defined as the probability density function of the ground-truth likelihoods $p(y^*|x)$ across the dataset $\mathcal{D}$ given the current model parameters $\theta$), which shifts as the model matures.

\begin{equation}
    \mathbb{E}[\Delta \mathcal{U}_{\text{test}}] \propto \eta \int_0^1 w_{\text{train}}(p) w_{\text{test}}(p) \cdot f_{\mathcal{D}}(p; \theta) \cdot \mathbb{E}[\|\mathbf{g}\|^2] \, dp
\end{equation}

While the true alignment is state-dependent, estimating $f_{\mathcal{D}}(p)$ requires tracking specific training dynamics. To derive a generalized metric, we assume $p \sim \text{Unif}[0,1]$. This acts as a \textbf{trajectory average}: since any successful model must traverse the probability space from ignorance ($p \approx 0$) to mastery ($p \approx 1$), integrating uniformly measures the alignment averaged over the entire lifecycle of training. This yields the \textbf{Alignment Coefficient ($\mathcal{A}$)}:

\begin{equation}
    \mathcal{A}(\text{Train}, \text{Test}) = \frac{\langle w_{\text{train}}, w_{\text{test}} \rangle}{\|w_{\text{train}}\|_2 \cdot \|w_{\text{test}}\|_2} = \frac{\int_0^1 w_{\text{train}}(p) w_{\text{test}}(p) \, dp}{\sqrt{\int_0^1 w_{\text{train}}^2(p) \, dp} \sqrt{\int_0^1 w_{\text{test}}^2(p) \, dp}}
\end{equation}

This metric provides a normalized scalar $\mathcal{A} \in [0, 1]$. By decoupling the objective function from the data distribution, $\mathcal{A}$ allows us to measure the \textit{intrinsic} compatibility of the training loss with the test metric, independent of the model's current calibration.
\begin{itemize}
    \item \textbf{$\mathcal{A} \approx 1$ (Perfect Alignment):} The training objective applies gradient pressure exactly where the test metric demands it, maximizing $\Delta \mathcal{U}_{\text{test}}$ regardless of data density.
    \item \textbf{$\mathcal{A} \approx 0$ (Orthogonality):} The training objective is active in regions where the test metric has zero utility. In this case, $\Delta \mathcal{U}_{\text{test}} \approx 0$, meaning the training compute is inherently wasted relative to the target task.
\end{itemize}
\subsection{Physical Interpretation: Gradient Misallocation}
\label{sec:capacity_misallocation}

Why should we care about the abstract scalar $\mathcal{A}$? Because it serves as a rigorous proxy for \textbf{Gradient Misallocation}.

Neural network training is a zero-sum game of resource allocation. The "resource" is the finite gradient magnitude the model can absorb before convergence. We define the \textbf{Effective Support} of an objective as the region $\mathcal{S} = \{p \mid w(p) > \epsilon\}$.
Misalignment occurs when the supports of $w_{\text{train}}$ and $w_{\text{test}}$ are disjoint.

\begin{figure}[h]
    \centering
    \includegraphics[width=0.95\textwidth]{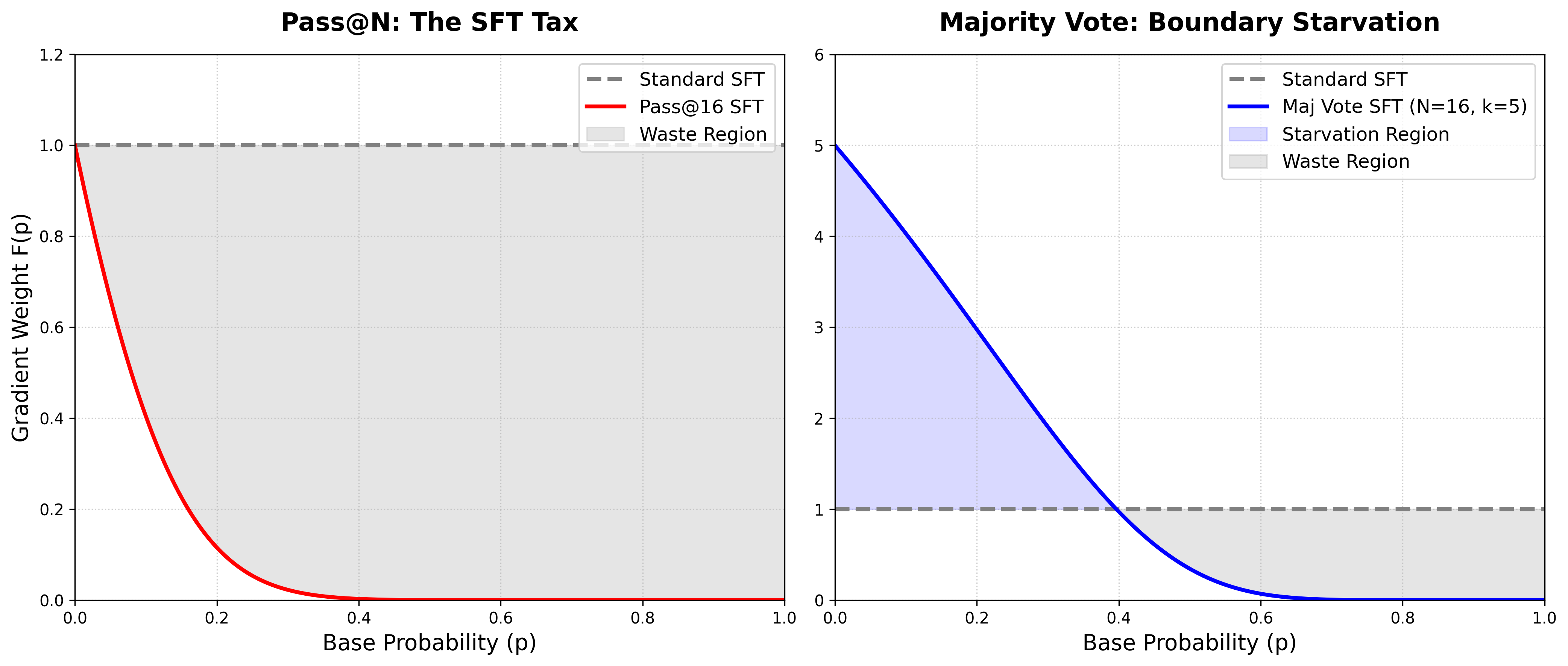}
    \caption{\textbf{The Support Mismatch:} Side-by-side visualization of strategy-aware SFT scaling factors $w(p)$ compared to the uniform gradient pressure of Standard SFT (Gray dashed line). \textbf{Left (Pass@$N$):} The objective acts as an efficiency regularizer, exposing a massive "Waste Region" where standard SFT unnecessarily optimizes already-solved problems (the "SFT Tax"). \textbf{Right (Majority Vote):} The objective concentrates gradient utility sharply at the decision boundary ($p \approx k/N$), revealing a "Starvation Region" where SFT fails to provide sufficient localized pressure to tip the consensus. Across both strategies, the shaded non-overlapping regions represent fundamentally misallocated gradient.}
    \label{fig:gradient_alignment}
\end{figure}

As illustrated in Figure \ref{fig:gradient_alignment}, we can decompose the probability space into three regions:
\begin{enumerate}
    \item \textbf{The Efficient Region ($\mathcal{S}_{\text{overlap}}$):} Where both objectives are active.
    \item \textbf{The Waste Region ($\Omega_{\text{waste}}$):} Where the training objective applies pressure, but the test metric is already satiated ($w_{\text{test}} \approx 0$). Gradients here are "wasted" on perfecting samples that are already good enough.
    \item \textbf{The Starvation Region ($\Omega_{\text{starve}}$):} Where the test metric demands improvement ($w_{\text{test}} > 0$), but the training objective provides no signal ($w_{\text{train}} \approx 0$).
\end{enumerate}

The Alignment Coefficient $\mathcal{A}$ is mathematically dominated by the integral over the overlap $\mathcal{S}_{\text{overlap}}$. Therefore, a low $\mathcal{A}$ guarantees high gradient misallocation. It allows us to detect inefficiency without needing to define arbitrary "satiation thresholds" for every metric.

\subsection{Transfer Analysis}
\label{sec:transfer_matrix}

Finally, we can use the Alignment Coefficient to answer a practical question: \textit{How much does training for hyperparameters $\phi_{train}$ improve performance at a different inference hyperparameter $\phi_{test}$?} We stick to the Pass@$N$ case for this subsection where $\phi_{train} = N_{train}$ and $\phi_{test} = N_{test}$. The coefficient allows us to analytically predict the test-time budget $N_{test}$ where a model trained for a larger budget $N_B$ will overtake a model trained for a smaller budget $N_A$. We find this theoretical crossover point by setting their alignment coefficients equal to each other:

\begin{equation}
    \mathcal{A}(N_A, N_{test}) = \mathcal{A}(N_B, N_{test})
\end{equation}

To empirically validate this framework, we numerically integrated the Alignment Coefficients for the Pass@$N$ SFT objective across our experimental training budgets $N_{\text{train}} \in \{1, 4, 16, 64\}$ against a continuous spectrum of inference budgets $N_{\text{test}} \in [1, 64]$. By tracking the trajectory of these coefficients, we can identify the exact theoretical crossover points where $\mathcal{A}(N_{\text{larger}}, N_{\text{test}}) > \mathcal{A}(N_{\text{smaller}}, N_{\text{test}})$. 

This numerical evaluation yields three specific predictions for when a larger training budget should overtake a smaller one:
\begin{enumerate}
\item $N=4$ overtakes the baseline ($N=1$) at $N_{\text{test}} = 2$, where the alignment coefficient for $N=4$ reaches $0.9295$ compared to $0.9100$ for the baseline.
\item $N=16$ overtakes $N=4$ at $N_{\text{test}} = 8$, with an alignment score of $0.9341$ edging out $0.9332$.
\item $N=64$ overtakes $N=16$ at $N_{\text{test}} = 32$, achieving a coefficient of $0.9344$ against $0.9343$.
\end{enumerate}

These theoretical predictions align remarkably well with the empirical observations from the MATH benchmark (\Cref{tab:pass_results}, \Cref{fig:pass_k_delta}). 

First, the baseline model dominates at Pass@1, but the $N=4$ model catches up around Pass@2. This bounds the first crossover exactly where the theory places it ($N_{\text{test}}=2$).

Next, the theoretical framework predicts that the $N=16$ model will overtake the $N=4$ model exactly at $N_{\text{test}}=8$. Looking at our empirical data, Pass@8 is just around the budget where the $N=16$ model ($40.33\%$) eclipses the $N=4$ model ($39.87\%$). 

Finally, the framework predicts $N=64$ will overtake $N=16$ at $N_{\text{test}}=32$. Empirically, at Pass@32, the two models are neck-and-neck ($59.29\%$ vs $59.85\%$), representing the exact empirical threshold of the crossover before the $N=64$ model pulls definitively ahead by Pass@64 ($67.60\%$ vs $66.20\%$).

This tight correspondence between the purely theoretical Alignment Coefficient and the observed scaling laws demonstrates that $\mathcal{A}(\text{Train}, \text{Test})$ can be used as a cheap predictive tool to match training hyperparameters to an expected deployment budget without relying on expensive empirical sweeps.

\section{Why Not Just Perform Full TTS Rollouts?}
\label{app: Full TTS Rollout}
A seemingly intuitive alternative to the CAT framework is to simply simulate the TTS exactly during the RL phase. For instance, to train for Pass@$N$, one could generate $N$ full rollouts per prompt, compute the empirical aggregate reward (e.g., $1$ if any of the $N$ samples are correct, $0$ otherwise), and use this as the reinforcement signal. With a leave-one-out baseline to stabilize the variance of this strategy, this is what \citet{tang2025optimizing} do in their work. While conceptually straightforward and more effective than standard training in this context, this empirical rollout approach suffers from two fatal flaws: computational inflexibility and credit assignment.

\subsection{The Computational Bottleneck: Inflexibility of Batch Size}

The most immediate barrier to empirical TTS rollouts is the strict coupling of the training compute budget to the target inference budget. If the target deployment strategy is Pass@$64$ or Majority Vote over 128 samples, an empirical training objective requires generating exactly $N=64$ or $N=128$ complete trajectories for every single prompt in the training batch.

This $O(N)$ scaling of generation cost per step creates a severe computational bottleneck. It drastically limits the number of unique prompts the model can see per epoch with fixed compute, restricting dataset coverage. In contrast, CAT completely decouples the training rollout budget ($M$) from the target inference budget ($N$). By using the analytical derivative of the effective probability $\frac{\partial \tilde{p}}{\partial p}$, CAT can train a model for Pass@$100$ while only generating $M=4$ rollouts per prompt during training, resulting in massive efficiency gains.

\subsection{Credit Assignment: Learning Signal Degeneracy}

The second, more insidious issue with full TTS rollouts is the mathematical degeneracy of the learning signal when dealing with aggregate rewards and the inability to assign credit to specific outputs.

Consider the empirical Pass@$N$ objective, where the aggregate reward for a batch of $N$ generations is simply the global maximum of the individual binary rewards: $R_{\text{batch}} = \max(r_1, \dots, r_N)$. In a standard RL setup (e.g., REINFORCE with a leave-one-out baseline like in \citet{tang2025optimizing}), the advantage $A_i$ for a specific sample $y_i$ is the difference between the global aggregate reward and the aggregate reward excluding $y_i$. Mathematically, this is expressed as:

\[A_i = \max(r_{1:N}) - \max(r_{-i})\]

This formulation reveals a critical flaw. Imagine we are training for Pass@$64$, and the model generates 3 correct answers out of the 64 rollouts. The global maximum reward is $1$. If we evaluate the advantage for any of the 3 correct answers by "leaving it out," the remaining batch still contains 2 correct answers. Therefore, the leave-one-out maximum is still $1$. The advantage calculation becomes:

\[A_i = 1 - 1 = 0\]

If the model solves the problem more than once in the batch, the advantage for every single sample zeroes out. The model receives no learning signal if it gets 3 correct instead of 2 correct, or 64 correct instead of 2 correct. As $N$ scales, the probability of this issue coming into play scales, meaning information is wasted.

CAT avoids this degeneracy entirely. Instead of relying on a discrete empirical step-function, CAT operates on the continuous, analytical gradient of the expected test-time probability. Even if the empirical batch is saturated with correct answers, the scalar weight $w(p, \phi)$ evaluates the underlying probability mass $p$. This ensures a smooth, continuous gradient that always accurately reflects the marginal utility of shifting the probability distribution, preventing the signal from vanishing.

\subsection{Rao--Blackwellization for Pass@$N$.}
Another interpretation of our utility relative to this style of method is that we are Rao-Blackwellizing a given update with respect to other generations in the batch. 

Fix a prompt $x$ and draw $y_1,\ldots,y_N$ independently from $\pi_\theta(\cdot\mid x)$, with $N\ge2$. Let $r_i=R(y_i\mid x)\in\{0,1\}$, $p=P(r_i=1\mid x)$, and $s_i=\nabla_\theta\log\pi_\theta(y_i\mid x)$. For binary rewards, the leave-one-out advantage is
\[
A_i=\max(r_{1:N})-\max(r_{-i})
   =r_i\prod_{j\ne i}(1-r_j).
\]
The leave-one-out gradient estimator is $G_{\mathrm{LOO}}=\sum_{i=1}^N A_i s_i$.

Choose $I$ uniformly from $\{1,\ldots,N\}$, independently of the rollouts, and define $Z=N A_I s_I$. Since $\mathbb{E}[Z\mid y_1,\ldots,y_N]=G_{\mathrm{LOO}}$, the two estimators have the same expectation. Conditioning instead on the selected rollout gives
\[
G_{\mathrm{CAT}}
:=\mathbb{E}[Z\mid I,y_I]
=N r_I s_I\,
  \mathbb{E}\!\left[\prod_{j\ne I}(1-r_j)\mid I,y_I\right]
=N(1-p)^{N-1}r_I s_I.
\]
This recovers the raw CAT Pass@$N$ weight. By the score-function identity,
\[
\mathbb{E}[G_{\mathrm{CAT}}]
=N(1-p)^{N-1}\nabla_\theta p
=\nabla_\theta\!\left[1-(1-p)^N\right].
\]

Assuming finite second moments, the law of total variance gives, for every fixed vector $u$,
\[
\operatorname{Var}(u^\top Z)
=
\operatorname{Var}(u^\top G_{\mathrm{CAT}})
+
\mathbb{E}\!\left[
  \operatorname{Var}(u^\top Z\mid I,y_I)
\right]
\ge
\operatorname{Var}(u^\top G_{\mathrm{CAT}}).
\]
Thus, CAT replaces the random event that all other rollouts fail with its conditional probability $(1-p)^{N-1}$. This calculation uses the true success probability $p$ and the raw estimator, before probability estimation, normalization, or clipping.
\section{Experimental Details: Pass@$N$ SFT}
\label{app:implementation}
This and all other experiments were run on an Nvidia V100 GPU on a cluster with around 20 Gigabytes of storage.

\subsection{Hyperparameters and Training Configuration}
All models were initialized from \texttt{Mistral-7B-Instruct-v0.2} (4-bit quantized) and fine-tuned using the \texttt{Unsloth} library \cite{unsloth}. We utilized a global effective batch size of 4 (real batch size 1 with 4 gradient accumulation steps). The specific hyperparameters are detailed in Table \ref{tab:hyperparams}.

\begin{table}[h]
\centering
\caption{Training hyperparameters for Test-Time Strategy experiments.}
\label{tab:hyperparams}
\begin{tabular}{@{}ll@{}}
\toprule
\textbf{Category} & \textbf{Value} \\ \midrule
\textbf{Optimization} & \\
Optimizer & AdamW (32-bit) \\
Learning Rate & $5\times 10^{-6}$ \\
Scheduler & Linear Decay \\
Batch Size & 1 (Per Device) / 4 (Effective) \\
Max Gradient Norm & 0.3 \\
\midrule
\textbf{Schedule} & \\
Warmup Phase & 2 Epochs (Standard CE) \\
Strategy Phase & 3 Epochs (Strategy Objective) \\
Total Training & 5 Epochs \\
\midrule
\textbf{Data Processing} & \\
Dataset & MATH (Train Subset: 5,000 | Test: 500) \\
Max Sequence Length & 512 (Input) + 512 (Output) \\ 
Inference Budgets ($N$) & $4, 16, 64$ \\ \bottomrule
\end{tabular}
\end{table}

\subsection{Pass@$N$ Implementation}
\paragraph{Loss Computation.} Directly computing Pass@$N$ loss via probabilities $p = \exp(\text{logits})$ leads to floating-point underflow when $p \to 0$. To address this, our implementation operates primarily in log-space and includes a branch for low-probability events.

Let $\ell_{\text{seq}}$ be the sequence-level negative log-likelihood. The base probability is $p = \exp(-\ell_{\text{seq}})$. The loss is defined as:
\begin{equation}
    \mathcal{L}_{\text{Pass}@N} = \begin{cases} 
      \ell_{\text{seq}} - \log(N) & \text{if } p < 10^{-4} \\
      -\log\left(1 - (1 - p)^N\right) & \text{otherwise}
   \end{cases}
\end{equation}
For the "tiny" probability regime ($p < 10^{-4}$), we approximate Pass@$N \approx Np$. Maximizing $\log(Np)$ is equivalent to minimizing $-\log p - \log N$. This linear approximation prevents gradient collapse when the model is far from the solution manifold.

\paragraph{Evaluation.} We evaluated the models by generating a fixed pool of $K=64$ samples per problem. To efficiently estimate Pass@$k$ for varying budgets $k \le 64$ without resampling, we utilized the unbiased estimator for selection without replacement:
\begin{equation}
    \text{Pass}@k = 1 - \prod_{i=0}^{k-1} \frac{n - c - i}{n - i}
\end{equation}
where $n=64$ is the pool size and $c$ is the count of correct samples.

\section{Majority Vote SFT Implementation Details}
\label{app:maj_implementation}
This and all other experiments were run on an Nvidia V100 GPU on a cluster with around 20 Gigabytes of storage.
\subsection{Training Configuration and Data Processing}
Similar to the Pass@$N$ experiments, models trained for Majority Vote were initialized from \texttt{Mistral-7B-Instruct-v0.2} (4-bit quantized) and fine-tuned using the \texttt{Unsloth} library. The optimization hyperparameters (learning rate, batch size, etc.) follow the standard configuration detailed in Table \ref{tab:hyperparams}. We used the Peft \cite{peft} and Transformers \cite{wolf2020transformers} libraries as backbones for our experiments.

Crucially, because Majority Vote requires the model to have a non-trivial baseline success rate to form a meaningful consensus, the training and testing subsets of the MATH dataset were filtered to include only difficulty levels 1 through 3. This ensures that the base probability $p$ is sufficiently far from zero, preserving numerical stability during gradient scaling and providing a dense enough signal for plurality dynamics to take effect.

\subsection{Implicit Log-Space Loss Computation}
Directly computing the effective probability $\tilde{p}$ for Majority Vote involves summing binomial coefficients: $\tilde{p} = \sum_{i=k}^{N} \binom{N}{i} p^i (1-p)^{N-i}$. This naive calculation leads to severe numerical instability, specifically underflow when probabilities are small, or overflow when $N$ is large.

To resolve this, our implementation calculates the loss strictly in log-space using a Log-Sum-Exp reduction. For a given input sequence, we compute the per-token negative log-likelihood (NLL) via standard Cross-Entropy, ignoring padding tokens. Summing these yields the sequence-level NLL, $\ell_{\text{seq}}$, from which we derive the base probability $p = \exp(-\ell_{\text{seq}})$. To ensure stability during training, $p$ is tightly clamped to the interval $[10^{-6}, 1 - 10^{-6}]$.

Let $f$ be the target consensus fraction $k/N$ (e.g., $0.5$ for simple majority). The absolute token threshold is dynamically computed as $k = \max(2, \min(N, \lceil N \cdot f \rceil))$. We iterate through each possible winning vote count $i \in [k, N]$ and compute the log-probability of that specific outcome using the $\log \Gamma$ function to continuously approximate the combinations:
\begin{equation}
    \log T_i = \log \Gamma(N+1) - \log \Gamma(i+1) - \log \Gamma(N-i+1) + i \log(p) + (N-i) \log(1-p)
\end{equation}
The final success log-probability is aggregated using the Log-Sum-Exp trick over all valid $i$:
\begin{equation}
    \log \tilde{p} = \text{LogSumExp}(\log T_k, \dots, \log T_N)
\end{equation}
The computed $\log \tilde{p}$ is then negated and normalized by the number of valid tokens in the target sequence to form the final strategy-aware loss.

\subsection{Evaluation via Bootstrapped Sampling}
Unlike Pass@$N$, which benefits from a simple unbiased estimator for test-time scaling, Majority Vote requires aggregating across the plurality of competing answers. Running exact combinatorial evaluations across large sets is computationally prohibitive.

To accurately estimate Majority Vote accuracy across varying inference budgets (up to $k = 128$), we employed a highly smoothed Monte Carlo bootstrapping approach. For each problem in the test set, we generated a pool of 128 candidate responses using a sampling temperature of $0.8$ and Top-$p$ of $0.95$.

To evaluate the model at a specific inference budget $k$, we ran 500 bootstrap iterations per problem. In each iteration, we sampled $k$ responses from the generated pool without replacement. The textual answers were parsed, and the most common valid answer (the mode) was selected as the system's final output. The overall Maj@$k$ performance is reported as the empirical win rate of this system output averaged across all 500 iterations and all test problems.

\subsection{Hyperparameter Sensitivity (The Sweep)}
\label{app:maj_sweep}

For the Majority Vote objective, the consensus threshold fraction $f$ is a critical hyperparameter. If $f$ is too low, the objective is trivial; if $f$ is too high, the gradient vanishes because the consensus is mathematically unreachable from the current policy. We performed a grid search over consensus fractions $f \in \{0.25, 0.33, 0.40\}$ for the different budgets, below we include an example in \Cref{fig:maj64_sweep}.

\begin{figure}[h]
    \centering
    \includegraphics[width=0.85\textwidth]{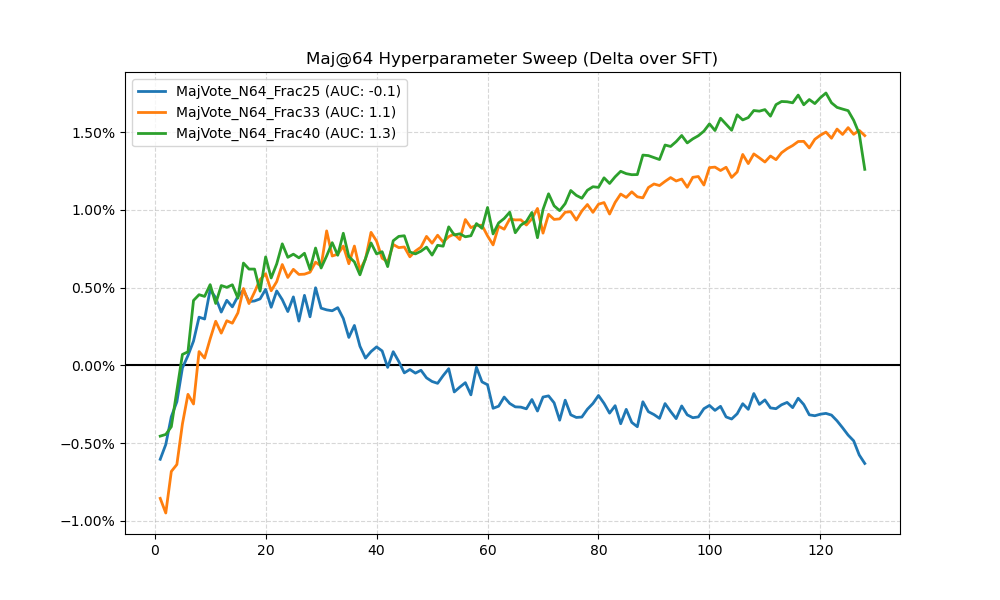}
    \caption{\textbf{Maj@64 Hyperparameter Sweep.} Performance delta relative to SFT. The loose threshold ($25\%$, Blue) fails to improve over the baseline, while the stricter threshold ($40\%$, Green) recovers performance.}
    \label{fig:maj64_sweep}
\end{figure}

As shown in Figure \ref{fig:maj64_sweep}, the model is highly sensitive to this threshold. The model trained with $f=0.25$ (requiring only 16/64 votes) actually degraded performance (AUC $-0.1$). We interpret this as a "weak signal" failure mode: the constraint was so loose that the model could satisfy it with a high-entropy distribution, rather than being forced to converge on the ground truth. Increasing the requirement to $f=0.40$ (26/64 votes) provided the necessary "pressure" to guide the model toward a robust consensus. For the Maj@$8$ model $f=.33$ won out and for the Maj@$16$ model $f=.25$ won out.

\section{Pass@$N$ Reinforcement Learning Implementation Details}
\label{app:rl_pass_implementation}
This and all other experiments were run on an Nvidia V100 GPU on a cluster with around 20 Gigabytes of storage.
\subsection{Training Configuration and Data Processing}
For the Reinforcement Learning experiments, we initialized the policy from a checkpoint that had already undergone a Supervised Fine-Tuning (SFT) warmup. This warmup establishes a baseline capability for formatting mathematical reasoning traces and extracting final answers. 

We utilized the \texttt{trl} library's \cite{vonwerra2020trl} Group Relative Policy Optimization (GRPO) trainer, wrapped with our custom strategy-aware weighting logic, and fine-tuned using Low-Rank Adaptation (LoRA) \cite{hu2022lora} on the 4-bit quantized \cite{dettmers2023qlora} \texttt{Mistral-7B-Instruct-v0.2} model. We use the AdamW \cite{loshchilov2018decoupled} optimizer. A critical detail of our GRPO setup is the group size (number of generations per prompt), which we set to $G=4$ to balance variance reduction with memory constraints. The full optimization hyperparameters are detailed in Table \ref{tab:rl_hyperparams}.

\begin{table}[h]
\centering
\caption{Hyperparameters for Pass@$N$ GRPO Training.}
\label{tab:rl_hyperparams}
\begin{tabular}{@{}ll@{}}
\toprule
\textbf{Category} & \textbf{Value} \\ \midrule
\textbf{Optimization} & \\
Optimizer & AdamW (32-bit) \\
Learning Rate & $1\times 10^{-6}$ \\
Max Gradient Norm & 1.0 \\
KL Penalty ($\beta$) & 0.05 \\
\midrule
\textbf{GRPO Specifics} & \\
Group Size ($G$) & 4 (Rollouts per prompt) \\
Batch Size & 4 \\
Epochs & 1 \\
\midrule
\textbf{Generation Limits} & \\
Max Prompt Length & 256 tokens \\
Max Completion Length & 300 tokens \\ \bottomrule
\end{tabular}
\end{table}

\subsection{Reward Formulation}
Because Pass@$N$ evaluates the ultimate correctness of a generation rather than the intermediate steps, we employed a binary, outcome-based reward function. For a generated sequence $y$, we extract the final answer (either via LaTeX \texttt{\textbackslash boxed\{\}} tags or trailing numerics) and compare it against the ground truth. The raw reward is $R(y) = 1.0$ for a correct match and $R(y) = 0.0$ otherwise. GRPO internally normalizes these raw rewards within the sampled group of size $G=4$ to compute the baseline advantage $A_i$ for each sequence.

\subsection{Log-Weighted Advantage Scaling (The Core Logic)}
To adapt GRPO for Pass@$N$, we modify the objective by applying a sample-specific scaling factor $w_i$ to the computed advantage $A_i$. Our primary implementation uses the "Log-Weighted" formulation (derived from the SFT objective in \Cref{app:sft_passn}), which includes an adaptive gain controller to stabilize the variance of the RL updates.

\paragraph{Estimating the Sequence Probability ($p$).} 
To compute the scaling factor, we require the model's base likelihood of generating the current sequence $y_i$. During the forward pass, we extract the log-softmax distribution over the vocabulary and gather the log-probabilities of the actual generated tokens. We mask out the prompt tokens and sum the completion token log-probabilities to yield the raw sequence log-probability, $\ell_{\text{seq}}$.

Because raw sequence probabilities in LLMs can vanish to infinitesimally small values, directly exponentiating $\ell_{\text{seq}}$ often results in a floating-point underflow to $0.0$. To maintain a usable gradient signal, we aggressively clamp the sequence log-probability before exponentiation:
\begin{equation}
    p_i = \exp(\text{clamp}(\ell_{\text{seq}}, -50.0, 0.0))
\end{equation}
This establishes a strict numerical floor, ensuring $p_i$ does not drop below $\approx 1.9 \times 10^{-22}$. For further stability during division, $p_i$ is subsequently clamped to the interval $[10^{-10}, 1 - 10^{-10}]$.

\paragraph{Computing the Stabilized Weight ($w_i$).}
The Log-Weighted scaling factor incorporates the raw marginal utility and the $p/\tilde{p}$ normalization term. To prevent overflow when $N$ is large or $p$ is small, we compute the numerator in log-space:
\begin{equation}
    \text{Numerator} = \exp\Big(\log(N) + \log(p_i) + (N - 1) \log(1 - p_i)\Big)
\end{equation}
The denominator represents the effective test-time probability $\tilde{p}$, computed as $1 - (1 - p_i)^N$. The final weight is $w_i = \text{Numerator} / (\tilde{p} + 10^{-10})$.

\paragraph{Gradient Weight Clipping.}
Even with the log-weighted formulation acting as an adaptive gain controller, batches containing exceptionally "hard" examples can occasionally produce large scaling factors. To definitively prevent variance explosion and policy destabilization, we apply a hard clipping ceiling to the final weight:
\begin{equation}
    w_i = \text{clamp}(w_i, 0.0, 50.0)
\end{equation}
The final strategy-aware loss for the sample becomes $\mathcal{L}_i = - (w_i \cdot A_i) \cdot \ell_{\text{seq}}$, effectively scaling the learning rate of the specific sequence proportional to its marginal utility under the Pass@$N$ metric.

\subsection{Test-Time Evaluation}
To evaluate the scaling laws of the trained RL models, we evaluated them on a held-out set of 500 MATH problems. For each prompt, we generated $K=32$ independent solutions using a sampling temperature of $0.7$ and Top-$p$ of $0.9$. 

We computed the Pass@$k$ performance at inference budgets $k \in \{1, \dots, 32\}$ using the unbiased estimator for selection without replacement. Specifically, if $c$ is the total number of correct answers found in the $K$ generations, the expected Pass@$k$ score is:
\begin{equation}
    \text{Pass}@k = 1 - \prod_{i=0}^{k-1} \left(1 - \frac{c}{K - i}\right)
\end{equation}
This allows us to construct smooth, variance-reduced scaling curves across varying inference budgets without needing to repeatedly resample from the model.

\section{Majority Vote RL Implementation Details}
\label{app:maj_details}
This and all other experiments were run on an Nvidia V100 GPU on a cluster with around 20 Gigabytes of storage.

\subsection{Pipeline Configuration}
We utilized the \texttt{Mistral-7B-Instruct-v0.2} model as our policy $\pi_\theta$. To manage memory constraints on consumer-grade hardware, we employed QLoRA (4-bit quantization with NF4 data type) via the \texttt{bitsandbytes} library \cite{dettmers2022bitsandbytes}. The policy was trained using Low-Rank Adaptation (LoRA) with rank $r=16$ and $\alpha=32$, targeting the \texttt{q\_proj} and \texttt{v\_proj} modules.

The training pipeline consisted of two phases:
\begin{enumerate}
    \item \textbf{SFT Warmup:} The model was first fine-tuned for 3 epochs on the target dataset (MATH levels 1--3) using standard Cross-Entropy loss. This phase stabilizes the instruction-following capabilities and ensures the model outputs valid reasoning traces.
    \item \textbf{Strategy-Aware RL:} The warmed-up model was then trained for 1 epoch using our custom weighted gradient estimator.
\end{enumerate}

\subsection{Dynamic Consensus Thresholding}
\label{sub:adaptive_k}

A critical challenge in optimizing for Majority Vote is determining the required consensus threshold $k$ during training. While at test time $k$ is fixed (e.g., $\lfloor N/2 \rfloor + 1$), during training with small batch sizes (rollouts), the "strength" of the rival answers fluctuates.

We implemented an \textbf{Adaptive $k$ Scaling} mechanism. For a target inference budget $N_{target}$ (e.g., 4 or 8) and a training rollout size $M$ (set to 8 in our experiments), we project the difficulty of the current prompt to the target budget.

Let $c_{max}$ be the count of the most frequent \textit{incorrect} answer in the current rollout of size $M$. The projected number of adversary votes in the target budget $N_{target}$ is:
\begin{equation}
    E[\text{wrong}] = c_{max} \times \frac{N_{target}}{M}
\end{equation}
The required consensus $k_{dynamic}$ is strictly one vote greater than the projected adversary:
\begin{equation}
    k_{dynamic} = \lfloor E[\text{wrong}] \rfloor + 1
\end{equation}
We clamp this value such that $2 \le k_{dynamic} \le \lfloor N_{target}/2 \rfloor + 1$. This ensures that the gradient "spotlight" (the weighting factor) focuses on the margin of victory required to beat the specific rival present in the current generation, rather than a static vacuum.

\subsection{Advantage Computation and Zero-Mean Penalty}
We utilized a simplified Group Relative Policy Optimization (GRPO) framework. For each prompt $x$, we generated $M=8$ rollouts. The binary rewards $r_i \in \{0, 1\}$ were normalized to compute advantages $A_i$:
\begin{equation}
    A_i = \frac{r_i - \mu_r}{\sigma_r + \epsilon}
\end{equation}
\textbf{The Silence Penalty:} A common failure mode in reasoning-heavy RL is "total collapse," where all $M$ rollouts are incorrect ($\mu_r = 0$). Standard normalization would result in $A_i = 0$ (undefined gradient). To explicitly penalize this state and force exploration, we implemented a shift:
\begin{equation}
    \text{if } \mu_r = 0, \quad A_i \leftarrow A_i - 0.5
\end{equation}
This ensures that if the model produces a batch of pure failures, it receives a negative reinforcement signal, pushing the policy away from the current mode.

\subsection{Hyperparameters}
Table \ref{tab:maj_rl_params} lists the exact hyperparameters used in the experiments. The learning rate for the RL phase was significantly lower than the SFT phase ($1\times 10^{-6}$) to prevent policy collapse given the high variance of the weighting factors.

\begin{table}[h]
\centering
\caption{Hyperparameters for Majority Vote RL Experiments.}
\label{tab:maj_rl_params}
\begin{tabular}{@{}ll@{}}
\toprule
\textbf{Parameter} & \textbf{Value} \\ \midrule
\textbf{Model Config} & \\
Base Model & Mistral-7B-Instruct-v0.2 \\
LoRA Rank ($r$) / Alpha & 16 / 32 \\
Target Modules & \texttt{q\_proj}, \texttt{v\_proj} \\
Quantization & 4-bit (NF4) \\
\midrule
\textbf{RL Optimization} & \\
Rollouts per Prompt ($M$) & 8 \\
Target Budgets ($N_{target}$) & 4, 8 \\
Learning Rate & $1\times 10^{-6}$ \\
Optimizer & AdamW \\
Gradient Accumulation & 4 steps \\
Batch Size & 1 (Effective 4) \\
Weight Clipping & $[0.0, 10.0]$ \\
\bottomrule
\end{tabular}
\end{table}

\section{Best-of-$N$ RL Implementation: Unconditional Protein Generation}
\label{app:bon_unconditional_implementation}
This and all other experiments were run on an Nvidia V100 GPU on a cluster with around 20 Gigabytes of storage.

\subsection{Training Configuration and Biochemical Topology}
For the unconditional protein generation task, we initialized the policy using the pre-trained \texttt{nferruz/ProtGPT2} model (738M parameters). The objective of this experiment was to test the model's ability to traverse a multi-modal reward landscape (the "Valley of Death") to discover a rare, high-reward global optimum. 

While this reward landscape is mathematically engineered to test optimization dynamics, it abstracts a fundamental thermodynamic challenge in protein engineering: the design of stable transmembrane domains without triggering aggregation. Let $h(y)$ represent the hydrophobicity ratio of a generated sequence $y$, defined as the fraction of amino acids belonging to the hydrophobic set \{A, I, L, M, F, W, Y, V\}. The reward function explicitly models three biochemical states:
\begin{equation}
    R(y) = 4.0 \exp\left(-\frac{(h(y)-0.35)^2}{0.05}\right) + 12.0 \exp\left(-\frac{(h(y)-0.75)^2}{0.015}\right) + \text{Penalty}(y)
\end{equation}

\begin{itemize}
    \item \textbf{The Trap ($h \approx 0.35$):} This represents standard, water-soluble globular proteins. Because the vast majority of characterized proteins in the pre-training data are soluble, the language model naturally gravitates here. It acts as a highly stable, "safe" local optimum.
    
    \item \textbf{The Jackpot ($h \approx 0.75$):}  This represents highly hydrophobic integral membrane proteins. These sequences are extremely valuable as therapeutic targets (e.g., GPCRs) but are structurally narrow and generationally rare.
    
    \item \textbf{The Valley ($0.45 < h < 0.60$):}  Intermediate hydrophobicity creates "sticky" sequences. They possess too many exposed hydrophobic patches to remain soluble in water, but lack the uniform hydrophobicity required to stably insert into a lipid bilayer. Biologically, these sequences misfold and form toxic aggregates (e.g., amyloid fibrils). We model this by applying $\text{Penalty}(y) = -2.0$ if the sequence falls into this aggregation zone.
\end{itemize}

Standard RL gets stuck in the soluble "Trap" because moving toward the transmembrane "Jackpot" requires traversing the heavily penalized aggregation "Valley." To prevent the model from generating pathologically short sequences or non-biological repetitions to game the metric, any sequence shorter than 10 amino acids receives a flat reward of $-5.0$, and the RL objective includes a Kullback-Leibler divergence penalty ($\beta = 0.3$) \cite{ouyang2022training} against the frozen pre-trained reference model. The model was trained for 300 steps with a learning rate of $1 \times 10^{-5}$.

\subsection{Global Quantile Estimation for Best-of-$N$}
The Best-of-$N$ RL objective relies on weighting the policy gradient by the marginal utility of the sample under a search process, which we derived as proportional to $N(P_{<y})^{N-1}$. The critical parameter to estimate is $P_{<y}$, the probability that a random sample from the current policy yields a reward strictly lower than the current sequence $y$.

Because the unconditional reward landscape is static (the reward depends solely on the output sequence), we can accurately estimate this quantile using a \textit{global historical buffer}. Relying strictly on small in-batch ranks introduces excessive variance; a mediocre sample might appear excellent simply because the rest of the batch was exceptionally poor.

We maintain a global, rolling First-In-First-Out (FIFO) buffer $\mathcal{B}$ containing the rewards of the most recent 4,000 generated sequences. For a newly generated sequence $y_i$ with reward $R_i$, its global quantile is estimated as the fraction of historical samples it outperforms, smoothed to prevent zero-probabilities:
\begin{equation}
    P_{<y_i} = \text{clamp}\left( \frac{\sum_{r \in \mathcal{B}} \mathbb{I}(r < R_i) + 1}{|\mathcal{B}| + 2}, \; 0.001, \; 0.999 \right)
\end{equation}

\subsection{Gradient Weight Normalization}
The raw Best-of-$N$ scaling weight is computed as $w_i = N (P_{<y_i})^{N-1}$. Because this function is highly convex (exponentially decaying for all but the top percentile), applying it directly alters the effective magnitude of the overall gradient step, which can stall training if the entire batch falls into a low-quantile regime.

To decouple the \textit{direction} of the update from the \textit{magnitude} of the learning rate, we normalize the weights across the current batch $M$:
\begin{equation}
    \tilde{w}_i = \frac{w_i}{\frac{1}{M} \sum_{j=1}^M w_j}
\end{equation}
This batch-level normalization ensures that the expected scale of the gradient remains consistent with standard RL, while strictly redistributing the optimization capacity toward the highest-performing samples in the batch. The final normalized weight $\tilde{w}_i$ is detached from the computation graph and multiplied by the standard normalized advantage to compute the loss.

\section{Best-of-$N$ RL Implementation: Conditional Protein Generation}
\label{app:bon_conditional_implementation}
This and all other experiments were run on an Nvidia V100 GPU on a cluster with around 20 Gigabytes of storage.

\subsection{Dataset Formulation and Biochemical Motivation}
The conditional generation experiment tests the framework's ability to learn complex relational rules. We synthesized a \texttt{PairedProteinDataset} where each prompt contains an input sequence with a specific hydrophobicity $h_{in}$. The task is to generate an output sequence satisfying the inverse rule: $h_{target} = 1.0 - h_{in}$.

This inverse relationship abstracts a critical practical objective in biomolecular engineering: the design of compensatory fusion tags and amphiphilic binding partners.  For instance, researchers frequently need to append highly hydrophilic sequences (e.g., SUMO or GST tags, $h \approx 0.2$) to forcibly solubilize highly hydrophobic target proteins ($h \approx 0.8$) for crystallization or manufacturing. Conversely, a highly hydrophilic catalytic domain might require the fusion of a hydrophobic anchor to ensure localization to a cellular membrane. 

To ground the generation process, the model underwent a Supervised Fine-Tuning (SFT) warmup phase. We generated 2,000 "gold" pairs where the target sequences were explicitly constructed to perfectly satisfy the compensatory relationship. The base \texttt{ProtGPT2} model was fine-tuned on this dataset for 60 steps at a learning rate of $5 \times 10^{-6}$ using standard cross-entropy, establishing a functional prior for the subsequent RL phase.

\subsection{Relational Reward and Penalty}
During the RL phase, the reward function dynamically evaluates the generated tag's thermodynamic complementarity to the input sequence. For an input hydrophobicity $h_{in}$ and generated output $y$, the reward is defined as a Gaussian ridge centered on the inverse target:
\begin{equation}
    R(y, h_{in}) = 10.0 \exp\left( - \frac{(h(y) - (1-h_{in}))^2}{0.05} \right) + \text{Penalty}(y, h_{in})
\end{equation}

To prevent the model from discovering a trivial "identity mapping" shortcut (i.e., simply copying the input sequence, which yields deceptively high rewards for inputs near $h_{in} \approx 0.5$), we explicitly penalize the model. A penalty of $-2.0$ is applied if the output hydrophobicity matches the input ($|h(y) - h_{in}| < 0.1$) AND the input is sufficiently far from the midpoint ($|h_{in} - 0.5| > 0.1$). This constraint forces the model to genuinely learn sequence complementarity rather than sequence repetition.

The RL optimization utilized a batch size of 25, a learning rate of $2 \times 10^{-6}$, and a lower KL penalty ($\beta = 0.05$) to allow the model more flexibility to diverge from the SFT prior.

\subsection{In-Batch Rank Estimation (Dynamic Quantiles)}
Unlike the unconditional task, the reward landscape in the conditional task shifts dynamically with every prompt. An absolute reward of $5.0$ might represent a spectacular success for a highly complex input sequence, but a mediocre failure for an easy one. Consequently, evaluating a sequence against a global historical buffer of rewards (which aggregates across differing inputs) would yield a mathematically invalid quantile $P_{<y}$.

Therefore, for conditional Best-of-$N$ RL, we exclusively utilize \textit{in-batch rank estimation}. Because all generated sequences in a specific mini-batch are conditioned on the same inputs, their rewards are directly comparable. 

For a batch of $M$ generated sequences (in our implementation, $M=25$), we sort the rewards and assign an ascending integer rank to each sample (where $0$ is the worst and $M-1$ is the best). The empirical quantile is computed smoothly as:
\begin{equation}
    P_{<y_i} = \frac{\text{rank}(R_i) + 1}{M + 1}
\end{equation}
This ensures that the Best-of-$N$ weight $w_i = N(P_{<y_i})^{N-1}$ accurately reflects the relative success of a sequence conditioned specifically on its prompt. As with the unconditional formulation, these weights are batch-normalized ($\tilde{w}_i = w_i / \mathbb{E}[w]$) before being applied to the advantage scaled log-probabilities.
\section{Ablations: Variance Reduction in Strategy-Aware RL}
\label{app:ablations}

In our theoretical formulation of the Compute Aligned Training framework for Reinforcement Learning (\Cref{app:rl}), the pure policy gradient estimator applies a weight directly proportional to the marginal increase in test-time success: $w_{\text{RL}} = \frac{\partial \tilde{p}}{\partial p}$. However, as analyzed in \Cref{app:rl_variance} and \Cref{app:normalized_estimator_theory}, this pure estimator can suffer from severe variance inflation, particularly for strategies like Pass@$N$ evaluated at large $N$.

To remedy this, our primary RL implementation utilizes an normalization mechanism derived from the SFT framework, multiplying the raw weight by a normalization term $\frac{p}{\tilde{p}}$ to create a "Log-Weighted" estimator. We discussed some motivation for this in \Cref{app:normalized_estimator_theory} .

To empirically validate the necessity of this stabilization term, we conducted an ablation study utilizing the Group Relative Policy Optimization (GRPO) pipeline. We compared models trained with the pure RL weights against those trained with the Log-Weighted formulation across training budgets $N \in \{4, 16\}$.

\subsection{Empirical Results}

The full test-time scaling results of this ablation are detailed in Table \ref{tab:rl_full_appendix} and visualized in Figure \ref{fig:appendix_full_scaling}.

\begin{table}[h]
\centering
\caption{Full comparison of RL models trained with different weighting strategies. SFT-derived (Log) weights consistently yield better scaling properties than pure RL-derived weights by preventing variance explosion during GRPO updates.}
\label{tab:rl_full_appendix}
\resizebox{\linewidth}{!}{%
\begin{tabular}{@{}lccccc@{}}
\toprule
\textbf{Model} & \textbf{@1} & \textbf{@4} & \textbf{@8} & \textbf{@16} & \textbf{@32} \\ \midrule
Standard RL & 8.27\% & 17.88\% & 24.09\% & 30.25\% & 35.80\% \\ \midrule
Pass@4 (Log) & \textbf{8.76\%} & \textbf{18.71\%} & \textbf{25.25\%} & 31.61\% & 36.60\% \\
Pass@16 (Log) & 8.28\% & 18.26\% & 25.08\% & \textbf{32.52\%} & \textbf{40.00\%} \\ \midrule
Pass@4 (Pure RL) & 8.57\% & 18.37\% & 24.74\% & 31.26\% & 37.20\% \\
Pass@16 (Pure RL)& 8.44\% & 17.77\% & 24.09\% & 30.88\% & 37.20\% \\ \bottomrule
\end{tabular}%
}
\end{table}

\begin{figure}[h]
    \centering
    \includegraphics[width=0.85\textwidth]{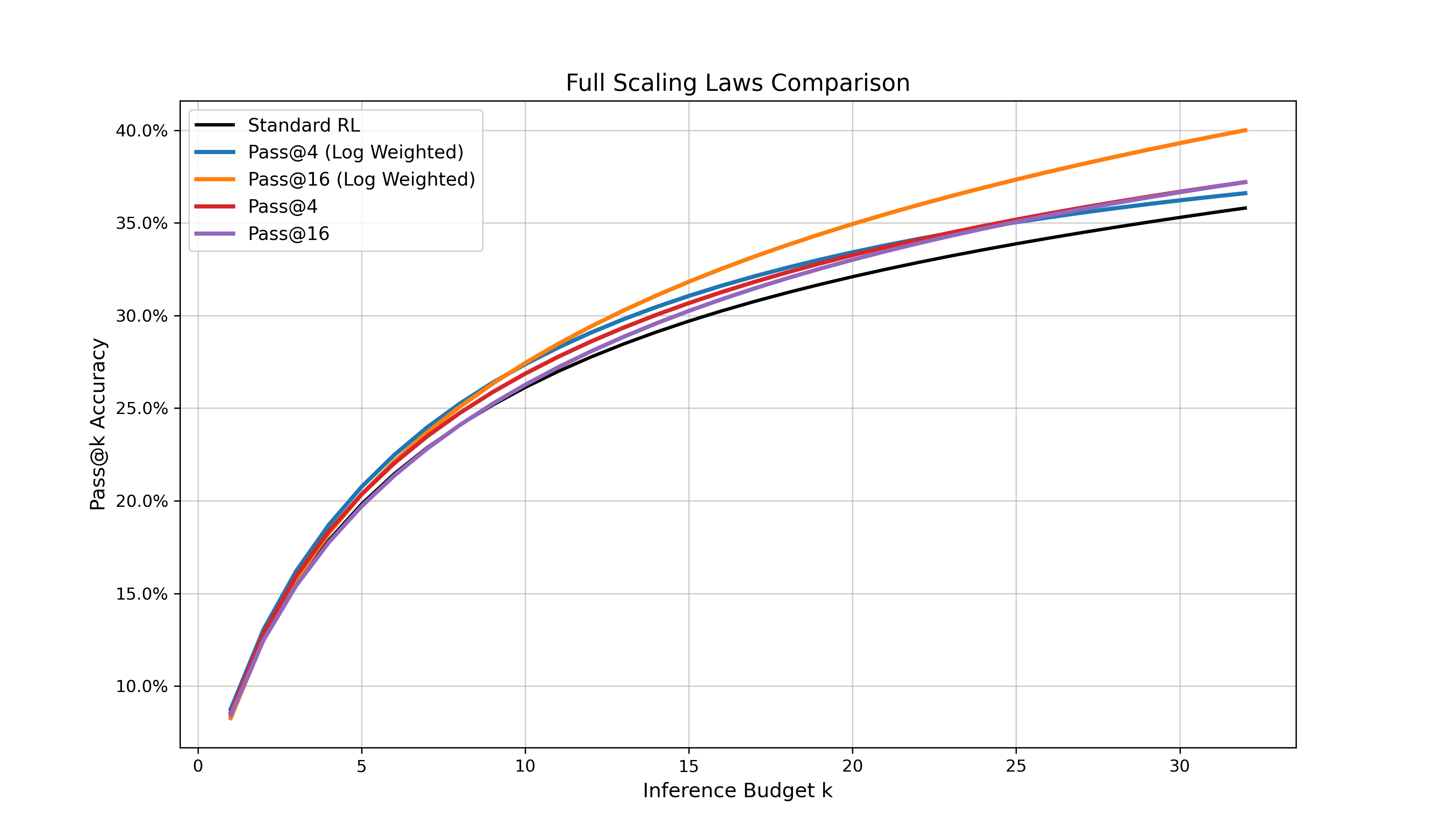}
    \caption{\textbf{Scaling Laws for RL Weighting Strategies.} The Log-Weighted estimator at $N=16$ avoids the optimization instability of the Pure RL estimator, allowing the model to successfully translate a higher training budget into superior test-time scaling.}
    \label{fig:appendix_full_scaling}
\end{figure}

\paragraph{Standard RL vs. Strategy-Aware RL:} 
All strategy-aware models, regardless of the specific weighting formulation, demonstrate superior scaling properties compared to the Standard RL baseline. Standard RL saturates quickly, reaching only $35.80\%$ at $k=32$, validating our hypothesis that mean-seeking objectives inadvertently suppress the distributional variance required for effective search.

\paragraph{Pure RL vs. Log-Weighted (Normalized):} 
The ablation clearly demonstrates the superiority of the Log-Weighted estimator, and highlights exactly when and why the Pure RL estimator fails:
\begin{itemize}
    \item \textbf{Low Budget ($N=4$):} When the training budget is small, the variance penalty of the Pure RL weight ($N^2 = 16$) remains manageable for the optimizer. Consequently, both Pass@4 and Pass@4 (Log) perform similarly, successfully optimizing for low-budget precision ($k \le 8$).
    \item \textbf{High Budget ($N=16$):} As the training budget scales, the instability of the Pure RL estimator degrades performance. The Pass@16 (Pure RL) model fails to eclipse the Pass@4 models, achieving only $37.20\%$ at $k=32$. The unnormalized gradient magnitudes likely cause optimization thrashing, preventing the policy from smoothly reallocating probability mass to the tails.
    \item \textbf{The Stabilized Champion:} In contrast, the Pass@16 (Log) model effectively leverages the higher "mental budget." Guarded by the adaptive gain controller, it successfully trains on harder examples without destabilizing the policy, achieving the highest asymptotic performance of $40.00\%$ at $k=32$ (a substantial $+4.2\%$ absolute improvement over Standard RL).
\end{itemize}

Ultimately, this ablation confirms that while aligning the RL objective with the test-time strategy provides the correct \textit{directional} signal, stabilizing the magnitude of the estimator via the $\frac{p}{\tilde{p}}$ normalization term is practically essential for robust optimization at scale.

\subsection{Ablation Study: SFT vs. RL Weighting in Majority Vote RL}
\label{subsec:maj_weight_ablation}

To rigorously test the impact of the specific gradient estimator used during Reinforcement Learning, we conducted an ablation study comparing the standard SFT-derived weighting factor against our proposed RL-specific weighting factor.

While both approaches aim to align the objective with the test-time metric, they differ in their normalization. The SFT-derived weight, $w_{SFT} = \frac{p}{\tilde{p}}\frac{\partial \tilde{p}}{\partial p}$, includes a normalization term $\frac{p}{\tilde{p}}$ designed to prevent variance explosion as $p \to 0$. In contrast, the RL-specific weight, $w_{RL} = \frac{\partial \tilde{p}}{\partial p}$, is the raw marginal utility, which concentrates the gradient signal entirely on the decision boundary.

We trained models using both estimators across two different target budgets ($N=4$ and $N=8$). The results are visualized in Figure \ref{fig:appendix_maj_ablation} and detailed quantitatively in Table \ref{tab:maj_ablation_results}.

\begin{figure}[h]
    \centering
    \includegraphics[width=0.85\textwidth]{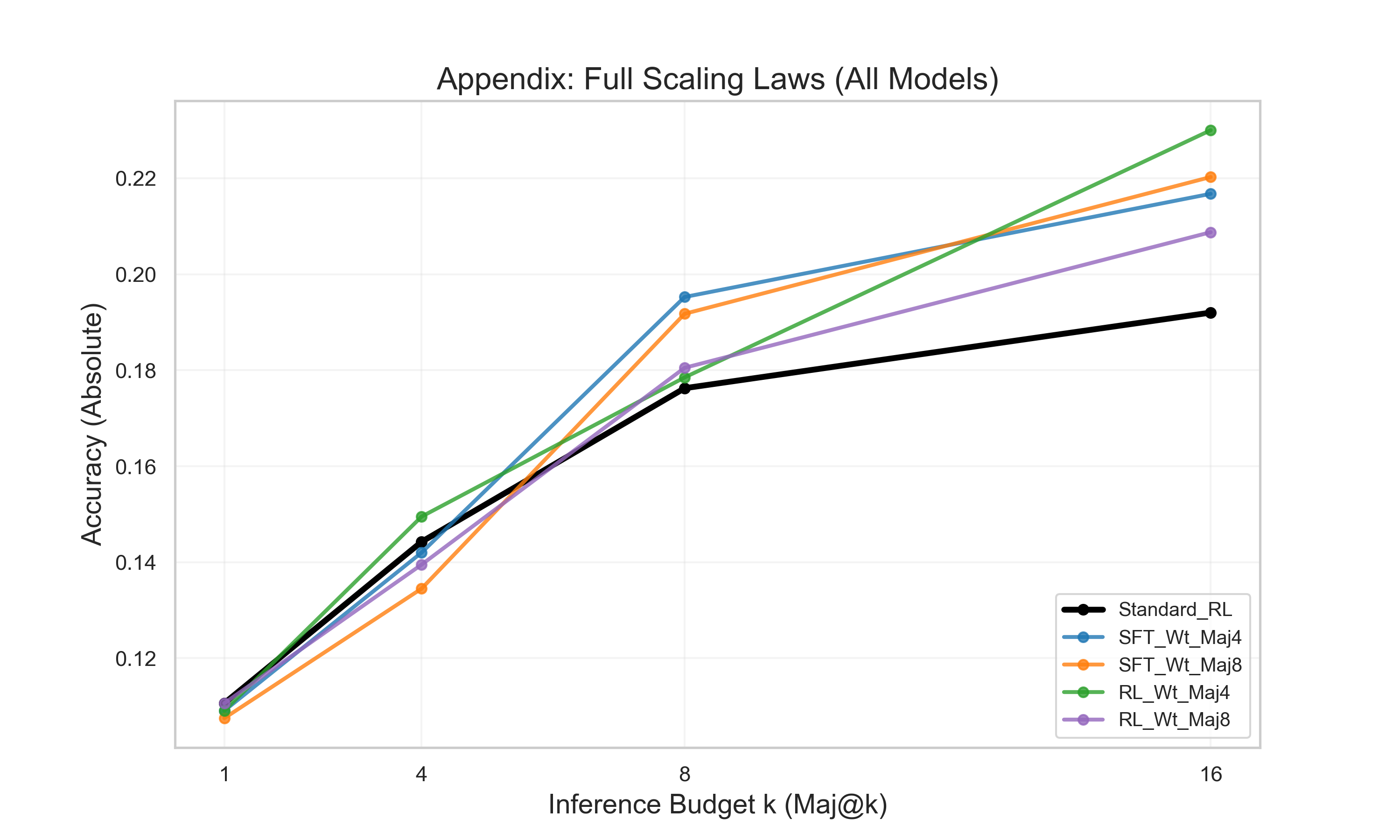}
    \caption{\textbf{Full Scaling Laws (All Models).} A comparison of Majority Vote accuracy across inference budgets $k$. The strategy-aware models (colored lines) consistently outperform the Standard RL baseline (black line) at higher budgets. Notably, the \texttt{RL\_Wt\_Maj4} model (Green) achieves the steepest scaling curve, demonstrating that the raw marginal utility estimator provides the strongest signal for consensus optimization.}
    \label{fig:appendix_maj_ablation}
\end{figure}

\begin{table}[h]
\centering
\caption{\textbf{Ablation Results: SFT vs. RL Weights.} Comparison of Majority Vote accuracy (Maj@$k$) across different weighting strategies and training budgets. The \texttt{RL\_Wt} estimators generally provide superior scaling at higher inference budgets compared to their \texttt{SFT\_Wt} counterparts, confirming that the "safety brake" normalization in SFT weights may dampen the critical signal required for consensus.}
\label{tab:maj_ablation_results}
\begin{tabular}{@{}lcccc@{}}
\toprule
\textbf{Model} & \textbf{Maj@1} & \textbf{Maj@4} & \textbf{Maj@8} & \textbf{Maj@16} \\ \midrule
Standard RL & \textbf{0.1106} & 0.1443 & 0.1763 & 0.1920 \\ \midrule
SFT\_Wt\_Maj4 & 0.1091 & 0.1420 & \textbf{0.1953} & 0.2168 \\
SFT\_Wt\_Maj8 & 0.1075 & 0.1345 & 0.1918 & 0.2203 \\ \midrule
RL\_Wt\_Maj4 & 0.1091 & \textbf{0.1495} & 0.1785 & \textbf{0.2300} \\
RL\_Wt\_Maj8 & 0.1106 & 0.1395 & 0.1805 & 0.2088 \\ \bottomrule
\end{tabular}%
\end{table}

\paragraph{Analysis.}
The results confirm two key findings:
\begin{enumerate}
    \item \textbf{Superiority of RL Weights:} The \texttt{RL\_Wt\_Maj4} model achieves the highest asymptotic performance (23.00\% at Maj@16), outperforming both the baseline and the SFT-weighted variants. This supports the hypothesis that the "spotlight" behavior of the raw derivative, which vanishes for easy/hard samples and explodes at the boundary, is a feature, not a bug, for consensus tasks.
    \item \textbf{The Trade-off:} The \texttt{SFT\_Wt} models (Blue/Orange) start slower (lower Maj@4) but scale robustly. Their normalized weights effectively reduce variance, but at the cost of "blurring" the critical decision boundary signal needed to maximize the plurality vote.
\end{enumerate}

\newpage
\section*{NeurIPS Paper Checklist}
\begin{enumerate}[leftmargin=*]
\item \textbf{Claims}\\
\textbf{Question:} Do the main claims made in the abstract and introduction accurately reflect the paper's contributions and scope?\\
\textbf{Answer:} \answerYes{}\\
\textbf{Justification:} We make the claim that our method alters classic post-training of generative models such that they work better under test time strategies without additional computational overhead. We show empirically and theoretically how CAT achieves that.

\item \textbf{Limitations}\\
\textbf{Question:} Does the paper discuss the limitations of the work performed by the authors?\\
\textbf{Answer:} \answerYes{}\\
\textbf{Justification:} We have a limitations paragraph in the last section that accurately captures where our method currently falls short.

\item \textbf{Theory assumptions and proofs}\\
\textbf{Question:} For each theoretical result, does the paper provide the full set of assumptions and a complete (and correct) proof?\\
\textbf{Answer:} \answerYes{}\\
\textbf{Justification:} We derive the test time aligned gradient for every strategy used in experiments, prove conditions under which the diagonal-approximation error is bounded, and show how those conditions apply in the cases considered.

\item \textbf{Experimental result reproducibility}\\
\textbf{Question:} Does the paper fully disclose the information needed to reproduce the main experimental results to the extent that it affects the main claims and conclusions?\\
\textbf{Answer:} \answerYes{}\\
\textbf{Justification:} The appendix includes detailed experimental procedures and settings intended to permit reproduction of the reported results.

\item \textbf{Open access to data and code}\\
\textbf{Question:} Does the paper provide open access to the data and code, with sufficient instructions to faithfully reproduce the main experimental results?\\
\textbf{Answer:} \answerYes{}\\
\textbf{Justification:} The paper points to the code repository used to reproduce the experiments.

\item \textbf{Experimental setting/details}\\
\textbf{Question:} Does the paper specify the training and test details necessary to understand the results?\\
\textbf{Answer:} \answerYes{}\\
\textbf{Justification:} The relevant training, evaluation, data-processing, optimizer, and hyperparameter details are included in the appendix.

\item \textbf{Experiment statistical significance}\\
\textbf{Question:} Does the paper report error bars suitably and correctly defined or other appropriate information about statistical significance?\\
\textbf{Answer:} \answerYes{}\\
\textbf{Justification:} Standard errors are included in the main experimental tables.

\item \textbf{Experiments compute resources}\\
\textbf{Question:} For each experiment, does the paper provide sufficient information on the computer resources needed to reproduce the experiments?\\
\textbf{Answer:} \answerYes{}\\
\textbf{Justification:} The appendix states that the experiments were run on an Nvidia V100 GPU and reports the cluster storage context.

\item \textbf{Code of ethics}\\
\textbf{Question:} Does the research conform, in every respect, with the NeurIPS Code of Ethics?\\
\textbf{Answer:} \answerYes{}\\
\textbf{Justification:} The authors reviewed the code of ethics and identified no particular societal harm beyond the potential harms associated with generative modelling as a field.

\item \textbf{Broader impacts}\\
\textbf{Question:} Does the paper discuss both potential positive societal impacts and negative societal impacts of the work performed?\\
\textbf{Answer:} \answerNA{}\\
\textbf{Justification:} We see no particular broader societal impact of our work.

\item \textbf{Safeguards}\\
\textbf{Question:} Does the paper describe safeguards for responsible release of data or models that have a high risk for misuse?\\
\textbf{Answer:} \answerNA{}\\
\textbf{Justification:} We do not release any novel data or models.

\item \textbf{Licenses for existing assets}\\
\textbf{Question:} Are creators or original owners of assets used in the paper properly credited and are license and terms of use explicitly mentioned and respected?\\
\textbf{Answer:} \answerYes{}\\
\textbf{Justification:} Every model, dataset, and library used is cited, with license information recorded in the bibliography where applicable.

\item \textbf{New assets}\\
\textbf{Question:} Are new assets introduced in the paper well documented, and is the documentation provided alongside the assets?\\
\textbf{Answer:} \answerNA{}\\
\textbf{Justification:} We do not release new assets.

\item \textbf{Crowdsourcing and research with human subjects}\\
\textbf{Question:} Does the paper include the required instructions, screenshots, and compensation details for crowdsourcing or human-subject experiments?\\
\textbf{Answer:} \answerNA{}\\
\textbf{Justification:} No human subjects are involved.

\item \textbf{Institutional review board approvals or equivalent}\\
\textbf{Question:} Does the paper describe risks to study participants and relevant IRB approvals?\\
\textbf{Answer:} \answerNA{}\\
\textbf{Justification:} This does not apply because the work does not involve human subjects.

\item \textbf{Declaration of LLM usage}\\
\textbf{Question:} Does the paper describe the usage of LLMs if it is an important, original, or non-standard component of the core methods?\\
\textbf{Answer:} \answerNA{}\\
\textbf{Justification:} Aside from formatting the LaTeX and editing code, LLMs were not used as a core method-development component; they were used to clean up code in preparation for submission.
\end{enumerate}

\end{document}